\documentclass[nohyperref,dvipsnames]{article}

\usepackage[accepted]{icml2023}

\usepackage[utf8]{inputenc} 
\usepackage[T1]{fontenc}    
\usepackage{hyperref}       
\usepackage{url}            
\usepackage{booktabs}       
\usepackage{amsfonts}       
\usepackage{nicefrac}       
\usepackage{microtype}      
\usepackage{wrapfig,lipsum}
\usepackage{amsthm,mathrsfs,amsmath}
\usepackage{amssymb}
\usepackage{paralist}
\usepackage{tcolorbox}
\usepackage{pifont}
\usepackage{sectsty}
\usepackage{xargs}
\usepackage{dsfont}
\usepackage{caption}
\usepackage{subcaption} 
\usepackage{upgreek}
\usepackage{mathrsfs}
\usepackage{enumitem}
\usepackage{aliascnt}
\usepackage[noend]{algpseudocode}  
\usepackage{algorithm}
\usepackage{algorithmicx}
\usepackage{bbm}
\usepackage{bm}

\hypersetup{colorlinks,citecolor = blue}  
\graphicspath{{figures/}}

\usepackage[capitalize,noabbrev]{cleveref}

\usepackage[textsize=tiny]{todonotes}
\usepackage{xcolor} 

\icmltitlerunning{Conformal Prediction for Federated Uncertainty Quantification Under Label Shift}

\theoremstyle{plain}
\newtheorem{theorem}{Theorem}[section]

\newtheorem{lemma}[theorem]{Lemma}
\newtheorem{corollary}[theorem]{Corollary}
\theoremstyle{definition}
\newtheorem{definition}[theorem]{Definition}
\newtheorem{assumption}{\textbf{H}\hspace{-3pt}}
\crefformat{assumption}{{\textbf{H}}#2#1#3}
\theoremstyle{remark}
\newtheorem{remark}[theorem]{Remark}

\definecolor{aurometalsaurus}{rgb}{0.43, 0.5, 0.5}
\definecolor{britishracinggreen}{rgb}{0.0, 0.26, 0.15}
\definecolor{burntumber}{rgb}{0.54, 0.2, 0.14}
\definecolor{cobalt}{rgb}{0.0, 0.28, 0.67}
\definecolor{bulgarianrose}{rgb}{0.28, 0.02, 0.03}
\definecolor{ceruleanblue}{rgb}{0.16, 0.32, 0.75}

\newcommand{\group}[1]{\bgroup\small\sffamily\noindent\color{green}{#1}\egroup}  
\newcommand{\new}[1]{\bgroup\small\sffamily\noindent\color{cobalt}{#1}\egroup}
\newcommand{\old}[1]{\bgroup\small\color{gray}{#1}\egroup}  

\definecolor{darkgreen}{RGB}{0,128,0}

\hypersetup{colorlinks = true, urlcolor = blue, bookmarksopen = true}

\newcommand{\N}{\mathbb{N}}
\newcommand{\Z}{\mathbb{Z}}
\newcommand{\Q}{\mathbb{Q}}
\newcommand{\R}{\mathbb{R}}

\newcommand{\E}{\mathbb{E}}

\newcommand{\argmin}{\operatornamewithlimits{\arg\min}}

\newcommand{\Oh}{\operatorname{\mathrm{O}}}

\newcommand{\var}{\operatorname{Var}}

\newcommand{\prob}{\mathbb{P}}

\newcommand{\rme}{\mathrm{e}}
\newcommand{\rmd}{\mathrm{d}}
\newcommand{\1}{\mathds{1}}

\newcommand{\pr}[1]{\left({#1}\right)}
\newcommand{\prt}[1]{({\textstyle{#1}})}
\newcommand{\prn}[1]{({#1})}

\newcommand{\prBig}[1]{\Big({#1}\Big)}
\newcommand{\prbigg}[1]{\bigg({#1}\bigg)}

\newcommand{\br}[1]{\left[{#1}\right]}

\newcommand{\brn}[1]{[{#1}]}
\newcommand{\brbig}[1]{\big[{#1}\big]}

\newcommand{\brBigg}[1]{\Bigg[{#1}\Bigg]}
\newcommand{\ac}[1]{\left\{{#1}\right\}}
\newcommand{\act}[1]{\{{\textstyle{#1}}\}}
\newcommand{\acn}[1]{\{{#1}\}}
\newcommand{\acbig}[1]{\big\{{#1}\big\}}

\newcommand{\norm}[1]{\left\|{#1}\right\|}

\newcommand{\normn}[1]{\|{#1}\|}
\newcommand{\normbig}[1]{\big\|{#1}\big\|}

\newcommand{\abs}[1]{\left\lvert{#1}\right\rvert}

\newcommand{\absn}[1]{|{#1}|}

\newcommand{\ps}[2]{\left\langle{#1},{#2}\right\rangle}
\newcommand{\psn}[2]{\langle{{#1},{#2}}\rangle}

\newcommand{\nofrac}[2]{{#1}/{#2}}
\newcommand{\gauss}{\mathcal{N}}

\newcommand{\q}[1]{Q_{#1}} 

\newcommandx{\qmoreau}[2][1=1 - \alpha,2=\gamma]{Q_{#1}^{#2}}
\newcommandx{\qhatmoreau}[3][1=1 - \alpha,2=\gamma,3=T]{\widehat{Q}_{#1,#3}^{#2}(\qdistr{\yquery})}
\newcommandx{\qbetagammay}[2][1=\beta,2=\gamma]{\qmoreau[#1][#2]\prn{\widehat{\mu}_{\yquery}}}
\newcommandx{\qalgo}[1][1=t]{\bar{q}_{#1}}
\newcommandx{\qalgoi}[2][1=t,2=i]{\bar{q}_{#1}^{#2}}
\newcommand{\rmexp}{\mathrm{e}}
\renewcommand{\algorithmiccomment}[1]{\textbf{\textcolor{darkgreen}{// \texttt{#1}}}}

\newcommand{\tcount}{N}
\newcommand{\lcount}[1]{\tcount_{#1}}
\newcommand{\ccount}[1]{\tcount^{#1}}
\newcommand{\clcount}[2]{\tcount_{#1}^{#2}}
\newcommand{\Mtcount}{\mathsf{M}}
\newcommand{\Mlcount}[1]{\Mtcount_{#1}}
\newcommand{\Mccount}[1]{\Mtcount^{#1}}
\newcommand{\Mclcount}[2]{\Mtcount_{#1}^{#2}}
\newcommand{\privtcount}{\widehat{\Mtcount}}
\newcommand{\privlcount}[1]{\privtcount_{#1}}

\newcommand{\privclcount}[2]{\privtcount_{#1}^{#2}}
\newcommand{\privnoise}[2]{z_{#1}^{#2}}

\newcommand{\Xtarget}{X_{\ccount{\target}+1}^{\target}}
\newcommand{\Ytarget}{Y_{\ccount{\target}+1}^{\target}}

\newcommand{\nclients}{n}
\newcommand{\cset}{[\nclients]}
\newcommand{\target}{\star}

\newcommand{\iweight}[1]{\widehat{w}_{#1}^{\target}}
\newcommand{\iweightc}[2]{w_{#1}^{#2}}

\newcommand{\iweightp}[1]{w_{#1}^{\target}}

\newcommand{\iweightN}[1]{\widehat{w}_{#1}}

\newcommand{\iweighthatN}[1]{\widehat{w}_{#1}^{\target}}
\newcommand{\iweightpN}[1]{w_{#1}^{\target}} 

\newcommand{\qweight}[2]{\widehat{p}_{#1, #2}^{\target}}
\newcommand{\qweightp}[2]{p_{#1, #2}^{\target}}
\newcommand{\qweightb}[2]{\Bar{p}_{#1, #2}^{\target}}
\newcommand{\qweighthatN}[2]{\widehat{p}_{#1}^{\widehat{D}_{#2}}} 
\newcommand{\qweightwhatN}[2]{\Bar{p}_{#1}^{\Bar{D}_{#2}}}

\newcommandx{\qweighthat}[1]{\widehat{p}_{#1}}

\newcommandx{\qweightz}[2][1=k,2=z_{1:\tcount+1}]{p_{#1}^{#2}}

\newcommand{\qdistr}[1]{\widehat{\mu}_{#1}}
\newcommand{\qdistrp}[1]{\mu_{#1}^{\target}}

\newcommand{\cpmf}[2]{P_{Y}^{#1}\prn{#2}}  
\newcommand{\tpmf}[1]{P_{Y}^{\target}\prn{#1}}  

\newcommand{\Pcal}[1]{P^{#1}}
\newcommand{\Pycal}[1]{P_{Y}^{#1}}

\newcommand{\xquery}{\mathbf{x}}
\newcommand{\yquery}{\mathbf{y}}

\newcommand{\thetav}{\bm{\theta}}
\newcommand{\XC}{\mathcal{X}}
\newcommand{\YC}{\mathcal{Y}}

\newcommandx{\predset}[1]{\mathcal{C}_{#1,\mu^{\target}}}
\newcommandx{\predsetg}[2]{\mathcal{C}_{#1,#2}}
\newcommandx{\predsetghat}[2]{\mathcal{C}_{#1,#2}^{\scalebox{.65}{$\mathrm{MLE}$}}}
\newcommand{\predsetmle}{\mathcal{C}_{\alpha,\widehat{\mu}^{\scalebox{.35}{$\mathrm{MLE}$}}}}
\newcommandx{\hatpredset}[1]{\widehat{\mathcal{C}}_{#1,\widehat{\mu}}}
\newcommandx{\predsetmoreau}[1]{C^{\gamma}_{#1,\widehat{\mu}}}
\newcommandx{\predsethatmoreau}[1]{\widehat{\mathcal{C}}^{\gamma}_{#1,\widehat{\mu}}}
\newcommandx{\unweightlocpredset}[1]{\mathcal{C}_{#1,\bar{\mu}^{\operatorname{loc},\target}}}
\newcommandx{\unweightglopredset}[1]{\mathcal{C}_{#1,\bar{\mu}}}
\newcommandx{\unweightlocdist}{\bar{\mu}^{\operatorname{loc},\target}_{\yquery}}
\newcommandx{\unweightglodist}{\bar{\mu}^{\operatorname{glo}}_{\yquery}}

\newcommandx{\indi}[2][1=]{\1^{#1}_{#2}}
\newcommand{\indiacc}[1]{\1_{\{#1\}}}

\newcommand{\varnoise}{\xi}
\newcommand{\fcv}[1]{F^{#1}}
\newcommand{\fcvtot}{F}
\newcommand{\tv}{\mathrm{d}_{\mathrm{TV}}}
\newcommand{\algo}{\texttt{DP-FedAvgQE}}
\newcommand{\algopred}{\texttt{DP-FedCP}}
\def\unwghtloc{\texttt{Unweighted Local}}
\def\unwghtglo{\texttt{Unweighted Global}}
\def\estwght{\texttt{Estimated Weights}}
\def\qthmhat{\widehat{\rho}_{k_{\mathrm{c}}}}

\def\iid{i.i.d.}

\newcommand{\temp}{T}
\newcommand{\format}[1]{\paragraph{#1}}
\newcommandx{\dinit}{\Delta}
\newcommandx{\pinball}[2][1=\alpha,2=v]{S_{#1,#2}}
\newcommandx{\pinballmoreau}[4][1=\alpha,2=\gamma]{S_{#1,#3}^{#2}(#4)}
\newcommandx{\pinballloss}[3][1=\beta]{S_{#1,#2}(#3)}
\newcommandx{\lossglobal}[2][1=\alpha,2=\gamma]{\operatorname{S}_{#1}^{#2}}
\newcommandx{\lossi}[2][1=\alpha,2=\gamma]{\operatorname{S}_{#1}^{i,#2}}

\newcommand{\noiseparamone}{\sigma_g}
\newcommand{\noiseparamtewo}{\bar{\sigma}}
\newcommand{\multi}{\mathcal{M}}

\newcommand{\card}[1]{\mathrm{Card}(#1)}
\newcommand{\Flip}{M}
\DeclareMathOperator{\ratiotv}{R}

\begin{document}


\twocolumn[
\icmltitle{Conformal Prediction for Federated \\ Uncertainty Quantification Under Label Shift}



\icmlsetsymbol{equal}{*}

\begin{icmlauthorlist}
\icmlauthor{Vincent Plassier}{equal,lagrange,ep}
\icmlauthor{Mehdi Makni}{equal,lagrange}
\icmlauthor{Aleksandr Rubashevskii}{skt,mbzuai}
\icmlauthor{Eric Moulines}{ep,mbzuai}
\icmlauthor{Maxim Panov}{tii}
\end{icmlauthorlist}

\icmlaffiliation{ep}{CMAP, Ecole Polytechnique, Paris, France}
\icmlaffiliation{tii}{Technology Innovation Institute, Abu Dhabi, UAE}
\icmlaffiliation{skt}{Skolkovo Institute of Science and Technology, Moscow, Russia}
\icmlaffiliation{lagrange}{Lagrange Mathematics and Computing Research Center, Paris, France}
\icmlaffiliation{mbzuai}{Mohamed bin Zayed University of Artificial Intelligence, Masdar City, UAE}

\icmlcorrespondingauthor{Vincent Plassier}{vincent.plassier@polytechnique.edu}
\icmlcorrespondingauthor{Maxim Panov}{maxim.panov@mbzuai.ac.ae}

\icmlkeywords{Machine Learning, ICML}

\vskip 0.3in
]



\printAffiliationsAndNotice{\icmlEqualContribution} 

\begin{abstract}
  Federated Learning (FL) is a machine learning framework where many clients collaboratively train models while keeping the training data decentralized. Despite recent advances in FL, the uncertainty quantification topic (UQ) remains partially addressed.
  Among UQ methods, conformal prediction (CP) approaches provides distribution-free guarantees under minimal assumptions.
  We develop a new federated conformal prediction method based on quantile regression and take into account privacy constraints. This method takes advantage of importance weighting to effectively address the label shift between agents and provides theoretical guarantees for both valid coverage of the prediction sets and differential privacy.
  Extensive experimental studies demonstrate that this method outperforms current competitors.
\end{abstract}

\addtocontents{toc}{\protect\setcounter{tocdepth}{0}}

\section{Introduction}
\label{sec:introduction}

\textbf{Federated learning} is an increasingly important framework for large-scale learning. FL allows many agents to train a model together under the coordination of a central server without ever transmitting the agents' data over the network, in an attempt to preserve privacy. There has been a considerable amount of FL work over the past 5 years, see e.g. \cite{bonawitz2019towards,yang2019federated,kairouz2021advances,li2020federated}.
Compared to classical machine learning techniques, FL has two unique features. First, the networked agents are massively distributed, communication bandwidth is limited, and agents are not always available (\emph{system heterogeneity}). Second, the data distribution at different agents can vary greatly (\emph{statistical heterogeneity}); see~\cite{huang2022heterogeneous,yoon2022heterogeneous}. These features lead to serious challenges for both training and inference in federated systems. The focus of this work is on federated inference procedures that allow to build prediction sets for each agent with a confidence level that can be guaranteed.

\textbf{Conformal Prediction}, originally introduced in~\cite{vovk1999machine,shafer2008tutorial,balasubramanian2014conformal}, has recently gained popularity. It generates prediction sets with guaranteed error rates. Conformal algorithms are shown to be always valid: the actual confidence level is the nominal one, without requiring any specific assumption about the distribution of the data beyond exchangeability; see~\cite{lei2013distribution,fontana2023conformal} and references therein. With few exceptions, CP methods were developed for centralized environments.

We consider below a supervised learning problem with features $x$ taking values in $\XC$ and labels $y$ taking values in $\YC$. Let $(X_k, Y_k)_{k = 1}^{N_{\mathrm{train}}+\tcount}$ be an independent and identically distributed (i.i.d.) dataset.
We divide the data into a \emph{training} and a \emph{calibration} dataset. Formally, let $\{\mathcal{K}_{\text{train}}, \mathcal{K}_{\text{cal}}\}$ be a partition of $\{1, \ldots, N_{\mathrm{train}}+\tcount\}$, and let $\tcount = |\mathcal{K}_{\text {cal }}|$. Without loss of generality, we take $\mathcal{K}_{\text{cal}} = \{1, \ldots, \tcount\}$.
We learn a predictor $\hat{f}\colon \XC \to \Delta_{\absn{\YC}}$ on the training set $\mathcal{K}_{\text {train}}$, where $\absn{\YC}$ is the number of classes and $\Delta_{\absn{\YC}}$ is the $\absn{\YC}$-dimensional probability simplex.
For any covariate $x \in \XC$ associated with a label $y \in \YC$, consider a classification score function $\mathcal{S}\colon \YC \times \Delta_{\absn{\YC}} \to [0, 1]$, independent of other covariates and labels, which yields a non-conformity score given by $V(x, y) = \mathcal{S}\prn{y, \hat{f}(x)}$. This non-conformity measure estimates how unusual an example looks.
Based on these non-conformity scores, standard CP procedure constructs, for each \emph{significance level} $\alpha \in [0, 1]$, a (measurable) \emph{set-valued} predictor $\mathcal{C}_{\alpha}(x)$ using $\{(X_k,Y_k)\}_{k = 1}^{\tcount}$ that satisfies the following conditions
\begin{equation}\label{eq:conformal-guarantee}
  \prob\pr{Y_{\tcount+1} \in \mathcal{C}_{\alpha}(X_{\tcount+1})} \ge 1 - \alpha,
\end{equation}
where $(X_{\tcount+1}, Y_{\tcount+1})$ is a \emph{test point} that is independent of $\mathcal{K}_{\text {train}}$ and $\mathcal{K}_{\text {cal}}$.
The quantity $1 - \alpha$ is called the \emph{confidence level}.
The guarantee~\eqref{eq:conformal-guarantee} is set up in a centralized environment -- all data are available at a central node and usually assuming that the distributions of calibration and test data satisfy $P^{\mathrm{cal}} = P^{\target}$.
If there is a mismatch between the distributions $P^{\mathrm{cal}}$ and $P^{\target}$, then corrections should be made to ensure an appropriate confidence level; see~\cite{tibshirani2019conformal,podkopaev2021distribution,barber2022conformal} and references therein.

\format{Setup.}
  In this work, we consider a federated learning system with $\nclients$ agents. We assume that, instead of storing the entire dataset on a centralized node, each agent $i \in [\nclients]$ owns a local calibration set $\mathcal{D}_{i} = \acn{\prn{X_k^{i}, Y_k^{i}}}_{k = 1}^{\ccount{i}}$, where $\ccount{i}$ is the number of calibration samples for the agent $i$.
  We further assume that the calibration data are \iid\ and that the statistical heterogeneity is due to \emph{label shifts}:
  \begin{equation*}
    \prn{X_k^{i}, Y_k^{i}} \sim \Pcal{i} = P_{X \mid Y} \times \Pycal{i},
  \end{equation*}
  where $P_{X \mid Y}$, the conditional distribution of the feature given the label, is assumed identical among agents but $\Pycal{i}$, the prior label distribution, may differ across agents. In federated learning, statistical heterogeneity is the rule rather than the exception, and it is essential to take into account the presence of label shift at the agent level. We assume that a predictive model $\hat{f}$ has been learned by federated learning. The results we present are agnostic to the learning procedure.

  For an agent $\target \in [\nclients]$, and each $\alpha \in (0, 1)$, we are willing to compute a set-valued predictor, $\mathcal{C}_{\alpha}$ with confidence level $1-\alpha$, which depends on the calibration data of \textit{all the agents}. The  goal is to construct  informative conformal prediction sets for each agent, even when its calibration set is limited in size, by using the calibration data of all the agents participating in the FL; we stress that  the calibration data must always remain local to the networked agents.
  Most importantly, the resulting algorithm should attain both conformal and theoretical privacy guarantees -- matched to the privacy guarantees that can be obtained in the FL training procedure.

  Our main \textbf{contributions} to solving this challenging problem can be summarized as follows.
  \begin{itemize}[noitemsep,topsep=0pt,parsep=0pt,partopsep=0pt,leftmargin=*]
    \item We introduce a new method, \algopred, to construct conformal prediction sets in a federated learning context that addresses label shift between agents; see~\Cref{sec:approach}. \algopred\ is a federated learning algorithm based on federated computation of weighted quantiles of agent's non-conformity scores, where the weights reflect the label shift of each client with respect to the population. The quantiles are obtained by regularizing the pinball loss using Moreau-Yosida inf-convolution and a version of federated averaging procedure; see~\Cref{sec:privacy}.

    \item We establish conformal prediction guarantees, ensuring the validity of the resulting prediction sets. Additionally, we provide differential private guarantees for \algopred; see~\Cref{sec:theory}.

    \item We show that \algopred\ provides valid confidence sets and outperforms standard approaches in a series of experiments on simulated data and image classification datasets; see \Cref{sec:experiments}.
  \end{itemize}

\format{Related Works.}
  The construction of predictions sets with confidence guarantees has been the subject of much work, mostly in a centralized framework. The conformal framework, introduced in the pioneering works of~\cite{vovk1999machine} is appealing in its simplicity/flexibility; see e.g. \citep{angelopoulos2020uncertainty,fontana2023conformal} and the references therein.
  For exchangeable data, this framework provides a model-free methodology for constructing prediction sets that satisfy the desired coverage~\citep{shafer2008tutorial,papadopoulos2002inductive,fannjiang2022conformal,angelopoulos2022recommendation}.

  These results can also be extended to non-exchangeable data. A method has been developed for dealing with covariate shift~\citep{tibshirani2019conformal}. This method is based on evaluating the discrepancy between the distribution of the calibration data set $P^{\mathrm{cal}}$ and the test point distributed according to $P^{\target}$. Using an estimate of the Radon-Nikodym derivative $\rmd P^{\target}/ \rmd P^{\mathrm{cal}}$, a valid prediction set can be obtained by weighting the non-conformity scores. The seminal work of~\citep{tibshirani2019conformal} led to several improvements, either to form valid prediction sets as long as the $f$-divergence of the discrepancy remains small~\citep{cauchois2020robust}, or to formulate hypothesis tests under covariate shifts~\citep{hu2020distribution}.
  In addition,~\citet{gibbs2021adaptive} examine the shift in an online environment; and~\citet{lei2021conformal} show the validity of the prediction sets even when the distributional shift is only approximated. Since many real-world data sets do not satisfy exchangeability, valid prediction sets are developed in~\citep{barber2022conformal} that put more mass around the point of interest.

  Conformal methods adapted to label shift are considered in~\citep{podkopaev2021distribution,podkopaev2021tracking} and have similar guarantees to those in~\citep[Corollary 1]{tibshirani2019conformal}.
  Methods for detecting and quantifying label shift have been proposed in~\citep{lipton2018detecting,garg2020unified}.   
  
  Differentially private quantiles can be derived based either on the exponential or Gaussian mechanisms~\citep{gillenwater2021differentially,pillutla2022differentially}. Using the exponential mechanism, valid prediction sets are generated in~\citep{angelopoulos2021private}. However, quantile computation in a federated learning environment remains a challenge. A first federated approach based on quantile averaging was proposed in~\citep{lu2021distribution}. However, this work does not provide theoretical guarantees, and the proposed method is vulnerable to distribution shifts. For federated deep learning, the differentially private versions are based on various techniques combination like gradient clipping and the addition of random noise~\cite{triastcyn2019federated,wei2020federated}.

\format{Notation.}
  Denote by $[\nclients]$ the set $\acn{1,\ldots,\nclients}$ and consider a finite number of labels, i.e., $\absn{\YC} < \infty$.
  Each agent $i \in [\nclients]$ has $\clcount{y}{i}$ calibration samples of label $y \in \YC$, and denote $\lcount{y} = \sum_{i = 1}^{\nclients} \clcount{y}{i}$ their total number over all the calibration examples.
  Recall that $\ccount{i}$ the total number of calibration samples on agent $i$, i.e., $\ccount{i} = \sum_{y \in \YC} \clcount{y}{i}$.
  Define the total number of calibration data points $\tcount := \sum_{i = 1}^{\nclients}\sum_{y \in \YC} \clcount{y}{i}$.
  For a cumulative distribution function $F$ and $\beta\in[0,1]$, define by $\q{\beta}(F) :=\inf \{z\colon F(z) \ge \beta\}$ the $\beta$-quantile.
  Finally, for $v\in\R$ denote by $\delta_v$ the point-mass distribution.

\section{Conformal Prediction for Federated Systems under Label Shift}
\label{sec:approach}

\format{Non-exchangeable data.}
  In this section, we explain how to take advantage of calibration data to obtain a valid $(1 - \alpha)$-prediction set.
  Consider the calibration dataset $\acn{(X_k^i,Y_k^i)\colon k\in[\ccount{i}]}_{i\in[\nclients]}$ with data distributed according to $\acn{\Pcal{i}}_{i\in[\nclients]}$.
  For $\acn{\pi_i}_{i\in[\nclients]}\in\Delta_{\nclients}$ we define the mixture distribution of labels given for $y\in\YC$ by
  \begin{equation*}\label{eq:mixture-weights}
    \textstyle P_{Y}^{\mathrm{cal}}(y) = \sum_{i=1}^{\nclients} \pi_i P_{Y}^{i}(y).
  \end{equation*}
  Our goal is to determine a set of likely outputs for a new data point $(\Xtarget,\Ytarget)$ drawn on agent $\target\in \cset$ from the distribution $P^{\target}$.
  The conformal approach relies on non-conformity scores $V_{k}^{i} = V(X_k^{i}, Y_k^{i})\in[0,1]$, $i\in[\nclients]$, $k\in[\ccount{i}]$ to determine the prediction set -- see~\citep{shafer2008tutorial}. These non-conformity scores are uniformly weighted to generate the conventional prediction set
  \begin{equation}\label{eq:def:predset:global}
    \begin{aligned}
      &
      \textstyle
      \predsetg{\alpha}{\Bar{\mu}}(\xquery) = \ac{\yquery\in\YC\colon V(\xquery, \yquery) \le \q{1 - \alpha}\pr{\Bar{\mu}}},\\
      &
      \textstyle
      \Bar{\mu} = \prn{\tcount+1}^{-1}\prn{\sum_{i=1}^{\nclients}\sum_{k=1}^{\ccount{i}}\delta_{V_{k}^{i}} + \delta_{1}}.
    \end{aligned}
  \end{equation}
  However, this method can lead to significant under-coverage in the presence of label shift~\citep{podkopaev2021distribution}. In fact, since the data $\{(X_k^i, Y_k^i)\colon k\in\ccount{i}\}_{i\in[\nclients]}$ are often not exchangeable, it is required to correct the quantile to account for label shift
  to obtain valid prediction sets~\citep{tibshirani2019conformal}.
  As proposed by~\citet{podkopaev2021distribution}, we begin by assuming that, for all $i\in[\nclients]$ and $y \in \YC$, we have access to the likelihood ratios:
  \begin{align*}
    \iweightc{y}{i} = \Pycal{i}(y)/P_{Y}^{\mathrm{cal}}(y).
  \end{align*}
  Denote by $\mathcal{I}=\{(i,k)\colon i \in [\nclients], k\in[\ccount{i}]\} \cup \acn{(\target, \ccount{\target}+1)}$.
  Using the weights $\{W_k^i\colon (i,k)\in\mathcal{I}\}$ provided in~\eqref{eq:def:Wki-tibshirani-FL}, the non-exchangeability correction of~\citet{tibshirani2019conformal} is given for any $\yquery\in\YC$ by
  \begin{align}\label{eq:def:ptarget-exact:1}
    &\qweightp{Y_k^i}{\yquery} = \frac{W_{k}^{i}}{W_{\ccount{\target} + 1}^{\target} + \sum_{j=1}^{\nclients} \sum_{l=1}^{\ccount{j}} W_{l}^{j}},
    \\
    \nonumber
    &\qdistrp{\yquery} = \qweightp{\yquery}{\yquery} \delta_{1} + \sum_{i=1}^{\nclients} \sum_{k = 1}^{\ccount{i}} \qweightp{Y_k^i}{\yquery} \delta_{V_{k}^i}.
  \end{align}
  For any covariate $\xquery\in\XC$, define the $(1 - \alpha)$-prediction set with \emph{oracle weights}
  \begin{equation*}\label{eq:oracle-confidence-sets:general}
    \predsetg{\alpha}{\qdistrp{}}(\xquery) = \ac{\yquery \in \YC\colon V(\xquery, \yquery) \le \q{1 - \alpha}\pr{\qdistrp{\yquery}}}.
  \end{equation*}
  In contrast to the exchangeable setting, the quantile is calculated based on a weighted empirical distribution depending on $\yquery$.
  The validity of the prediction set is based on the concept of weighted exchangeability, which was introduced in~\citep[Definition~1]{tibshirani2019conformal}; see also~\citep[Theorem~2]{podkopaev2021distribution}.
  In the following, we will suppose that the next assumption holds.
  \begin{assumption}\label{ass:iid-noties}
    The calibration data points $\{(X_k^i, Y_k^i)\colon (i, k) \in \mathcal{I}\}$ are pairwise independent, and there are no ties between $\{V_k^i\colon (i,k)\in\mathcal{I}\}$ almost surely.
  \end{assumption}
  \begin{theorem}\label{thm:hatYinhatC}
    If \Cref{ass:iid-noties} holds, then for any $\alpha \in [0, 1)$, we have
    \begin{multline}\label{eq:bound:coverage:general2}
        1 - \alpha
        \le \prob\pr{\Ytarget \in \predsetg{\alpha}{\qdistrp{}}(\Xtarget)}
        \\
        \le 1 - \alpha + \E\br{\max_{(i, k) \in \mathcal{I}} \Bigl\{\qweightp{Y_k^i}{\Ytarget}\Bigr\}},
    \end{multline}
    where $\qweightp{Y_k^i}{\Ytarget}$ is defined in~\eqref{eq:def:ptarget-exact:1}.
  \end{theorem}
  This theorem is directly adapted from~\citep[Corollary~1]{tibshirani2019conformal}. For completeness, a formal proof is postponed to \Cref{subsec:predsec:Ywidehat}.
  It is important to note that the lower bound in~\eqref{eq:bound:coverage:general2} holds even in the presence of ties between non-conformity scores. Although \Cref{thm:hatYinhatC} guarantees the validity of $\predsetg{\alpha}{\qdistrp{}}(\Xtarget)$, this prediction set requires the challenging computation of the weights $\qweightp{y}{\yquery}$. Indeed, the calculation of $W_{y,\yquery}$ requires the summation over $\tcount!$ elements. The first key contribution of our work is given in \Cref{thm:YinCN-approx}, where we show that alternative weights, which are easier to compute, can lead to valid prediction sets. Specifically, the new weights $\qweightb{y}{\yquery}$ are computed on a smaller number of data points $\Bar{\tcount} \le \tcount$, which are randomly selected based on a multinomial random variable with parameter $(\Bar{\tcount}, \acn{\pi_{i}}_{i \in [\nclients]})$.
  Actually, we denote by $\Bar{\tcount}^{i}$ the multinomial count associated with agent $i$. We take $\Bar{\tcount}^{i}\wedge \ccount{i}$ calibration data from agent $i$ and denote $V_k^i = V(X_k^i,Y_k^i)$.
  The core idea is to generate a smaller calibration dataset $\{ (X_k^i,Y_k^i)\colon k\le \ccount{i}\wedge\Bar{\tcount}^{i} , i\in[\nclients] \}$ of i.i.d. variables distributed according to $P_{X \mid Y} \times P_{Y}^{\mathrm{cal}}$.
  For any label $y\in\YC$, the weight $\qweightb{y}{\yquery}$ is given by:
  \begin{equation}\label{eq:def:qweightb}
      \qweightb{y}{\yquery}
      = \frac{\iweightp{y}}{\iweightp{\yquery} + \sum_{i=1}^{\nclients}\sum_{k=1}^{\ccount{i}\wedge \Bar{\tcount}^{i}}\iweightp{Y_{k}^{i}}}.
  \end{equation}
  In addition, consider the following prediction set
  \begin{equation}\label{eq:def:muBar-CBar}
    \begin{aligned}
      &\textstyle\Bar{\mu}_{\yquery}^{\target} = \qweightb{\yquery}{\yquery} \delta_{1} + \sum_{i=1}^{\nclients} \sum_{k=1}^{\ccount{i}\wedge \Bar{\tcount}^i} \qweightb{Y_k^i}{\yquery} \delta_{V_k^i},
      \\
      &\textstyle\mathcal{C}_{\alpha,\Bar{\mu}^{\target}}(\xquery)
      = \ac{\yquery\in\YC\colon V(\xquery,\yquery) \le \q{1 - \alpha}\prn{\Bar{\mu}_{\yquery}^{\target}}}
      .
    \end{aligned}
  \end{equation}
  Denote by $\normn{\iweightp{}}_{\infty}=\max_{y\in\YC}\{\iweightp{y}\}$.
  Using the new prediction set $\mathcal{C}_{\alpha,\Bar{\mu}^{\target}}$, we obtain the following result.
  \begin{theorem}\label{thm:YinCN-approx}
    Assume \Cref{ass:iid-noties}. Set $\Bar{\tcount}=\lfloor\tcount/2\rfloor$ and $\pi_{i} = \ccount{i}/\tcount$, for any $i\in[\nclients]$.
    Then,
    \begin{multline*}
      \abs{\prob\pr{\Ytarget \in \mathcal{C}_{\alpha,\Bar{\mu}^{\target}}(\Xtarget)}
      - 1 + \alpha}
      \le
      \frac{6}{\tcount}
      \\
      + \frac{36 + 6 \log \tcount}{\tcount} \normn{\iweightp{}}_{\infty}^2
      + \frac{14\log \tcount}{\tcount} \textstyle\sum_{i \colon \frac{\ccount{i}}{12} < \log \tcount} \sqrt{\ccount{i}}.
    \end{multline*}
  \end{theorem}
  The preceding theorem shows that $\mathcal{C}_{\alpha,\Bar{\mu}^{\target}}(\Xtarget)$ contains the true label $\Ytarget$ with probability \emph{close to} $1-\alpha$.
  If $\nclients=1$ and $\tcount\ge 46$, the set $\{i\in[\nclients]\colon \ccount{i} < 12\log \tcount\}$ is empty. In this case, the convergence rate reduces to $\tcount^{-1} \log \tcount$.
  More precisely, if each agent has the same number of calibration data, the convergence rate $\tcount^{-1} \log \tcount$ is ensured when $\tcount \ge 12 \nclients \log \tcount$. This is for example the case when $\ccount{i}= 200$ and $\nclients \le 86538$.
  On the other hand, if $\nclients=\tcount$, each agent has only one data point, and in this case the bound becomes $\tcount^{-1}\nclients (\log \tcount)^{3/2}$.

\format{Approximate Weights.}
Ideally, we would like to use the weights defined in equation~\eqref{eq:def:qweightb} to compute valid prediction sets. However, these weights depend on the probability distribution of the labels for each agent, which in many scenarios must be estimated (and therefore known up to an error).
Based on empirical estimation of these label probability distributions $\{\widehat{P}_{Y}^i\}_{i\in[\nclients]}$, for each label $y$, define the likelihood ratio as follows:
\begin{equation}\label{eq:def:wytarget}
  \begin{aligned}
    \iweight{y} = \frac{\widehat{P}_{Y}^{\target}\prn{y}}{\sum_{i = 1}^{\nclients} \pi_{i} \widehat{P}_{Y}^i(y)},
  \end{aligned}
\end{equation}
and denote by $\qweight{y}{\yquery}$ the weight defined in~\eqref{eq:def:qweightb} with $\iweightp{y}$ replaced by $\iweight{y}$.
We also consider $\widehat{\mu}_{\yquery}$ defined as in~\eqref{eq:def:muBar-CBar} with $\qweightb{y}{\yquery}$ replaced by $\qweight{y}{\yquery}$. The prediction set becomes
\begin{align}\label{eq:def:approx-muytarget}
  &\predsetg{\alpha}{\qdistr{}}(\xquery) = \ac{\yquery\in\YC\colon V(\xquery, \yquery) \le \q{1 - \alpha}\pr{\qdistr{\yquery}}}.
\end{align}
Since computing the exact weights $\qweightb{y}{\yquery}$ in~\eqref{eq:def:qweightb} may not be feasible, we consider the approximation $\qweight{y}{\yquery}$ given in~\eqref{eq:def:wytarget}.
We also construct a random variable $(\widehat{X}_{\ccount{\target}+1}^{\target}, \widehat{Y}_{\ccount{\target}+1}^{\target})$ as in~\citep{lei2021conformal} such that $\prob(\widehat{Y}_{\ccount{\target}+1}^{\target}=y)=[\sum_{\tilde{y} \in \YC}\iweight{\tilde{y}}P_{Y}^{\mathrm{cal}}\prn{\tilde{y}}]^{-1} \iweight{y}P_{Y}^{\mathrm{cal}}\prn{y}$, where $\iweight{y}$ is defined in~\eqref{eq:def:wytarget}; and $\widehat{X}_{\ccount{\target}+1}^{\target}|\widehat{Y}_{\ccount{\target}+1}^{\target}$ is drawn according to $P_{X|Y}$. The validity of the resulting prediction set is established in \Cref{lem:YminushatY:general}. Note that this approach makes the weights' computation feasible, at the cost of introducing one additional approximation.
\begin{lemma}\label{lem:YminushatY:general}
  For any $\alpha \in (0, 1)$, we have
  \begin{multline}\label{eq:def:R}
    \hspace{-0.9em}
    \scalebox{.96}{$\displaystyle
    \abs{\prob\prn{\Ytarget \in \predsetg{\alpha}{\qdistr{}}(\Xtarget)} - \prob\prn{\widehat{Y}_{\ccount{\target}+1}^{\target} \in \predsetg{\alpha}{\qdistr{}}(\widehat{X}_{\ccount{\target}+1}^{\target})}}
    $}
    \\
    \scalebox{.96}{$\displaystyle
    \le \frac{1}{2} \sum_{y \in \YC} \abs{P_{Y}^{\target}(y) - \frac{\iweight{y} P_{Y}^{\mathrm{cal}}(y)}{\sum_{\tilde{y} \in \YC}\iweight{\tilde{y}}P_{Y}^{\mathrm{cal}}\prn{\tilde{y}}}}
    := \ratiotv,
    $}
  \end{multline}
  where $\iweight{y}$, $\predsetg{\alpha}{\qdistr{}}$ are defined in~\eqref{eq:def:wytarget} and~\eqref{eq:def:approx-muytarget}, respectively.
\end{lemma}
When $\iweight{y}$ is sufficiently close to $\iweightp{y}$, \Cref{lem:YminushatY:general} shows that the approximate weights generate accurate prediction sets (as discussed in \Cref{subsec:predsec:Ywidehat}).
The error disappears entirely when $\iweight{y}=\iweightp{y}$ for all $y\in\YC$.
Furthermore, using~\citep[Corollary~1]{tibshirani2019conformal}, we can establish that $\widehat{Y}_{\ccount{\target}+1}^{\target}\in\predsetg{\alpha}{\qdistr{}}(\widehat{X}_{\ccount{\target}+1}^{\target})$ with probability nearly $1-\alpha$.
Finally, similar ideas that developed for \Cref{thm:YinCN-approx} on $\Ytarget$, in conjunction with \Cref{lem:YminushatY:general}, give a more accurate bound on the coverage validity.

\begin{theorem}\label{cor:main-coverage}
  Assume \Cref{ass:iid-noties}. For any $i\in[\nclients]$, set $\pi_{i} = \ccount{i}/\tcount$ and take $\Bar{\tcount}=\lfloor\tcount/2\rfloor$.
  Then,
  \begin{multline*}
      \abs{\prob\pr{\Ytarget \in \mathcal{C}_{\alpha,\qdistr{}}(\Xtarget)}
      - 1 + \alpha}
      \le \frac{36 \normn{\iweight{}}_{\infty}^2}{\tcount (\E\iweight{Y^{\mathrm{cal}}})^2}
      \\
      + \ratiotv
      + \frac{6}{\tcount}
      + \frac{2\log \tcount}{\tcount} \pr{\frac{3\normn{\iweight{}}_{\infty}^2}{\prn{\E\iweight{Y^{\mathrm{cal}}}}^{2}} \vee 7 \textstyle\sum_{i \colon \frac{\ccount{i}}{12} < \log \tcount} \sqrt{\ccount{i}}}
      ,
  \end{multline*}
  where $\iweight{Y_k^i}$, $\ratiotv$ are defined in~\eqref{eq:def:wytarget}-\eqref{eq:def:R} and $Y^\mathrm{cal} \sim P_{Y}^{\mathrm{cal}}$.
\end{theorem}
This theorem controls the probability of coverage with a bound depending on the approximated likelihood ratio. A formal proof can be found in~\Cref{sec:cov-tv}.
This result demonstrates that it is essential to include all agents with the most data. However, it also highlights a counterproductive effect when incorporating agents with few data.

\format{Maximum Likelihood Estimation Weights.}
  Denote by $\Mclcount{y}{i}$ the number of training data on agent $i$ associated to label $y$. Consider the total number of local data $\Mccount{\target} = \sum_{y\in\YC} \Mclcount{y}{i}$, the number of training data with label $y$ written by $\Mlcount{y} = \sum_{i=1}^{\nclients} \Mclcount{y}{i}$, and the total number of samples on all agents by $\Mtcount = \sum_{y\in\YC} \Mlcount{y}$.
  When each agent independently learns its approximate label distribution based on counting the number of label in its training datasets, then $\widehat{P}_{Y}^i(y)=\1_{\Mlcount{y}\ge 1} \Mclcount{y}{i}/\Mccount{i}$.
  Therefore, the empirical counterpart of~\eqref{eq:def:wytarget} is given for any labels $(y, \yquery) \in \YC^2$ by
  \begin{align}\label{eq:def:approx-probyy}
    \iweight{y} =
    \frac{\Mtcount \Mclcount{y}{\target}}{\Mccount{\target} \Mlcount{y}} \1_{\Mlcount{y}\ge 1}.
  \end{align}
  All the results in this article are given conditionally to the training dataset, meaning that they hold regardless of the specific training data.
  In order to determine the order of magnitude of the bound of \Cref{cor:main-coverage}, we analyze the average value of $\ratiotv$.
  Given the number of training samples $\{\Mccount{i}\}_{i \in [\nclients]}$, if we assume that each training point $(X_k^i,Y_k^i)$ is distributed according to $P^{i}$, then taking the expectation over the training set yields:
  \begin{equation}\label{eq:bound:YwidehatY}
    \E\br{\ratiotv}
    \le \frac{6}{\sqrt{\Mccount{\target}}} + 12\sqrt{\frac{\log \absn{\YC} + \log \Mccount{\target}}{\Mtcount \min_{y \in \YC}{P_{Y}^{\mathrm{cal}}\prn{y}}}}
    .
  \end{equation}
 The proof is given in \Cref{subsec:predsec:YwidehatY}.
 Interestingly, if the previous upper bound is plugged in \Cref{cor:main-coverage} instead of \Cref{lem:YminushatY:general}, then the leading error of order $\Oh(\Mccount{\target}{}^{-1/2} \vee \tcount^{-1}\log \tcount)$ is due to the weights' estimates $\acn{\qweight{y}{\yquery}}_{y\in\YC}$.
 This bound shows that we should not attempt to estimate the likelihood ratios for a single agent, especially when the square root number of local training data on agent $\target$ is small compared to the number of calibration data. Rather, we need to do this for a group of agents that have approximately the same distribution, which will give us more stable estimators.
 The agent can benefit from learning simultaneous tasks by exploiting common structures~\citep{caruana1998multitask}.

\section{Privacy Preserving Federated CP}
\label{sec:privacy}

In the previous section, we constructed prediction sets that were valid in theory. However, their practical implementation in a federated environment posed challenges due to the reliance on estimations that are difficult to evaluate. In particular, estimating $\q{1 - \alpha}\pr{\qdistr{\yquery}}$ in order to derive the prediction set $\predsetg{\alpha}{\qdistr{}}(\xquery)$, defined in~\eqref{eq:def:approx-muytarget}, is challenging because it requires knowledge of the global distribution $\qdistr{y}$.
This section is divided into two parts: (1) a new method is developed, called \algopred, for estimating quantiles under the federated constraints; (2) then, a method for computing probabilities $\acn{\qweight{y}{\yquery}}_{y, \yquery \in \YC}$ with differential privacy (DP) guarantees is presented.

\format{Quantile Regression and Moreau-Yosida Regularization.}
  Let $\alpha \in (0, 1)$, we now propose to estimate the weighted $(1 - \alpha)$-quantile of $\qdistr{\yquery}$ defined in~\eqref{eq:def:lambdai-mui}.
  To this end, we develop a federated optimization algorithm based on ``pinball loss'' minimization, a quantile regression techniques with asymmetric penalties~\citep{koenker2001quantile}. For $v \in \R$ and $q \in \R$ define the pinball loss as
  \begin{equation*}
    \pinball[\alpha][v](q)
    = (1 - \alpha) (v - q) \indiacc{v \ge q} + \alpha (q - v) \indiacc{q > v}.
  \end{equation*}
  For any $\yquery \in \YC$, the $(1 - \alpha)$-quantile of $\widehat{\mu}_{\yquery}$ is given by
  \begin{equation}\label{eq:def:QuantilePinball}
    \q{1 - \alpha}\pr{\widehat{\mu}_{\yquery}}
    \in \argmin_{q \in \R}\ac{\E_{V \sim \widehat{\mu}_{\yquery}}\br{\pinball[\alpha][V](q)}};
  \end{equation}
  e.g. see~\citep{buhai2005quantile}. The pinball loss $\pinball[\alpha][v]$ is lower semi-continuous but not differentiable on $\R$. Hence, we consider the Moreau-Yosida inf-convolution (or \emph{envelope}) $\pinball[\alpha][v]^{\gamma}$ instead of $\pinball[\alpha][v]$ -- where $\gamma$ is the regularization parameter; see e.g. \citep{moreau1963proprietes} and~\citep[Chapter~3]{parikh2014proximal}, whose expression is given by
  \begin{equation*}
    \pinballmoreau{v}{q}
    = \min_{\tilde{q} \in \R}\Big\{\pinball[\alpha][v](\tilde{q}) + \frac{1}{2\gamma}(\tilde{q} - q)^2\Big\}.
  \end{equation*}
  The function $\pinball[\alpha][v]^{\gamma}(\cdot)$ has an explicit expression given in~\eqref{eq:def:moreaupinball}.
  Note that the minima of $\pinball[\alpha][v]$ and $\pinball[\alpha][v]^{\gamma}$ coincide. 
  We obtain the weighted quantile by considering $\pinball[\alpha][v]^{\gamma}$ instead of $\pinball[\alpha][v]$.
  An important property is that the inf-convolution of a proper lower semicontinuous convex function is a differentiable function whose derivative is Lipschitz; see~\citep[Theorem~2.26]{rockafellar2009variational}.
  The original optimization problem given in~\eqref{eq:def:QuantilePinball} is replaced by a convex/smooth loss:
  \begin{equation}\label{eq:def:qmoreau}
    \qmoreau[1 - \alpha][\gamma]\pr{\widehat{\mu}_{\yquery}}
    \in \argmin_{\R} \acn{\lossglobal(q)},
  \end{equation}
  where $\lossglobal\colon \R \to \R_+$ is the function given by
  \begin{equation*}
    \lossglobal\colon q \mapsto \E_{V \sim \widehat{\mu}_{\yquery}}\brn{\pinballmoreau{V}{q}}.
  \end{equation*}
  For almost every value of $\alpha \in (0, 1)$, there exists a unique minimizer of $\lossglobal$.
  This minimizer $\qmoreau[1 - \alpha][\gamma]\prn{\widehat{\mu}_{\yquery}}$ of the regularized loss function deviates from the true quantile.
  However, the error is controlled by the regularization parameter $\gamma$ and is asymptotically exact when $\gamma\to 0$. More precisely (see \Cref{subsec:approx-moreau} for a proof) it holds that:
  \begin{theorem}\label{thm:moreau-approx}
    Let $\gamma>0$ and $\alpha \in (0,1)$.
    Assume that for all $\acn{y_{\ell}}_{\ell \in [\tcount+1]} \in \YC^{[\tcount+1]}$, $1 - \alpha \notin \acn{W_k / W_{\tcount + 1}}_{k \in [\tcount + 1]}$, where $W_k = \sum_{\ell = 1}^{k} \iweight{y_\ell}$.
    Then, we have
    $
      \big\vert\qmoreau[1 - \alpha][\gamma]\prn{\widehat{\mu}_{\yquery}}
      \textstyle
      - \q{1 - \alpha}\prn{\widehat{\mu}_{\yquery}}\big\vert
      \le \gamma.
    $
  \end{theorem}
  The condition on $\alpha$ assumed in \Cref{thm:moreau-approx} ensures the uniqueness of the minimizer of $\lossglobal$.

  \begin{algorithm}[t]
    \caption{\algo}
    \label{algo:Dp-moreau}
    \begin{algorithmic}[]
        \State {\bfseries Input:} initial quantile $q_{0}$, target significance level $\alpha$, number of rounds $T$, learning rate $\eta$, Moreau regularization parameter $\gamma$, local gradients $\acn{\nabla \lossi}_{i \in [\nclients]}$, local non-conformity scores $\acn{V_k^i}_{k \in [\ccount{i} + 1]}$, mixture weights $\acn{\lambda_{\yquery}^i}_{i \in [\nclients]}$, standard deviation of Gaussian mechanism $\noiseparamone$, number of local iterations $K$.
        \For{$t=0$ {\bfseries to} $T-1$}
            \State $S_{t + 1} \gets$ random subset of $[\nclients]$ \hfill\Comment{Server side}
            \For{each agent $i \in S_{t + 1}$} \hfill\Comment{In parallel}
                \State Initialize quantile $q_{t, 0}^i\gets q_t$
                \For{$k=0$ {\bfseries to} $K - 1$}
                    \State \Comment{Gradient with DP noise}
                    \State $g_{t,k}^i \gets \nabla\lossi\prn{q_{t,k}^i} + z_{t, k}^i$, $z_{t, k}^i\sim\gauss(0,\noiseparamone^2)$
                    \State \Comment{Update local quantile}
                    \State $q_{t, k + 1}^i \gets q_{t, k}^i - \eta g_{t, k}^i$
                \EndFor
            \State $(\Delta q_{t + 1}^i, \Delta \qalgoi[t + 1][i]) \gets (q_{t, K}^i - q_{t, 0}^i, \sum_{k \in [K]}\frac{q_{t, k}^i}{K})$
            \EndFor
            \State \Comment{On the central server}
            \State $q_{t + 1} \gets q_t + \frac{\nclients}{\absn{S_{t + 1}}} \sum_{i \in S_{t + 1}} \lambda_{\yquery}^i \Delta q_{t + 1}^i$
            \State $\qalgo[t + 1] \gets \frac{t}{t + 1} \qalgo[t] + \frac{\nclients}{\absn{S_{t + 1}}} \sum_{i \in S_{t + 1}} \frac{\lambda_{\yquery}^i \Delta \qalgoi[t + 1][i]}{t + 1}$
        \EndFor
        \State {\bfseries Output:}  $\qhatmoreau \gets \qalgo[T]$.
    \end{algorithmic}
  \end{algorithm}

\format{Federated quantile computation.}
  We now describe the Differentially Private Federated Average Quantile Estimation (\algo) algorithm (see \Cref{algo:Dp-moreau}), a novel method to compute quantile in a federated learning setting, with DP guarantees. We briefly described this method below. For each query $\yquery \in\YC$, we consider the distributions $\widehat{\mu}_{\yquery} = \sum_{i = 1}^n\lambda_{\yquery}^i \widehat{\mu}_{\yquery}^i$, where $\lambda_{\yquery}^i$ and $\widehat{\mu}_{\yquery}^i$ are given by
  \begin{equation}\label{eq:def:lambdai-mui}
    \begin{aligned}
      &\textstyle\lambda_{\yquery}^{i} = \frac{\ccount{i}}{\tcount} \qweight{\yquery}{\yquery} + \sum_{k = 1}^{\ccount{i}\wedge \Bar{\tcount}^i} \qweight{Y_k^{i}}{\yquery}, \\
      &\textstyle\widehat{\mu}_{\yquery}^{i} = \frac{\ccount{i} \qweight{\yquery}{\yquery}}{\lambda_{\yquery}^{i} \tcount} \delta_{1} + \sum_{k=1}^{\ccount{i}\wedge \Bar{\tcount}^i} \frac{\qweight{Y_k^i}{\yquery}}{\lambda_{\yquery}^{i}} \delta_{V_k^i}.
    \end{aligned}
  \end{equation}
  To simplify the notation, for any client $i \in [\nclients]$, we introduce the local loss function $\lossi\colon q\in\R \mapsto \E_{V \sim \qdistr{y}^{i}}\brn{\pinballmoreau{V}{q}}$; see~\eqref{eq:def:S-alpha-i-gamma} for explicit expression.

  At each iteration $t \in [T]$, the server subsamples the participating agents $S_{t + 1} \subseteq [\nclients]$ independently from the past.
  Each selected agent $i \in S_{t + 1}$ performs $K$ local updates: (1) they independently compute their local gradient;
  (2) a Gaussian noise is added as in~\eqref{eq:def:gradnoisy} to ensure the differential privacy. More precisely, for agent $i \in S_{t + 1}$, at local iteration $k \in \acn{0, \ldots, K - 1}$, we define:
  \begin{equation}\label{eq:def:gradnoisy}
    g_{t, k}^{i}
    = \nabla\lossi\prn{q_{t, k}^i}
    + z_{t, k}^i,
  \end{equation}
  where $\acn{z_{t, k}^i\colon (t,k) \in \acn{0, \ldots, T - 1} \times [K]}_{i \in [\nclients]}$ are i.i.d. Gaussian random variables with zero mean and variance $\noiseparamone^2$.
  For any agent $i \in S_{t+1}$,
  $g_{t,k}^i$ is an unbiased estimate of $\nabla\lossi\prn{q_{t,k}^i}$.
  (3) The participating agents update their local quantiles $q_{t,k+1}^i\gets q_{t,k}^i - \eta g_{t,k}^i$, where $\eta$ is a positive stepsize;
  (4) then transmit $(\Delta q_{t+1}^i, \Delta \qalgoi[t+1][i]) = \prn{q_{t,K}^i - q_{t,0}^i, \sum_{k\in[K]} q_{t,k}^i / K}$ to the central server.
  The parameter $\Delta q_{t+1}^i$ is used to update the common parameter $q_{t}$, while $\Delta \qalgoi[t+1][i]$ is necessary to keep track of the average of the sampled parameters denoted $\bar{q}_{t}$; see~\cite{nemirovski2009robust,bubeck2015convex}.
  (5) Finally, the server performs an online average to update $\bar{q}_{t}$ and computes the new parameter following
  \begin{equation*}
    \textstyle q_{t+1} = q_{t} + (\nofrac{\nclients}{\absn{S_{t + 1}}}) \sum_{i\in S_{t + 1}} \Delta q_{t+1}^i .
  \end{equation*}
  At the final stage, the central server output the quantile estimate is given by
  \begin{equation}\label{eq:def:qhatmoreau}
    \textstyle\qhatmoreau = \sum_{t = 1}^{T} (\nofrac{n}{\absn{S_{t}}}) \sum_{i\in S_{t}} \lambda_{\yquery}^i \Delta \qalgoi[t][i] / T.
  \end{equation}
  \Cref{algo:Dp-moreau} is a Federated Averaging procedure~\citep{mcmahan2017communication} applied to the Moreau envelope of the pinball loss. As we will see in \Cref{sec:theory}, the addition of an independent Gaussian noise on the parameter at each update round provides  differential privacy guarantees; see \Cref{thm:privacy:fedavg} for more details.

  \begin{remark}
    Privacy is also at risk when computing probabilities $\acn{\qweight{y}{\yquery}}_{y, \yquery \in \YC}$.
    To compute the probabilities $\acn{\qweight{y}{\yquery}}_{y, \yquery \in \YC}$ while preserving privacy, we need specific mechanisms to transmit the number of training labels $\prn{\Mclcount{y}{i}}_{y \in \YC}$ from each agent $i$ to the server. For this purpose, we use the method proposed in~\citep{canonne2020discrete}. The idea is to add a discrete noise to the counts $\acn{\Mclcount{y}{i}\colon i \in [\nclients]}_{y \in \YC}$ and then transmit these noisy proxies.
    The resulting algorithm that combines the differentially-private count queries and federated quantile computation is given in \Cref{algo:conf-fed}.
  \end{remark}

  \begin{remark}
    \Cref{algo:conf-fed} is designed to build a confidence set for the single agent $\target$. By vectorizing all computations, the algorithm can be scaled to compute a confidence set for each agent. This would result in an algorithm that remains linear in the number of clients but would be more efficient than computing several independent runs. From a practical perspective, complexity can be further improved by clustering clients into groups based on their label distributions and performing conformal prediction on a group level.
  \end{remark}

  \begin{remark}
    The local loss functions $\lossi$ are expressed as the expectation of pinball loss functions. Since the sensitivity of these pinball loss functions is 1, there is no need to clip the gradient. It is sufficient adding Gaussian noise $\gauss(0,\sigma_g^2)$ to guarantee differential privacy. The value of $\sigma_g$ is chosen to provide a suitable trade-off between privacy and utility, balancing the need for strong privacy protection with useful outputs.
    For an explicit setting of $\sigma_g$, refer to \Cref{thm:privacy:fedavg}.
  \end{remark}

  \begin{algorithm}[t]
    \caption{\algopred}
    \label{algo:conf-fed}
    \begin{algorithmic}[]
    \State {\bfseries Input:}
    calibration dataset $\acn{(X_k^i, Y_{k}^i)\colon k\in[\ccount{i}]}_{i \in [\nclients]}$,
    covariate $\xquery$, communication round number $T$, subsampling number $\Bar{\tcount}$, Gaussian noise parameters $\noiseparamtewo\ge 0$.
    \For{each agent $i\in [\nclients]\cup\acn{\target}$} \hfill\Comment{In parallel}
      \State Set $\forall y \in\YC, \Mclcount{y}{i}\gets\text{number train data with label $y$}$
      \State Generate $\acn{\privnoise{y}{i}}_{y \in \YC}$ i.i.d. according to $\gauss_{\Z}\pr{0, \noiseparamtewo^2}$
      \State Send $\forall y \in \YC, \privclcount{y}{i} \gets \max(1, \Mclcount{y}{i} + \privnoise{y}{i}$)
      \State Compute \& Send $\acn{V(X_{k}^i, Y_{k}^i)\colon k \in [\ccount{i}]}_{i\in[\nclients]}$
    \EndFor
      \State \Comment{On the central server}
      \State Aggregate $\privlcount{y} \gets \sum_{i\in[\nclients]} \privclcount{y}{i}, \forall y\in\YC$
      \State Aggregate $\privtcount \gets \sum_{y \in \YC} \privlcount{y}$
      \For{each query $\yquery \in \YC$}
        \State Sample $\acn{\Bar{\tcount}^{i}}_{i\in[\nclients]} \sim {\mathcal{M}ulti} (\Bar{\tcount},\acn{\ccount{i}/\tcount}_{i\in[\nclients]})$
        \State Compute
        $\qweight{y}{\yquery}$ as in~\eqref{eq:def:qweightb} with ${\iweight{y}}$ given in~\eqref{eq:def:approx-probyy}
        \State Compute $\qhatmoreau\gets$\algo\
      \EndFor
    \State{\bfseries Output:} $\predsethatmoreau{\alpha}(\xquery) \gets \acbig{\yquery\colon V(\xquery, \yquery) \le \qhatmoreau}$.
    \end{algorithmic}
  \end{algorithm}

\section{Theoretical Guarantees}
\label{sec:theory}

\format{Convergence guarantee.}
We provide a convergence guarantee for \algo. Details of the proofs can be found in the supplementary paper. We show the convergence of $\acn{\qhatmoreau[1-\alpha][\gamma][t]}_{t\in\N}$ to a minimizer which is unique under the assumptions discussed in \Cref{subsec:approx-moreau}.
We briefly sketch key steps from the theoretical derivations, since the local loss functions $\acn{\lossi}_{i \in [\nclients]}$ have different minimizers, this client drift/heterogeneity may slow down the convergence~\citep{li2019convergence}. This dissimilarity is evaluated by the parameter $\zeta\ge 0$, which is given by
\begin{equation*}
  \textstyle \zeta = \max_{i \in [\nclients]}{\normn{\nabla \lossi - \nabla \lossglobal}_{\infty}^{1/2}}.
\end{equation*}
The convergence analysis is performed for the estimate parameter $\qhatmoreau[1 - \alpha][\gamma]$ given in~\eqref{eq:def:qhatmoreau}.
We provide below the statements without subsampling, i.e. $S_t = [\nclients]$, given in \Cref{sec:convergence-FedMoreau}.
Recall that $\qmoreau[1 - \alpha][\gamma]\prn{\widehat{\mu}_{\yquery}}$ is provided in~\eqref{eq:def:qmoreau} and denote $\dinit = \E_{q_0}\normn{q_{0} - \qbetagammay[1 - \alpha][\gamma]}^2$.
The following results hold with fixed train/calibration datasets ($\mathcal{D}_{\mathrm{train}},\mathcal{D}_{\mathrm{cal}}$), and define their union by $\mathcal{D}=\mathcal{D}_{\mathrm{train}}\cup\mathcal{D}_{\mathrm{cal}}$.
\begin{theorem}\label{thm:cv-fedmoreau}
  Let $\gamma\in(0,1]$, $S_t = [\nclients]$ and consider the stepsize $\eta\in (0, \gamma / 10]$.
  Then, for $t\in\acn{0, \ldots, T - 1}$, $k \in \acn{0, \ldots, K - 1}$, we have
  \begin{multline*}
    \E\brbig{\lossglobal(\qhatmoreau[1 - \alpha][\gamma]) \,\vert\, \mathcal{D}} - \lossglobal\prn{\qbetagammay[1 - \alpha][\gamma]}
    \\
    \le \prn{\eta K T}^{-1} {\dinit}
    + 14 \gamma^{-1} \eta^2 K \prn{\noiseparamone^2 + K \zeta^2}.
  \end{multline*}
\end{theorem}
The presence of heterogeneity among local datasets significantly influences convergence dynamics, particularly when the number of targets $K$, is significantly larger than 1. In such cases, the term $K^2 \zeta^2$ poses challenges by potentially hindering the effectiveness of numerous local steps.
Consider the stepsize $\eta_\star$ defined by
  \begin{equation*}
    \eta_\star=\min \ac{\frac{\gamma}{10}, \prbigg{\frac{\gamma \dinit}{13 K^2 T (\noiseparamone^2 + \zeta^2 K)}}^{1/3}}.
  \end{equation*}
Setting $\eta=\eta_{\star}$, we obtain the following result.
\begin{corollary}\label{cor:main:fedmoreau-cv}
  Let $\gamma\in(0,1]$, $S_t=[\nclients]$ and consider the stepsize $\eta_\star$.
  Then, for any $t\in\acn{0,\ldots,T-1}$, $k\in\acn{0,\ldots,K-1}$, we have
  \begin{multline}\label{eq:def:epsilon-FLerror}
      \epsilon_{\mathrm{optim}}^{(\gamma)}
      = \E\brbig{\lossglobal \prn{\qhatmoreau[1 - \alpha][\gamma]} \,\vert\, \mathcal{D}} - \lossglobal\prn{\qbetagammay[1 - \alpha][\gamma]}
      \\
      \le \frac{10 \dinit}{\gamma K T}
      + \frac{5 \pr{\noiseparamone^2 + \zeta^2 K}^{1/3} \dinit^{2/3}}{(\gamma K T^{2})^{1/3}}.
  \end{multline}
\end{corollary}
As shown in \Cref{cor:main:fedmoreau-cv}, $\epsilon_{\mathrm{optim}}^{(\gamma)}$ increases inversely proportional to $\gamma$.  The smaller the regularization parameter $\gamma$, the smaller the stepsize $\eta$ must be, and the more iterations are required to achieve the same accuracy. However, the error caused by the Moreau envelope vanishes for $\gamma \downarrow 0^+$, i.e. $\qbetagammay[1 - \alpha][\gamma]$ approaches $\q{1 - \alpha}\prn{\widehat{\mu}_{\yquery}}$.
Thus, there is a tradeoff between the accuracy of the quantile approximation $\qhatmoreau$ and the computational cost.

\vspace{0em}
\format{Conformal guarantees for \algopred{}.}
We show that the confidence set $\predsethatmoreau{\alpha}(\Xtarget)$ provided by~\algopred\ constitutes valid coverage of $\Ytarget$.
The theoretical derivations and complete statements are given in \Cref{sec:coverage-fedmoreau}.  
For all $i\in[\nclients]$, denote by $P_V^i$ the distribution of $V(X^i,Y^i)$ where $(X_i,Y_i)\sim P^i$, and consider $Y^{\mathrm{cal}}\sim P_Y^{\mathrm{cal}}$.
\begin{theorem}\label{thm:right-coverage}
  Assume there exist $m,M>0$ such that for any $i\in[\nclients]$, $P_{V}^i$ admits a density $f_V^i$ with respect to the Lebesgue measure that satisfies $m \le f_{V}^i \le M$.
  For any $\alpha \in [0, 1] \setminus \Q$, it holds
  \begin{multline*}
      \hspace{-1.3em}
      \abs{\prob\prn{\Ytarget \in \predsethatmoreau{\alpha}(\Xtarget)} - \prob\prn{\Ytarget \in \mathcal{C}_{\alpha,\widehat{\mu}}(\Xtarget)}}
      \\
      \le 
      6 \Flip \sqrt{\frac{\log(\tcount) \sum_{y\in\YC} P_Y^{\mathrm{cal}}(y) \iweight{y}}{m \min_{y\in\YC}\iweight{y}} \pr{\E\brbig{\epsilon_{\mathrm{optim}}^{(\gamma)} | \mathcal{D}_{\mathrm{train}}} + \gamma}}
      \\
      + \frac{2 \Flip \log \tcount}{m \tcount}
      + \frac{4 \var(\iweight{Y^{\mathrm{cal}}})}{\tcount (\E\iweight{Y^{\mathrm{cal}}})^2} + \frac{2 \E \iweight{\Ytarget}}{\tcount \E\iweight{Y^{\mathrm{cal}}}}
      + \frac{m}{2\tcount \log \tcount} + \frac{1}{\tcount^{2}}
      ,
  \end{multline*}
  where $\epsilon_{\mathrm{optim}}^{(\gamma)}$ is defined in~\eqref{eq:def:epsilon-FLerror}.
\end{theorem}

This result illustrates an interesting tradeoff introduced by the regularization parameter $\gamma$.
As shown in \Cref{cor:main:fedmoreau-cv}, $\epsilon_{\mathrm{optim}}^{(\gamma)}$ increases inversely proportional to $\gamma$.
Therefore, setting $\gamma\approx T^{-1/2}$ ensures a convergence rate of order $T^{-1/4}$ for the optimization procedure.
In this case, the error term of order $\Oh(\tcount^{-1} \log\tcount)$ is guaranteed by choosing the number of iterations $T \approx \tcount^{4}$.
The condition $\alpha \in [0,1]\setminus \Q$ is a strong but unnecessary assumption. However, it provides a simple way to ensure that $\widehat{\mu}_{\yquery}$ has no jump at level $1-\alpha$.
Interestingly, the same condition on $\alpha$ is used in~\citep[Corollary~1]{podkopaev2021distribution}, where the authors explain why this condition cannot be avoided to ensure the consistency of the empirical quantile estimator.

\vspace{0em}
\format{Differential privacy guarantees.}  
  The $(\epsilon,\delta)$-differentially private nature of \algo\ relies on two components: the additional Gaussian noise, combined with the bounded gradient which avoids extreme values/outliers.
  The parameter $\epsilon$ controls the level of privacy protection provided by a differentially private algorithm, by limiting the probability of inferring any information about an individual in a given dataset. However, there is a small chance that the algorithm may leak some information, even though this probability is kept under control by the parameter $\delta$.
  Based on the R\'enyi differential privacy~\citep{mironov2017renyi}, joined to agent subsampling mechanism~\citep{balle2018privacy}, we establish the $(\epsilon, \delta)$-DP property following similar ideas to those of~\citep[Theorem~4.1]{noble2022differentially}.
  Detailed proof and definitions are provided in \Cref{sec:differential-privacy}.
  \begin{theorem}\label{thm:privacy:fedavg}
    If there is a constant number $S\in[\nclients]$ of sampled agents, i.e., $S_t=S$, for all $t \in [T]$.
    Then, for all $\epsilon > 0$ and $\delta \in (0, 1-(1+\sqrt{\epsilon})(1-S/\nclients)^{T})$, the \Cref{algo:Dp-moreau} is $(\epsilon, \delta)$-DP towards a third party when
    \begin{align*}
      &\noiseparamone \ge 2 \sqrt{\frac{K \max_{i\in[\nclients]} \lambda_{\yquery}^i}{\epsilon} \pr{1 + \frac{24 S \sqrt{T} \log(1/\Bar{\delta})}{\epsilon \nclients}}},
      \\
      &\text{where} \quad \Bar{\delta} = \frac{\nclients}{S} \br{1 - \pr{\frac{1-\delta}{1+\sqrt{\epsilon}}}^{1/T}}.
    \end{align*}
  \end{theorem}

\section{Numerical experiments}
\label{sec:experiments}

We conducted the experimental study of \algopred\ using both synthetic toy examples and real datasets. To perform a comprehensive evaluation, we compared our method with relevant baselines, namely \texttt{Unweighted Local} and \texttt{Unweighted Global} (see \Cref{sec:algorithms} for details).
The \texttt{Unweighted Local} method computes the quantile based on the local validation data of the agent $\target$ and derives the local unweighted prediction set with $(1 - \alpha)$ confidence level, given by
\begin{equation*}
  \unweightlocpredset{\alpha}(\xquery) = \ac{\yquery\in\YC\colon V(\xquery,\yquery)\le \q{1 - \alpha}\pr{\unweightlocdist}},
\end{equation*}
where $\unweightlocdist = \frac{1}{N^{\target} + 1}\sum_{k=1}^{N^{\target}}\delta_{V^{\target}_k} + \frac{1}{N^{\target} + 1} \delta_{1}$.
This method is the adaptive classification technique with split-conformal calibration applied to agent $\target$, as introduced in~\cite{romano2020classification} and also described in~\cite{angelopoulos2020uncertainty}.
On the other hand, the \texttt{Unweighted Global} method estimates the quantile based on aggregated non-conformity scores from all the agents, without taking into account the shift between calibration and target distributions.
This method computes the $(1 - \alpha)$-quantile in an analogous way to the ``classical'' conformal method recalled in \eqref{eq:def:predset:global}.

For our experiments, we apply split-conformal calibration on the entire dataset, which requires all agents to report their non-conformity scores to a central server.
We use the same non-conformity score $V(x, y)$ as considered in~\cite{romano2020classification,angelopoulos2020uncertainty}. Given the covariate $x$, the predictor $\hat{f}\colon \XC \to \Delta_{\absn{\YC}}$ estimates the probability of each class, and orders them from the most to the least likely label. The non-conformity score is then computed as the sum of all the probabilities greater than the true label $y$. 
Formally, the non-conformity scores are given by 
\begin{align*}\label{eq:con-conformity-score}
    \nonumber 
    &\textstyle\rho(X_k^i,Y_k^i) = \sum_{y\in\YC} \hat{f}(X_k^i)[y] \, \indiacc{\hat{f}(X_k^i)[y] > \hat{f}(X_k^i)[Y_k^i]}, \\
    &\textstyle V(X_k^i, Y_k^i) = \rho(X_k^i,Y_k^i) + U_k^i \times \hat{f}(X_k^i)[Y_k^i],
\end{align*}
where $U_k^i \in [0,1]$ is a uniform random variable.

\begin{figure}[t!]
  \centering
  \hspace*{.7em}
  \includegraphics[width=0.46\textwidth]{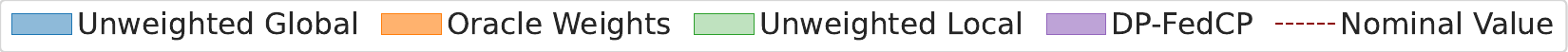}
  \begin{subfigure}{0.235\textwidth}
    \includegraphics[width=\textwidth, height=0.765\textwidth]{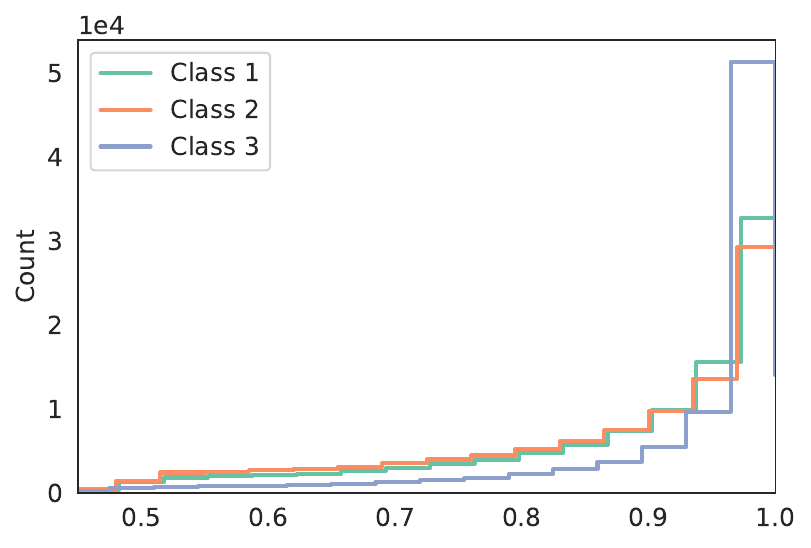}
    \caption{Non-conformity scores}
    \label{fig:toy-scores}
  \end{subfigure}
  \begin{subfigure}{0.235\textwidth}
    \includegraphics[width=\textwidth]{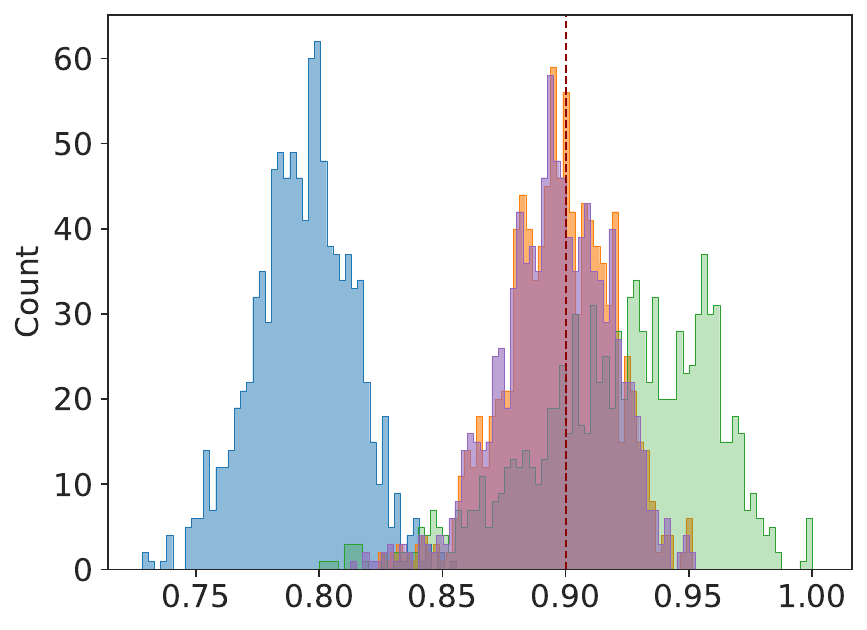}
    \caption{Empirical coverage}
    \label{fig:toy-empcov}
  \end{subfigure}
  \caption{Simulated data experiment with 2D data. Target confidence level $(1 - \alpha) = 0.9$.}
  \label{fig:toy-cov}
\end{figure}

\begin{figure}[t!]
  \centering
      \quad \includegraphics[width=1.\linewidth, height=0.04\linewidth]{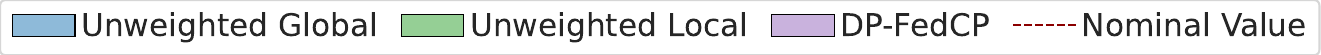}
      \centering
      \includegraphics[width=0.6\linewidth, height=0.35\linewidth]
      {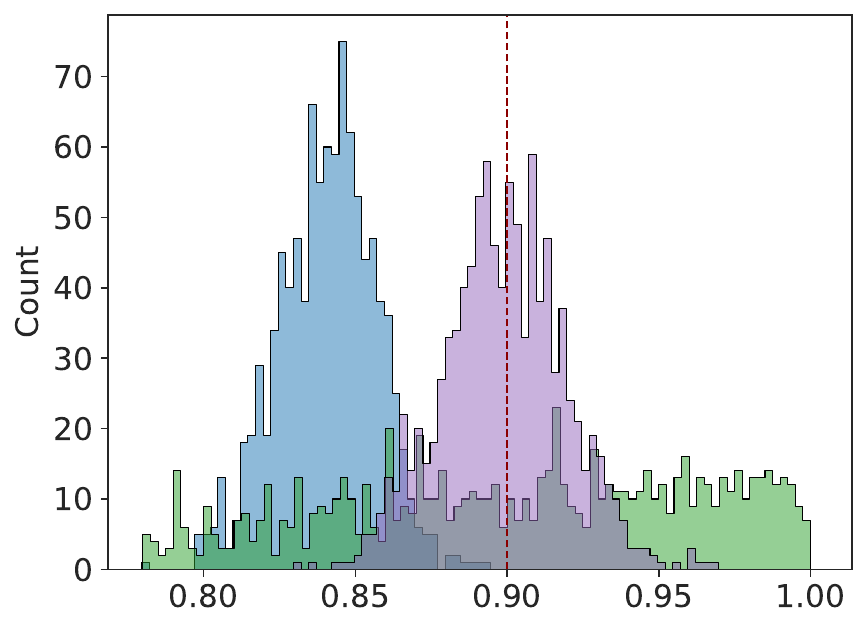}
  \caption{Empirical coverage on the CIFAR-$10$ data.}
  \label{fig:cifar10_cov}
\end{figure}

\begin{figure*}[t!]
  \centering
  \begin{subfigure}{0.8\textwidth}
    \includegraphics[width=\linewidth]{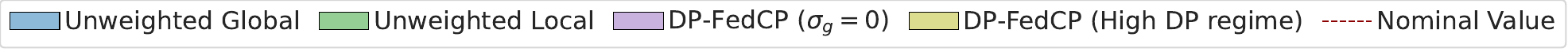}
    \label{fig:methods_legend}
  \end{subfigure}
  \centering
  \begin{subfigure}{0.22\textwidth}
    \includegraphics[width=\linewidth]{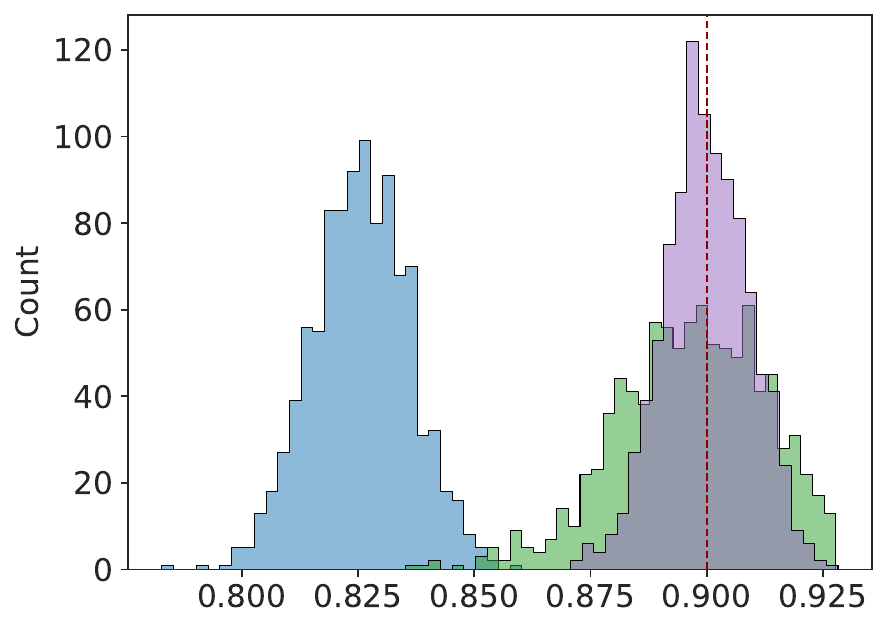}
    \caption{Empirical Coverage} \label{fig:imagenet_cov}
  \end{subfigure}
  ~~
  \begin{subfigure}{.7\textwidth}
    \includegraphics[width=\linewidth]{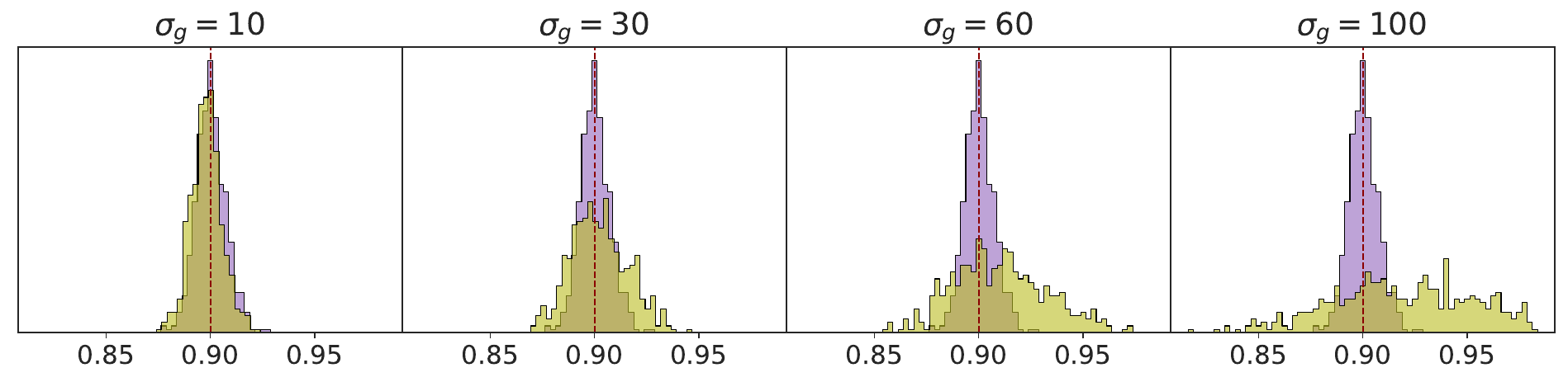}
    \caption{Privacy-Utility tradeoff} \label{fig:imagenet_DP_histos}
  \end{subfigure}
  \centering
  \hspace*{.26em}
  \begin{subfigure}{.95\textwidth}
    \includegraphics[width=\linewidth]{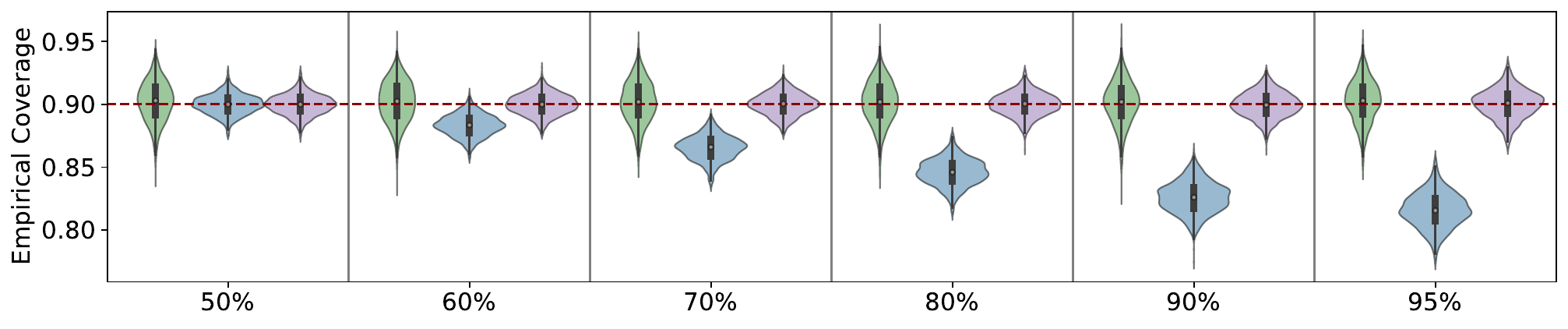}
    \caption{Distribution of the empirical coverage for the methods as a function of the shift} \label{fig:imagenet_violin_distributions}
  \end{subfigure}
  \caption{ImageNet experimental results: (a) Empirical coverage comparison of \algopred\ with unweighted baselines (b) Empirical coverage comparison of \algopred\ with non-DP version at different privacy parameter values (c) Effect of distribution shifts on empirical coverage for  \algopred\ and unweighted baselines.}
\end{figure*}

\format{Simulated Data Experiment.}
  In the first experiment, we demonstrate that it is necessary to consider label shifts between agents to obtain valid coverage of prediction sets.
  We consider a simple classification problem with 3 labels. The conditional distributions of the features given the class label are $3$ two-dimensional Gaussian distributions with means $\thetav_1 = [-1, 0], \thetav_2 = [1, 0], \thetav_3 = [1, 3]$ and with identity covariance matrices. We consider $\nclients = 2$ agents with the distribution of labels $\{\cpmf{1}{y}\}_{y \in [3]} = \{0.8, 0.1, 0.1\}$ and $\{\cpmf{2}{y}\}_{y \in [3]} = \{0.1, 0.1, 0.8\}$. We use the Bayes classifier and consider calibration data with $(\ccount{1}, \ccount{2}) = (1000, 50)$. The inference is performed for agent 2.

  We run independently $1000$ experiments with different splits and record the obtained empirical coverage each time.
  Figure~\ref{fig:toy-scores} shows the distribution of non-conformity scores for the different labels, and Figure~\ref{fig:toy-empcov} shows the empirical coverage of $(1 - \alpha)$ prediction sets with $\alpha = 0.1$ using the \algopred\ method (\Cref{algo:conf-fed}) compared to \unwghtloc\ and \unwghtglo. We also included results obtained with \texttt{oracle-weights}, in which the conformal prediction sets are obtained using \eqref{eq:def:oracle-weights}, i.e., assuming that the exact ratios $\{\iweightp{y}\}_{y\in\YC}$ are known.

  The quantiles calculated via the \unwghtglo\ method are mostly due to the non-conformity scores from agent $1$. This is due to the larger local dataset of agent $1$, whose label distribution is very different from that of the target; see \Cref{fig:toy-scores}.
  The \unwghtloc{} method computes the quantiles based on the local data of agent 2, which has too little data to produce robust prediction sets.
  Therefore, \algopred\ yields much better conformal prediction sets (see \Cref{fig:toy-empcov}), which are little different from those obtained using the adaptive prediction set methods with oracle weights of~\cite{podkopaev2021tracking}.

\format{CIFAR-10 Experiments.}
  We investigate the performance of \algopred\ on the CIFAR-$10$ dataset.
  We use a ResNet-$56$~\citep{he2016resnet} pre-trained on the CIFAR-$10$ training dataset as the underlying classifier with temperature scaling $\temp=1.6$. We also randomly split the CIFAR-$10$ test dataset into a calibration dataset and a test dataset, each containing $5000$ points, and repeat the experiment $1000$ times. The number of agents is $\nclients = 10$, and the prediction set is learned for the agent $\target = 4$ that has the smallest number of data points. The distribution of labels for agent $i$ is $\cpmf{i}{i} = 0.55$ and $\cpmf{i}{y} = 0.05$ for all $y \in [10] \backslash \{i\}$.
  We set the validation size for agent $\target$ to $\ccount{\target} = 50$, and for agent $2$ the validation size is $\ccount{2} = 2150$. The remaining agents have the same validation size of $\ccount{i} = 350$ for all $i \in [10]\backslash\{2, 4\}$.
  The significance level $\alpha$ is set to $0.1$.
  In this configuration, both \unwghtloc\ and \unwghtglo\ methods perform significantly worse than \algopred; see \Cref{fig:cifar10_cov}.

\format{ImageNet Experiments.}
  We use a pre-trained ResNet-152~\citep{he2016resnet} as a base model with temperature scaling $\temp=10$. We perform $1000$ runs with different splits of the $50$K ImageNet test dataset into calibration and test datasets of size $40$K and $10$K samples, respectively. The calibration data is split into $11$ agents. For agent $i \in [10]$, the size of the calibration dataset is $\ccount{i} = 3950$, while we $\ccount{11} = 500$.
  For ImageNet, the distribution of  non-conformity scores $V\prn{\hat{f}\prn{X}, Y}$ varies significantly as a function of the given label $Y=y$.
  In this experiment, we distribute the data between agents to ensure distinct non-conformity score distributions across agents, illustrated in \Cref{fig:imagenet_lowscore,fig:imagenet_highscore}.
  For this, we compute the mean of the non-conformity scores in function of the given label.
  We call $G_1$ the set of the $500$ labels with the lowest means and $G_2$ the set of the remaining $500$ labels. Agents $i \in [10]$ (\texttt{low-score} group) take $90 \%$ of their data from $G_1$ and the remaining $10\%$ from $G_2$. Agent $1$1 takes $90\%$ of its calibration data from $G_2$ and the remaining $10\%$ from $G_1$.

  We construct a prediction set with significance level $\alpha = 0.1$ for the distribution of the $11$-th agent.
  \Cref{fig:imagenet_cov} shows the empirical coverage of the prediction sets.
  In contrast to unweighted alternatives, \algopred\ achieves valid coverage.
  In \Cref{fig:imagenet_violin_distributions}, we evaluate the sensitivity of the different methods to the shift between $G_1$ and $G_2$.
  We repeat the previous experiment varying the shift parameter ($90\%$ in the first experiment) with $100$ runs for each coefficient and show the Violin plot of the obtained empirical coverage.
  The experimental results show that \algopred\ overcomes the challenge of obtaining valid conformal predictions in the presence of label shifts at a federated level compared to alternative methods.

\vspace{0em}
\format{Differential Privacy Experiments.}
  We explore the tradeoff between privacy and coverage quality.
  We conducted the ImageNet experiment with different values of $\sigma_g$ in the set $\{10, 30, 60, 100\}$. The results of the experiment are shown in \Cref{fig:imagenet_DP_histos}, which illustrates the tradeoff between the differential privacy parameter $\sigma_g$ and the robustness of the method. In particular, we observe that as $\sigma_g$ increases, the robustness of the method decreases.

\section{Conclusion}
\label{main:sec-conclusion}

We present a novel method called \algopred, which is designed to construct personalized conformal prediction sets in a federated learning scenario.
Unlike existing algorithms, the proposed method takes into account the label shifts between different agents, and computes prediction sets with a prescribed confidence level.
The resulting sets are theoretically guaranteed to provide valid coverage, while ensuring differential privacy.
Finally, we illustrate the strong performance of \algopred\ in a series of benchmarks.

\section*{Acknowledgements}
The authors gratefully thank Aymeric Dieuleveut for his valuable feedback and Nikita Kotelevskii for his insightful remarks.
This work received support from an artificial intelligence research grant (agreement identifier 000000D730321P5Q0002 dated November 2, 2021 No. 70-2021-00142 with ISP RAS).
Part of this work has been carried out under the auspice of the Lagrange Mathematics and Computing Research Center.

\bibliography{conformal}

\begin{thebibliography}{}

\bibitem[Agrawal and Jia, 2017]{agrawal2017optimistic}
Agrawal, S. and Jia, R. (2017).
\newblock Optimistic posterior sampling for reinforcement learning: worst-case
  regret bounds.
\newblock {\em Advances in Neural Information Processing Systems}, 30.

\bibitem[Angelopoulos et~al., 2021]{angelopoulos2020uncertainty}
Angelopoulos, A.~N., Bates, S., Jordan, M., and Malik, J. (2021).
\newblock Uncertainty sets for image classifiers using conformal prediction.
\newblock In {\em International Conference on Learning Representations}.

\bibitem[Angelopoulos et~al., 2022a]{angelopoulos2021private}
Angelopoulos, A.~N., Bates, S., Zrnic, T., and Jordan, M.~I. (2022a).
\newblock Private {Prediction} {Sets}.
\newblock {\em Harvard Data Science Review}, 4(2).
\newblock https://hdsr.mitpress.mit.edu/pub/deziirvg.

\bibitem[Angelopoulos et~al., 2022b]{angelopoulos2022recommendation}
Angelopoulos, A.~N., Krauth, K., Bates, S., Wang, Y., and Jordan, M.~I.
  (2022b).
\newblock Recommendation systems with distribution-free reliability guarantees.
\newblock {\em arXiv preprint arXiv:2207.01609}.

\bibitem[Balasubramanian et~al., 2014]{balasubramanian2014conformal}
Balasubramanian, V., Ho, S.-S., and Vovk, V. (2014).
\newblock {\em Conformal prediction for reliable machine learning: theory,
  adaptations and applications}.
\newblock Newnes.

\bibitem[Balle et~al., 2018]{balle2018privacy}
Balle, B., Barthe, G., and Gaboardi, M. (2018).
\newblock Privacy amplification by subsampling: Tight analyses via couplings
  and divergences.
\newblock {\em Advances in Neural Information Processing Systems}, 31.

\bibitem[Barber et~al., 2022]{barber2022conformal}
Barber, R.~F., Candes, E.~J., Ramdas, A., and Tibshirani, R.~J. (2022).
\newblock Conformal prediction beyond exchangeability.
\newblock {\em arXiv preprint arXiv:2202.13415}.

\bibitem[Bonawitz et~al., 2019]{bonawitz2019towards}
Bonawitz, K., Eichner, H., Grieskamp, W., Huba, D., Ingerman, A., Ivanov, V.,
  Kiddon, C., Kone{\v{c}}n{\`y}, J., Mazzocchi, S., McMahan, B., et~al. (2019).
\newblock Towards federated learning at scale: System design.
\newblock {\em Proceedings of Machine Learning and Systems}, 1:374--388.

\bibitem[Boucheron et~al., 2013]{boucheron2013concentration}
Boucheron, S., Lugosi, G., and Massart, P. (2013).
\newblock {\em Concentration inequalities: A nonasymptotic theory of
  independence}.
\newblock Oxford university press.

\bibitem[Bubeck et~al., 2015]{bubeck2015convex}
Bubeck, S. et~al. (2015).
\newblock Convex optimization: Algorithms and complexity.
\newblock {\em Foundations and Trends{\textregistered} in Machine Learning},
  8(3-4):231--357.

\bibitem[Buhai, 2005]{buhai2005quantile}
Buhai, S. (2005).
\newblock Quantile regression: overview and selected applications.
\newblock {\em Ad Astra}, 4(4):1--17.

\bibitem[Canonne et~al., 2020]{canonne2020discrete}
Canonne, C.~L., Kamath, G., and Steinke, T. (2020).
\newblock The discrete gaussian for differential privacy.
\newblock {\em Advances in Neural Information Processing Systems},
  33:15676--15688.

\bibitem[Caruana, 1998]{caruana1998multitask}
Caruana, R. (1998).
\newblock {\em Multitask learning}.
\newblock Springer.

\bibitem[Cauchois et~al., 2020]{cauchois2020robust}
Cauchois, M., Gupta, S., Ali, A., and Duchi, J.~C. (2020).
\newblock Robust validation: Confident predictions even when distributions
  shift.
\newblock {\em arXiv preprint arXiv:2008.04267}.

\bibitem[Dwork et~al., 2014]{dwork2014algorithmic}
Dwork, C., Roth, A., et~al. (2014).
\newblock The algorithmic foundations of differential privacy.
\newblock {\em Foundations and Trends{\textregistered} in Theoretical Computer
  Science}, 9(3--4):211--407.

\bibitem[Fannjiang et~al., 2022]{fannjiang2022conformal}
Fannjiang, C., Bates, S., Angelopoulos, A., Listgarten, J., and Jordan, M.~I.
  (2022).
\newblock Conformal prediction for the design problem.
\newblock {\em arXiv preprint arXiv:2202.03613}.

\bibitem[Fontana et~al., 2023]{fontana2023conformal}
Fontana, M., Zeni, G., and Vantini, S. (2023).
\newblock Conformal prediction: a unified review of theory and new challenges.
\newblock {\em Bernoulli}, 29(1):1--23.

\bibitem[Garg et~al., 2020]{garg2020unified}
Garg, S., Wu, Y., Balakrishnan, S., and Lipton, Z. (2020).
\newblock A unified view of label shift estimation.
\newblock {\em Advances in Neural Information Processing Systems},
  33:3290--3300.

\bibitem[Gibbs and Candes, 2021]{gibbs2021adaptive}
Gibbs, I. and Candes, E. (2021).
\newblock Adaptive conformal inference under distribution shift.
\newblock {\em Advances in Neural Information Processing Systems},
  34:1660--1672.

\bibitem[Gillenwater et~al., 2021]{gillenwater2021differentially}
Gillenwater, J., Joseph, M., and Kulesza, A. (2021).
\newblock Differentially private quantiles.
\newblock In {\em International Conference on Machine Learning}, pages
  3713--3722. PMLR.

\bibitem[He et~al., 2016]{he2016resnet}
He, K., Zhang, X., Ren, S., and Sun, J. (2016).
\newblock Deep residual learning for image recognition.
\newblock In {\em IEEE Conference on Computer Vision and Pattern Recognition
  (CVPR)}, pages 770--778.

\bibitem[Hu and Lei, 2020]{hu2020distribution}
Hu, X. and Lei, J. (2020).
\newblock A distribution-free test of covariate shift using conformal
  prediction.
\newblock {\em arXiv preprint arXiv:2010.07147}.

\bibitem[Huang et~al., 2022]{huang2022heterogeneous}
Huang, W., Ye, M., and Du, B. (2022).
\newblock Learn from others and be yourself in heterogeneous federated
  learning.
\newblock In {\em 2022 IEEE/CVF Conference on Computer Vision and Pattern
  Recognition (CVPR)}, pages 10133--10143.

\bibitem[Kairouz et~al., 2021]{kairouz2021advances}
Kairouz, P., McMahan, H.~B., Avent, B., Bellet, A., Bennis, M., Bhagoji, A.~N.,
  Bonawitz, K., Charles, Z., Cormode, G., Cummings, R., et~al. (2021).
\newblock Advances and open problems in federated learning.
\newblock {\em Foundations and Trends{\textregistered} in Machine Learning},
  14(1--2):1--210.

\bibitem[Kairouz et~al., 2015]{kairouz2015composition}
Kairouz, P., Oh, S., and Viswanath, P. (2015).
\newblock The composition theorem for differential privacy.
\newblock In {\em International conference on machine learning}, pages
  1376--1385. PMLR.

\bibitem[Koenker and Hallock, 2001]{koenker2001quantile}
Koenker, R. and Hallock, K.~F. (2001).
\newblock Quantile regression.
\newblock {\em Journal of economic perspectives}, 15(4):143--156.

\bibitem[Lei et~al., 2013]{lei2013distribution}
Lei, J., Robins, J., and Wasserman, L. (2013).
\newblock Distribution-free prediction sets.
\newblock {\em Journal of the American Statistical Association},
  108(501):278--287.

\bibitem[Lei and Cand{\`e}s, 2021]{lei2021conformal}
Lei, L. and Cand{\`e}s, E.~J. (2021).
\newblock Conformal inference of counterfactuals and individual treatment
  effects.
\newblock {\em Journal of the Royal Statistical Society: Series B (Statistical
  Methodology)}.

\bibitem[Li et~al., 2020]{li2020federated}
Li, T., Sahu, A.~K., Talwalkar, A., and Smith, V. (2020).
\newblock Federated learning: Challenges, methods, and future directions.
\newblock {\em IEEE Signal Processing Magazine}, 37(3):50--60.

\bibitem[Li et~al., 2019]{li2019convergence}
Li, X., Huang, K., Yang, W., Wang, S., and Zhang, Z. (2019).
\newblock On the convergence of fedavg on non-iid data.
\newblock In {\em International Conference on Learning Representations}.

\bibitem[Lipton et~al., 2018]{lipton2018detecting}
Lipton, Z., Wang, Y.-X., and Smola, A. (2018).
\newblock Detecting and correcting for label shift with black box predictors.
\newblock In {\em Proceedings of the 35th International Conference on Machine
  Learning}, volume~80, pages 3122--3130. PMLR.

\bibitem[Lu and Kalpathy-Cramer, 2021]{lu2021distribution}
Lu, C. and Kalpathy-Cramer, J. (2021).
\newblock Distribution-free federated learning with conformal predictions.
\newblock {\em arXiv preprint arXiv:2110.07661}.

\bibitem[McMahan et~al., 2017]{mcmahan2017communication}
McMahan, B., Moore, E., Ramage, D., Hampson, S., and y~Arcas, B.~A. (2017).
\newblock Communication-efficient learning of deep networks from decentralized
  data.
\newblock In {\em Artificial intelligence and statistics}, pages 1273--1282.
  PMLR.

\bibitem[Mironov, 2017]{mironov2017renyi}
Mironov, I. (2017).
\newblock R{\'e}nyi differential privacy.
\newblock In {\em 2017 IEEE 30th computer security foundations symposium
  (CSF)}, pages 263--275. IEEE.

\bibitem[Moreau, 1963]{moreau1963proprietes}
Moreau, J.~J. (1963).
\newblock Propri{\'e}t{\'e}s des applications
  {\guillemotleft}prox{\guillemotright}.
\newblock {\em Comptes rendus hebdomadaires des s{\'e}ances de l'Acad{\'e}mie
  des sciences}, 256:1069--1071.

\bibitem[Nemirovski et~al., 2009]{nemirovski2009robust}
Nemirovski, A., Juditsky, A., Lan, G., and Shapiro, A. (2009).
\newblock Robust stochastic approximation approach to stochastic programming.
\newblock {\em SIAM Journal on optimization}, 19(4):1574--1609.

\bibitem[Nesterov, 2003]{nesterov2003introductory}
Nesterov, Y. (2003).
\newblock {\em Introductory lectures on convex optimization: A basic course},
  volume~87.
\newblock Springer Science \& Business Media.

\bibitem[Noble et~al., 2022]{noble2022differentially}
Noble, M., Bellet, A., and Dieuleveut, A. (2022).
\newblock Differentially private federated learning on heterogeneous data.
\newblock In {\em International Conference on Artificial Intelligence and
  Statistics}, pages 10110--10145. PMLR.

\bibitem[Papadopoulos et~al., 2002]{papadopoulos2002inductive}
Papadopoulos, H., Proedrou, K., Vovk, V., and Gammerman, A. (2002).
\newblock Inductive confidence machines for regression.
\newblock In {\em European Conference on Machine Learning}, pages 345--356.
  Springer.

\bibitem[Parikh et~al., 2014]{parikh2014proximal}
Parikh, N., Boyd, S., et~al. (2014).
\newblock Proximal algorithms.
\newblock {\em Foundations and trends{\textregistered} in Optimization},
  1(3):127--239.

\bibitem[Pillutla et~al., 2022]{pillutla2022differentially}
Pillutla, K., Laguel, Y., Malick, J., and Harchaoui, Z. (2022).
\newblock Differentially private federated quantiles with the distributed
  discrete gaussian mechanism.
\newblock In {\em International Workshop on Federated Learning: Recent Advances
  and New Challenges}.

\bibitem[Podkopaev and Ramdas, 2021]{podkopaev2021distribution}
Podkopaev, A. and Ramdas, A. (2021).
\newblock Distribution-free uncertainty quantification for classification under
  label shift.
\newblock In {\em Uncertainty in Artificial Intelligence}, pages 844--853.
  PMLR.

\bibitem[Podkopaev and Ramdas, 2022]{podkopaev2021tracking}
Podkopaev, A. and Ramdas, A. (2022).
\newblock Tracking the risk of a deployed model and detecting harmful
  distribution shifts.
\newblock In {\em International Conference on Learning Representations}.

\bibitem[Rockafellar and Wets, 2009]{rockafellar2009variational}
Rockafellar, R.~T. and Wets, R. J.-B. (2009).
\newblock {\em Variational analysis}, volume 317.
\newblock Springer Science \& Business Media.

\bibitem[Romano et~al., 2020]{romano2020classification}
Romano, Y., Sesia, M., and Candes, E. (2020).
\newblock Classification with valid and adaptive coverage.
\newblock {\em Advances in Neural Information Processing Systems},
  33:3581--3591.

\bibitem[Shafer and Vovk, 2008]{shafer2008tutorial}
Shafer, G. and Vovk, V. (2008).
\newblock A tutorial on conformal prediction.
\newblock {\em Journal of Machine Learning Research}, 9(3).

\bibitem[Tibshirani et~al., 2019]{tibshirani2019conformal}
Tibshirani, R.~J., Foygel~Barber, R., Candes, E., and Ramdas, A. (2019).
\newblock Conformal prediction under covariate shift.
\newblock {\em Advances in neural information processing systems}, 32.

\bibitem[Triastcyn and Faltings, 2019]{triastcyn2019federated}
Triastcyn, A. and Faltings, B. (2019).
\newblock Federated learning with bayesian differential privacy.
\newblock In {\em 2019 IEEE International Conference on Big Data (Big Data)},
  pages 2587--2596. IEEE.

\bibitem[Vovk et~al., 1999]{vovk1999machine}
Vovk, V., Gammerman, A., and Saunders, C. (1999).
\newblock Machine-learning applications of algorithmic randomness.
\newblock In {\em Proceedings of the Sixteenth International Conference on
  Machine Learning}, pages 444--453.

\bibitem[Wei et~al., 2020]{wei2020federated}
Wei, K., Li, J., Ding, M., Ma, C., Yang, H.~H., Farokhi, F., Jin, S., Quek,
  T.~Q., and Poor, H.~V. (2020).
\newblock Federated learning with differential privacy: Algorithms and
  performance analysis.
\newblock {\em IEEE Transactions on Information Forensics and Security},
  15:3454--3469.

\bibitem[Woodworth et~al., 2020]{woodworth2020minibatch}
Woodworth, B.~E., Patel, K.~K., and Srebro, N. (2020).
\newblock Minibatch vs local sgd for heterogeneous distributed learning.
\newblock {\em Advances in Neural Information Processing Systems},
  33:6281--6292.

\bibitem[Yang et~al., 2019]{yang2019federated}
Yang, Q., Liu, Y., Chen, T., and Tong, Y. (2019).
\newblock Federated machine learning: Concept and applications.
\newblock {\em ACM Transactions on Intelligent Systems and Technology (TIST)},
  10(2):1--19.

\bibitem[Yoon et~al., 2022]{yoon2022heterogeneous}
Yoon, J., Park, G., Jeong, W., and Hwang, S.~J. (2022).
\newblock Bitwidth heterogeneous federated learning with progressive weight
  dequantization.
\newblock In Chaudhuri, K., Jegelka, S., Song, L., Szepesvari, C., Niu, G., and
  Sabato, S., editors, {\em Proceedings of the 39th International Conference on
  Machine Learning}, volume 162 of {\em Proceedings of Machine Learning
  Research}, pages 25552--25565. PMLR.

\end{thebibliography}

\clearpage
\newpage
\appendix
\onecolumn

\addtocontents{toc}{\protect\setcounter{tocdepth}{2}}
\allowdisplaybreaks
{
  \hypersetup{linkcolor=burntumber}
  \tableofcontents
}

\section{Moreau Envelope for Quantile Computation}\label{sec:moreau}

\subsection{Federated quantile using the Moreau envelope}

  \begin{lemma}
    Let $\alpha \in [0, 1]$ and $(v, q) \in \R^2$, the Moreau envelope of the pinball loss with regularization parameter $\gamma > 0$ 
    is given by
    \begin{equation}\label{eq:def:moreaupinball}
       \pinballmoreau{v}{q} =
       \begin{cases}
        (1 - \alpha) (v - q) - \frac{\gamma (1 - \alpha)^2}{2}; \hfill \frac{v-q}{\gamma} > 1 - \alpha, \\
        \frac{(q - v)^2}{2\gamma}; \hfill 0\le \frac{q - v}{\gamma} + 1 - \alpha \le 1, \\
        \alpha (q - v) - \frac{\gamma \alpha^2}{2}; \hfill \quad\frac{q - v}{\gamma} > \alpha.
      \end{cases}
    \end{equation}
    Moreover, its gradient is given by
    \begin{equation*}\label{eq:def:grad-pinball}
      \nabla \pinballmoreau{v}{q}
      = -(1 - \alpha) \indiacc{q < v - \gamma (1 - \alpha)} + \alpha \indiacc{q > v + \gamma \alpha} + \frac{1}{\gamma}(q - v)\indiacc{v - \gamma (1 - \alpha) < q < v + \gamma \alpha}.
    \end{equation*}
  \end{lemma}
  \begin{proof}
    For all $\alpha \in [0, 1]$, $(v, q) \in \R^2$, recall that the pinball loss and its subgradient are given by
    \begin{align*}
      \pinball[\alpha][v](q)
      = (1 - \alpha) (v - q) \indiacc{v \ge q} + \alpha (q - v) \indiacc{q > v}, && \partial \pinball[\alpha][v](q) =
      \begin{cases}
        -(1 - \alpha), &q < v \\
        [-(1 - \alpha), \alpha], &q = v\\
        \alpha, &q > v
      \end{cases}.
    \end{align*}
      Note that, by construction
    \begin{equation*}
      \pinballmoreau{v}{q}
      = \min_{\tilde{q} \in \R}\Big\{\pinball[\alpha][v](\tilde{q}) + \frac{1}{2\gamma}(\tilde{q} - q)^2\Big\}
      = \min_{\tilde{q} \in \R}\ac{(1 - \alpha) (v - \tilde{q})\indiacc{v \ge \tilde{q}} + \alpha (\tilde{q} - v)\indiacc{v<\tilde{q}} + \frac{1}{2\gamma}(\tilde{q} - q)^2}.
    \end{equation*}
    Denote $q_\star=\argmin_{\tilde{q} \in \R}\{\pinball[\alpha][v](\tilde{q}) + \frac{1}{2\gamma}(\tilde{q} - q)^2\}$ which exists and is unique (the function to be minimized is coercive and strongly convex).
    The stationary condition for the Moreau envelope is given by:
    \begin{align*}
      0 \in \partial \pinball[\alpha][v](q_\star) + \frac{1}{\gamma}(q_\star - q), \text{ with } \quad \partial \pinball[\alpha][v](q) =
      \begin{cases}
          -(1 - \alpha), &q < v \\
          [-(1 - \alpha), \alpha], &q = v\\
          \alpha, &q > v
      \end{cases}.
    \end{align*}
    Considering the 3 different cases, we find that:
    \begin{align*}
      q_{\star} &=
      \begin{cases}
          q + \gamma (1 - \alpha), &q < v - \gamma (1 - \alpha) \\
          v, &q \in [v - \gamma (1 - \alpha), v + \gamma \alpha] \\
          q - \gamma \alpha, & q > v + \gamma \alpha
      \end{cases}.
    \end{align*}
    We conclude the derivation by  using the identity from Moreau envelope: $\pinballmoreau{v}{q}= \pinball[\alpha][v]\pr{q_{\star}} + \frac{1}{2\gamma}\pr{q_{\star} - q}^2$ and plugging in $q_{\star}$.
  \end{proof}

  To simplify the manuscript presentation, we now provide the definition of the local loss function $\lossi\colon q\in\R \mapsto\E_{V \sim \qdistr{y}^{i}}\brn{\pinballmoreau{V}{q}}\in\R_+$.
  Recall the weights $\qweight{y}{\yquery}$ are given in \eqref{eq:def:wytarget}, and also that
  \[
    \lambda_{\yquery}^{i} = \frac{\ccount{i}}{\tcount} \qweight{\yquery}{\yquery} + \sum_{k = 1}^{\ccount{i}} \qweight{Y_k^{i}}{\yquery}.
  \]
  Therefore, for $q\in\R$, we have
  \begin{equation}\label{eq:def:S-alpha-i-gamma}
    \lossi\prn{q}
    = \frac{\ccount{i} \qweight{\yquery}{\yquery}}{\lambda_{\yquery}^{i} \tcount} \pinballmoreau{1}{q} + \sum_{k=1}^{\ccount{i}} \frac{\qweight{Y_k^i}{\yquery}}{\lambda_{\yquery}^{i}} \pinballmoreau{V_k^i}{q}.
  \end{equation}

\subsection{Moreau's approximation error}\label{subsec:approx-moreau}

  In this section, we consider fixed parameters $\alpha \in (0, 1)$, $\gamma > 0$, $\acn{v_k}_{k \in [\tcount]}$ and $\acn{p_k}_{k \in [\tcount]} \in [0, 1]^\tcount$ satisfying $\sum_k p_k = 1$.
  We define $F := \sum_{k=1}^\tcount p_k S_{\alpha, v_k}$ and $F_{\gamma}:= \sum_{k = 1}^\tcount p_k S_{\alpha, v_k}^{\gamma}$, where $S_{\alpha, v_k}$ and $S_{\alpha, v_k}^{\gamma}$ are the pinball loss and its Moreau envelope defined for $v, q \in \R$ by~\eqref{eq:def:moreaupinball}.
  Without loss of generality, it is assumed that $\acn{v_k}_{k\in[\tcount]}$ is increasing since we can re-index $\acn{v_k}_{k \in [\tcount]}$ and if there exist $(j,j') \in [\tcount]^2$ such that $j \neq j'$ and $v_j = v_{j'}$, we have $p_{j} S_{\alpha, v_j} + p_{j'} S_{\alpha, v_{j'}} = (p_j + p_j') S_{\alpha, v_j}$.
  Finally, for $k\in[\tcount]$, denote
  \begin{equation}\label{eq:def:Ik}
    I_k = \left[v_k - \gamma (1 - \alpha), v_k + \gamma \alpha\right].
  \end{equation}

  \begin{lemma}\label{lem:moreau-approx:1}
    If $(1 - \alpha) \notin \acn{\sum_{l = 1}^k p_l}_{k \in [\nclients]}$, then $F$ admits a unique minimizer. Moreover, this minimizer belongs to $\acn{v_k}_{k\in[\nclients]}$ and we denote $k_\star \in [\nclients]$ its index, i.e., $v_{k_\star} = \argmin F$.
    In addition, $F$ is decreasing on $(-\infty, v_{k_\star}]$ and increasing on $[v_{k_\star}, \infty)$.
    The function $F_\gamma$ also admits a unique minimizer denoted $\qmoreau[1 - \alpha] \in \R$, and $F_{\gamma}$ is decreasing on $(-\infty, \qmoreau[1 - \alpha]]$ and increasing on $[\qmoreau[1 - \alpha], \infty)$.
  \end{lemma}
  \begin{proof}
    Note that $F$ is differentiable on $\R \setminus \acn{v_k}_{k \in [N]}$, and for all $q \in \R \setminus \acn{v_k}_{k \in [N]}$, we have
    \begin{equation*}
      F'(q)
      = \alpha \sum_{k\colon v_k<q}p_k - (1 - \alpha) \sum_{k\colon v_k \ge q}p_k
      = \sum_{k\colon v_k < q}p_k - (1 - \alpha).
    \end{equation*}
    Since $\alpha \in (0, 1)$ with $(1 - \alpha) \notin \acn{\sum_{l=1}^k p_l}_{k \in [N]}$, we deduce that there exists a unique $v_{k_\star} \in \acn{v_k}_{k \in [N]}$ such that, for $q \in \R$,
    \begin{equation*}
      \sum_{k\colon v_k < q}p_k - (1 - \alpha) \quad\text{ is }\quad
      \begin{cases}
        < 0 &\text{if $q < v_{k_\star}$,} \\
        > 0 &\text{if $q > v_{k_\star}$.}
      \end{cases}
    \end{equation*}
    Thus, from the continuity of $F$ it follows that $F$ is decreasing on $(-\infty,v_{k_\star}]$ and increasing on $[v_{k_\star},\infty)$.
    Moreover, since $F_\gamma'=F'$ on $\R \setminus \acn{\cup_{k \in [N]} I_k}$, its minimizer lies in $\cup_{k \in [N]} I_k$, where $I_k$ is defined in \eqref{eq:def:Ik}. Finally, the strong convexity of $F_\gamma$ on $\cup_{k \in [N]} I_k$ shows the uniqueness of $\qmoreau[1 - \alpha]=\argmin F_\gamma$ and also that $F_\gamma$ is decreasing on $(-\infty, \qmoreau[1 - \alpha]]$ and increasing on $[\qmoreau[1 - \alpha],\infty)$.
  \end{proof}

  Denote by $\partial F$  subgradient of $F$. If $F$ is differentiable at $q\in\R$, then $\partial F(q)= \{ F'(q) \}$. When $\partial F(q)$ is a singleton, by an abuse of notation, we use the same notation for the set and the unique element it contains.
  \begin{theorem}\label{thm:FedMoreau-approx}
    Assume $(1 - \alpha) \notin \acn{\sum_{l = 1}^k p_l}_{k\in[N]}$, then the unique minimizers ($v_{k_\star}$, $\qmoreau[1 - \alpha]$) resp. of ($F$, $F_\gamma$) resp. satisfy $\absn{\qmoreau[1 - \alpha] - v_{k_\star}} \le \gamma$.
  \end{theorem}
  \begin{proof}
    First, for any $k \in [N]$ such that $v_{k_\star} - \gamma (1 - \alpha) \in I_{k}$, we obtain
    \begin{equation}\label{eq:cases:frac-vk}
      \frac{v_{k_\star} - v_k - \gamma (1 - \alpha)}{\gamma} \le
      \begin{cases}
        -(1 - \alpha) &\text{if $v_k \ge v_{k_\star}$}, \\
        \alpha &\text{else $v_k < v_{k_\star}$}.
      \end{cases}
    \end{equation}
    Since \Cref{lem:moreau-approx:1} shows that $F$ is decreasing on $(-\infty, v_{k_\star}]$ and increasing on $[v_{k_\star},\infty)$ where $v_{k_\star}$ is the unique minimizer of $F(v_{k_\star})$, the convexity of $F$ implies that:
    \begin{equation}\label{eq:cases:partialF}
      \partial F(q) \subset
      \begin{cases}
        (-\infty,0) &\text{if $q < v_{k_\star}$,} \\
        (0,\infty) &\text{if $q > v_{k_\star}$.}
      \end{cases}
    \end{equation}
    Thus, \eqref{eq:cases:frac-vk} combined with~\eqref{eq:cases:partialF} give that
    \begin{multline*}
      F_\gamma'\bigl(v_{k_\star} - \gamma (1 - \alpha)\bigr) \\
      = \sum_{k\colon v_{k_\star} - \gamma (1 - \alpha) \notin I_{k}} p_k S_{\alpha, v_k}' \bigl(v_{k_\star} - \gamma (1 - \alpha)\bigr)
      + \sum_{k\colon v_{k_\star} - \gamma (1 - \alpha) \in I_{k}}p_k \frac{\bigl(v_{k_\star} - v_k -\gamma (1 - \alpha)\bigr)}{\gamma} \\
      \le \alpha \sum_{k\colon v_k<v_{k_\star}}p_k - (1 - \alpha) \sum_{k\colon v_k\ge v_{k_\star}}p_k
      = \partial F(v_{k_\star} - \epsilon)
      < 0,
    \end{multline*}
    where $\epsilon = 2^{-1} \min_{k = 1}^{N - 1}\acn{v_{k + 1} - v_{k}}$.
    A similar reasoning shows that $F_\gamma'(v_{k_\star} + \gamma \alpha)\ge \partial F(v_{k_\star} + \epsilon)>0$.
    Since $F_\gamma$ is decreasing on $(-\infty,\qmoreau[1 - \alpha]]$ and increasing on $[\qmoreau[1 - \alpha],\infty)$ by \Cref{lem:moreau-approx:1}, we have $F_\gamma' < 0$ on $(-\infty, v_{k_\star} - \gamma (1 - \alpha)]$ and $F_\gamma' > 0$ on $[v_{k_\star} - \gamma \alpha,\infty)$.
    Therefore, we deduce that $\qmoreau[1 - \alpha] \in I_{k_\star}$. Using that the interval $I_{k_\star}$ is of length $\gamma$, this implies that $\absn{\qmoreau[1 - \alpha] - v_{k_\star}}\le \gamma$.
  \end{proof}

\section{FL convergence guarantee: proof of \Cref{thm:cv-fedmoreau}}\label{sec:convergence-FedMoreau}

In this section, we suppose that $\acn{S_t}_{t\in[T]}$ is a sequence of i.i.d. random variables, such that, for any $(i,i')\in[\nclients]^2$, $i\in S_t$ and $i'\in S_t$ are independent if $i\neq i'$.
For any $i\in[n]$, let $\acn{z_{t,k}^i\colon k\in\acn{0,\ldots,K}}_{t=0}^T$ be a sequence of i.i.d. standard Gaussian variables.
Moreover, consider the local loss function $\fcv{i}:\R\to \R$ and denote, for $t\in\acn{0,\ldots,T}, k\in\acn{0,\ldots,K}$
\begin{align*}
    \fcvtot = \sum_{i=1}^n \lambda_{\yquery}^i \fcv{i},&
    &g_{t,k}^i = \nabla F^i(q_{t,k}^i) + z_{t,k}^i.
\end{align*}
In this section, we establish the convergence of the iterates given by \Cref{thm:cvFL} under the following assumptions:
\begin{assumption}\label{ass:suppl:minimizer}
    The function $\sum_{i=1}^\nclients \lambda_{\yquery}^i \fcv{i}$ admits at least a minimizer in $\R$, we denote $q_\star$ one of them, i.e., $q_\star\in \argmin\acn{\sum_{i=1}^\nclients \lambda_{\yquery}^i \fcv{i}}$.
\end{assumption}
\begin{assumption}\label{ass:suppl:convexity}
    For any $i \in [\nclients]$, $\fcv{i}$ is continuously differentiable and convex, i.e., for any $q,\tilde{q}\in \R$,
    \begin{equation*}
        \fcv{i}(\tilde{q}) \le \fcv{i}(q) + \psn{\nabla \fcv{i}(q)}{\tilde{q}-q}.
    \end{equation*}
\end{assumption}
\begin{assumption}\label{ass:suppl:smoothness}
    For any $i \in [\nclients], t\in\acn{0,\ldots,T}, K\in\acn{0,\ldots K}$, $\nabla F^i$ is continuously differentiable. In addition, there exist $H^i\ge 0$ such that the function $\nabla F^i$ is $H^i$-smooth, i.e., for any $q,\tilde{q}\in \R$,
    \begin{equation*}
        \nabla F^i(\tilde{q}) \le \nabla F^i(q) + \ps{\nabla F^i(q)}{\tilde{q}-q} + (\nofrac{H^i}{2}) \norm{\tilde{q}-q}^2.
    \end{equation*}
    Moreover, denote $H=\max_{i\in[n]}\acn{H^i}$.
\end{assumption}
We introduce the key assumptions appearing in the theoretical derivations below.
\begin{assumption}\label{ass:suppl:variance}
    For any $i \in [\nclients]$, the variance of the gradients is uniformly bounded, for all $q\in\R$, we have
    \begin{align*}
        \E\br{\norm{\frac{\1_{i\in S_t}}{\prob\prn{i\in S_{t}}} \nabla F^i(q) - \nabla \fcv{i}(q)}^2} \le \varnoise^2,&
        &\E\br{\norm{\sum_{i\in S_t} \frac{\lambda_{\yquery}^i}{\prob\prn{i\in S_{t}}} \nabla F^i(q_{\star}) - \nabla \fcvtot(q_{\star})}^2} \le \varnoise_\star^2.
    \end{align*}
\end{assumption}

\begin{assumption}\label{ass:suppl:heterogeneity}
    The heterogeneity denoted $\zeta$ is bounded everywhere
    \begin{equation*}
        \max_{i\in[\nclients]}\ac{\normn{\nabla \fcv{i} - \nabla \fcvtot}_{\infty}} \le \zeta^2.
    \end{equation*}
\end{assumption}
We prove \Cref{thm:cvFL} using \Cref{lem:E-cvFL} and \Cref{lem:E-biasFL}. Note that these results are close to \citet[Appendix~C]{woodworth2020minibatch}.
However, we treat partial participation, i.e., $S_t\subseteq [n]$ and consider an objective function defined by importance weights $\acn{\lambda_{\yquery}^i}_{i\in[\nclients]}$.
At time $t\in\acn{0,\ldots,T}$, denote 
\begin{equation*}\label{eq:definition-barq-t-k}
    \bar{q}_{t,k} = \sum_{i=1}^n \lambda_{\yquery}^i(\indi{S_{t+1}}(i)/\prob\prn{i\in S_{t+1}}) q_{t,k}^{i}
\end{equation*} 
the average of the local parameters defined in \Cref{algo:Dp-moreau}.
Finally, we introduce the following stepsize:
\begin{equation}\label{eq:def:eta0}
    \eta_{0} = \frac{1}{10}\min\pr{\frac{1}{H}, \min_{i=1}^{\nclients}\ac{\frac{\prob\pr{i\in S_t}}{\lambda_{\yquery}^i H^i}}}.
\end{equation}

\begin{lemma}\label{lem:E-cvFL}
    Assume \Cref{ass:suppl:minimizer}-\Cref{ass:suppl:convexity}-\Cref{ass:suppl:smoothness}-\Cref{ass:suppl:variance} and consider $\eta \in (0,\eta_{0}]$.
    Then, for any $t\in\acn{0,\ldots,T-1}$, $k\in\acn{0,\ldots,K-1}$, we have
    \begin{multline*}
        \E\br{\fcvtot\pr{\bar{q}_{t,k}}-\fcvtot\prn{q_\star}}
        \le \frac{1}{\eta} \E\norm{\bar{q}_{t,k}-q_{\star}}^2
        -\frac{1}{\eta} \E\norm{\bar{q}_{t,k+1}-q_{\star}}^2
        + 2 H \sum_{i=1}^{\nclients} \lambda_{\yquery}^i \E\norm{\bar{q}_{t,k}-q_{t,k}^{i}}^2 \\
        + 3 \eta \varnoise_\star^2
        + \eta^2 \sigma^2 \sum_{i=1}^{\nclients} \frac{(\lambda_{\yquery}^i)^2}{\prob\prn{i\in S_{t+1}}} .
    \end{multline*}
\end{lemma}

\begin{proof}
    Developing the squared norm, we find
    \begin{equation}\label{eq:lemA6:1}
        \E\norm{\bar{q}_{t,k+1}-q_{\star}}^2
        = \E\norm{\bar{q}_{t,k}-\eta \sum_{i=1}^{\nclients} \lambda_{\yquery}^i\nabla \fcv{i}\pr{q_{t,k}^{i}}-q_{\star}}^2
        + \eta^2 \E\norm{\sum_{i=1}^{\nclients} \lambda_{\yquery}^i\ac{\frac{\indi{S_{t+1}}(i)}{\prob\prn{i\in S_{t+1}}}g_{t,k}^i-\nabla \fcv{i}\pr{q_{t,k}^{i}}}}^2.
    \end{equation}
    We start by upper bounding the first term, we have
    \begin{multline}\label{eq:lemA6:2}
        \E\norm{\bar{q}_{t,k}- \eta \sum_{i=1}^{\nclients} \lambda_{\yquery}^i \nabla \fcv{i}\pr{q_{t,k}^{i}}-q_{\star}}^2 \\
        =\E\norm{\bar{q}_{t,k}-q_{\star}}^2
        - 2 \eta \sum_{i=1}^{\nclients} \lambda_{\yquery}^i \E\ps{\bar{q}_{t,k}-q_{\star}}{\nabla \fcv{i}\pr{q_{t,k}^{i}}}
        +\eta^2 \E\norm{\sum_{i=1}^{\nclients}\lambda_{\yquery}^i \nabla \fcv{i}\pr{q_{t,k}^{i}}}^2.
    \end{multline}
    Using \Cref{ass:suppl:smoothness}, we know that $\fcv{i}$ is $H^i$-smooth and thus $\fcvtot=\sum_{i=1}^\nclients \lambda_{\yquery}^i \fcv{i}$ is $\bar{H}$-smooth, where $\bar{H}=\sum_{i=1}^n \lambda_{\yquery}^i H^i$.
    Following \citep{nesterov2003introductory}, the smoothness and convexity of $\fcvtot$ imply that
    \begin{equation*}
        \norm{\nabla \fcvtot\pr{\bar{q}_{t,k}} - \nabla \fcvtot(q_{\star})}^2
        \le 2 \bar{H} \pr{\fcvtot\pr{\bar{q}_{t,k}} - \fcvtot(q_{\star})}.
    \end{equation*}
    For any $a,b\in\R$, using that $\normn{a+b}^2\le 2\normn{a}^2 + 2\normn{b}^2$, the last term of~\eqref{eq:lemA6:2} can be upper bounded as follows:
    \begin{align}
        \nonumber
        \E\norm{\sum_{i=1}^{\nclients}\lambda_{\yquery}^i \nabla \fcv{i}\pr{q_{t,k}^{i}}}^2
        &\le 2 \E\norm{\sum_{i=1}^{\nclients}\lambda_{\yquery}^i \ac{\nabla \fcv{i}\pr{q_{t,k}^{i}}-\nabla \fcv{i}\pr{\bar{q}_{t,k}}}}^2
        +2 \E\norm{\sum_{i=1}^{\nclients}\lambda_{\yquery}^i \ac{\nabla \fcv{i}\pr{\bar{q}_{t,k}}-\nabla \fcv{i}\pr{q_{\star}}}}^2 \\
        \nonumber
        &\le 2 \sum_{i=1}^{\nclients} \lambda_{\yquery}^i\E\norm{\nabla \fcv{i}\pr{q_{t,k}^{i}}-\nabla \fcv{i}\pr{\bar{q}_{t,k}}}^2
        +2 \E\norm{\nabla \fcvtot\pr{\bar{q}_{t,k}}-\nabla \fcvtot\pr{q_{\star}}}^2 \\
        \label{eq:lemA6:t1:1}
        &\le 2 \sum_{i=1}^{\nclients} (H^i)^2 \lambda_{\yquery}^i\E\norm{q_{t,k}^{i}-\bar{q}_{t,k}}^2
        + 4 \bar{H} \E\br{\fcvtot\pr{\bar{q}_{t,k}}-\fcvtot\pr{q_{\star}}}.
    \end{align}
    Regarding the inner product in~\eqref{eq:lemA6:2}, \Cref{ass:suppl:convexity} and \Cref{ass:suppl:smoothness} show
    \begin{align}
        \nonumber
        &- \sum_{i=1}^{\nclients} \lambda_{\yquery}^i \E\ps{\bar{q}_{t,k}-q_{\star}}{\nabla \fcv{i}\pr{q_{t,k}^{i}}}
        = -\sum_{i=1}^{\nclients} \lambda_{\yquery}^i \E\ps{ q_{t,k}^{i}-q_{\star}}{\nabla \fcv{i}\pr{q_{t,k}^{i}}}
        +\sum_{i=1}^{\nclients} \lambda_{\yquery}^i \E\ps{ q_{t,k}^{i}-\bar{q}_{t,k}}{\nabla \fcv{i}\pr{q_{t,k}^{i}}} \\
        \nonumber
        &\le -\sum_{i=1}^{\nclients} \lambda_{\yquery}^i \E\br{\fcv{i}\pr{q_{t,k}^{i}}-\fcv{i}\pr{q_{\star}}}
            +\sum_{i=1}^{\nclients} \lambda_{\yquery}^i \E\br{\fcv{i}\pr{q_{t,k}^{i}}-\fcv{i}\pr{\bar{q}_{t,k}}
            +\frac{H^i}{2}\E\norm{q_{t,k}^{i}-\bar{q}_{t,k}}^2} \\
        &\le - \E\br{\fcvtot\pr{\bar{q}_{t,k}}-\fcvtot\pr{q_\star}}
        \label{eq:lemA6:t1:2}
        + \frac{1}{2} \sum_{i=1}^{\nclients} H^i \lambda_{\yquery}^i \norm{q_{t,k}^{i}-\bar{q}_{t,k}}^2.
    \end{align}
    Moreover, recall that the estimator $(\nofrac{\indi{S_{t+1}}(i)}{\prob\prn{i\in S_{t+1}}}) \nabla F^i$ is unbiased, i.e.,
    \begin{equation}\label{eq:lemA6:3}
        \E\br{\frac{\indi{S_{t+1}}(i)}{\prob\prn{i\in S_{t+1}}} \nabla F^i(q_{t,k}^i) - \nabla \fcv{i}\pr{q_{t,k}^{i}}} = 0,
    \end{equation}
    and we also have
    \begin{multline}\label{eq:lemA6:4}
        \frac{\indi{S_{t+1}}(i)}{\prob\prn{i\in S_{t+1}}} \nabla F^i \pr{q_{t,k}^{i}} - \nabla \fcv{i}\pr{q_{t,k}^{i}}
        = \ac{\frac{\indi{S_{t+1}}(i)}{\prob\prn{i\in S_{t+1}}} \nabla F^i \pr{q_\star} - \nabla \fcv{i}\pr{q_\star}}
        \\
        + \ac{\frac{\1_{i\in S_{t+1}}
        }{\prob\prn{i\in S_{t+1}}} \nabla F^i \pr{\bar{q}_{t,k}} - \frac{\indi{S_{t+1}}(i)}{\prob\prn{i\in S_{t+1}}} \nabla F^i \pr{q_\star}
        - \nabla \fcv{i}\pr{\bar{q}_{t,k}} + \nabla \fcv{i}\pr{q_\star}}
        \\
        + \ac{\frac{\indi{S_{t+1}}(i)}{\prob\prn{i\in S_{t+1}}} \nabla F^i \pr{q_{t,k}^{i}} - \frac{\indi{S_{t+1}}(i)}{\prob\prn{i\in S_{t+1}}} \nabla F^i \pr{\bar{q}_{t,k}}
        - \nabla \fcv{i}\pr{q_{t,k}^{i}} + \nabla \fcv{i}\pr{\bar{q}_{t,k}}}
        .
    \end{multline}
    We now upper bound the second term of~\eqref{eq:lemA6:1}.
    Since for all random variable $X$, $\E[\normn{X-\E X}^2] \le \E\normn{X}^2$, combining~\eqref{eq:lemA6:3} and~\eqref{eq:lemA6:4} gives
    \begin{align}
        \nonumber
        &\E\norm{\sum_{i=1}^{\nclients} \lambda_{\yquery}^i \ac{\frac{\indi{S_{t+1}}(i)}{\prob\prn{i\in S_{t+1}}} g_{t,k}^i - \nabla \fcv{i}\pr{q_{t,k}^{i}}}}^2
        = \sum_{i=1}^{\nclients} (\lambda_{\yquery}^i)^2 \E\norm{\frac{\indi{S_{t+1}}(i)}{\prob\prn{i\in S_{t+1}}} \nabla F^i \pr{q_{t,k}^{i}} - \nabla \fcv{i}\pr{q_{t,k}^{i}}}^2 \\
        \nonumber
        &\qquad + \sum_{i=1}^{\nclients} (\lambda_{\yquery}^i)^2 \E\norm{ \frac{\indi{S_{t+1}}(i)}{\prob\prn{i\in S_{t+1}}} z_{t,k}^i}^2 \\
        \nonumber
        &\le 3 \sum_{i=1}^{\nclients} \frac{(\lambda_{\yquery}^i)^2 (1 - \prob\pr{i\in S_{t+1}})}{\prob\pr{i\in S_{t+1}}}
            \br{\E\norm{\nabla F^i\pr{q_{t,k}^{i}}-\nabla F^i\pr{\bar{q}_{t,k}}}^2
            + \E\norm{\nabla F^i\pr{\bar{q}_{t,k}}-\nabla F^i\pr{q_{\star}}}^2} \\
            \nonumber
            &\qquad+ 3 \E\norm{\sum_{i=1}^{\nclients} \lambda_{\yquery}^i \frac{\indi{S_{t+1}}(i)}{\prob\prn{i\in S_{t+1}}}\nabla F^i \pr{q_\star} - \nabla \fcvtot \pr{q_\star}}^2
            + \sum_{i=1}^n \frac{(\lambda_{\yquery}^i)^2}{\prob\prn{i\in S_{t+1}}} \E\norm{z_{t,k}^i}^2 \\
        \nonumber
        &\le 3 \sum_{i=1}^{\nclients} \frac{(\lambda_{\yquery}^i)^2}{\prob\pr{i\in S_{t+1}}} \ac{(H^i)^2 \E\norm{q_{t,k}^{i}-\bar{q}_{t,k}}^2
            + 2 H^i \E\br{\fcv{i}\pr{\bar{q}_{t,k}}-\fcv{i}\pr{q_{\star}}}} \\
            \nonumber
            &\qquad+ 3 \varnoise_\star^2 + \sum_{i=1}^n \frac{(\lambda_{\yquery}^i)^2}{\prob\prn{i\in S_{t+1}}}\sigma^2 \\ 
        \nonumber
        &\le 3 \sum_{i=1}^{\nclients} \frac{(\lambda_{\yquery}^i)^2}{\prob\pr{i\in S_{t+1}}} (H^i)^2 \E\norm{q_{t,k}^{i}-\bar{q}_{t,k}}^2
        + 6 \sum_{i=1}^{\nclients} \frac{(\lambda_{\yquery}^i)^2}{\prob\pr{i\in S_{t+1}}} H^i \E\br{\fcv{i}\pr{\bar{q}_{t,k}}-\fcv{i}\pr{q_{\star}}} \\
        \label{eq:lemA6:t2}
        &\qquad+ 3 \varnoise_\star^2 + \sigma^2 \sum_{i=1}^n \frac{(\lambda_{\yquery}^i)^2}{\prob\prn{i\in S_{t+1}}}.
    \end{align}
    Plugging~\eqref{eq:lemA6:t1:1}-\eqref{eq:lemA6:t1:2}-\eqref{eq:lemA6:t2} back into~\eqref{eq:lemA6:1} with $\eta \le \eta_{0}$, it holds
    \begin{align*}
        &\E\norm{\bar{q}_{t,k+1}-q_{\star}}^2
        \le \E\norm{\bar{q}_{t,k}-q_{\star}}^2
        + \eta \sum_{i=1}^{\nclients} \pr{1 + 2\eta H^i + 3\eta \frac{\lambda_{\yquery}^i H^i}{\prob\pr{i\in S_{t+1}}}} \lambda_{\yquery}^i H^i \E\norm{q_{t,k}^{i}-\bar{q}_{t,k}}^2 \\
        &+ \eta \pr{4\eta \bar{H} + 6 \eta \max_{i=1}^\nclients\ac{\frac{(\lambda_{\yquery}^i)^2 H^i}{\prob\pr{i\in S_{t+1}}}} - 2} \E\br{\fcvtot\pr{\bar{q}_{t,k}}-\fcvtot\pr{q_\star}}
        + 3 \eta^2 \varnoise_\star^2 + \eta^2 \sigma^2 \sum_{i=1}^n \frac{(\lambda_{\yquery}^i)^2}{\prob\prn{i\in S_{t+1}}} \\
        &\le \E\norm{\bar{q}_{t,k}-q_{\star}}^2
        + 2 H \eta \sum_{i=1}^{\nclients} \lambda_{\yquery}^i \E\norm{q_{t,k}^{i}-\bar{q}_{t,k}}^2
        -\eta \E\br{\fcvtot\pr{\bar{q}_{t,k}}-\fcvtot\pr{q_\star}}
        + 3 \eta^2 \varnoise_\star^2
        + \eta^2 \sigma^2 \sum_{i=1}^n \frac{(\lambda_{\yquery}^i)^2}{\prob\prn{i\in S_{t+1}}}.
    \end{align*}
\end{proof}

\begin{lemma}\label{lem:E-biasFL}
    Assume \Cref{ass:suppl:minimizer}-\Cref{ass:suppl:convexity}-\Cref{ass:suppl:smoothness}-\Cref{ass:suppl:variance}-\Cref{ass:suppl:heterogeneity}, and for all $t\in[T]$, suppose that $S_t = [\nclients]$. We consider $\eta \in (0,2/\sum_{i=1}^{\nclients} \lambda_{\yquery}^i H^i]$.
    Then, for any $t\in\acn{0,\ldots,T-1}$, $k\in\acn{0,\ldots,K-1}$, we have
    \begin{equation*}
        \sum_{i=1}^{\nclients} \lambda_{\yquery}^i \E\norm{q_{t,k}^{i}-\bar{q}_{t,k}}^2 \le 6 K \eta^2 \pr{\sigma^2 + \varnoise^2 + K \zeta^2}.
    \end{equation*}
\end{lemma}

\begin{proof}
    Let $\epsilon >0$, for any $i, i' \in [n]$ and any $k\in[K]$,
    \begin{align}
        \nonumber
        &\E\normbig{q_{t,k}^{i}-q_{t,k}^{i'}}^2 - 2\eta^2 (\sigma^2 + \varnoise^2)
        = \E\norm{q_{t,k-1}^{i}-q_{t,k-1}^{i'} - \eta \pr{g_{t,k-1}^i - g_{t,k-1}^{i'}}}^2 - 2\eta^2 (\sigma^2 + \varnoise^2) \\
        \nonumber
        &= \E\norm{q_{t,k-1}^{i}-q_{t,k-1}^{i'} -\eta \pr{\nabla \fcv{i}(q_{t,k-1}^i) - \nabla \fcv{i'}(q_{t,k-1}^{i'})}}^2 \\
        \nonumber
        &\qquad+ \eta^2 \E\norm{\pr{\nabla \fcv{i}(q_{t,k-1}^i) - \nabla \fcv{i'}(q_{t,k-1}^{i'})} - \pr{g_{t,k-1}^i - g_{t,k-1}^{i'}}}^2 - 2\eta^2 (\sigma^2 + \varnoise^2) \\
        \nonumber
        &\le \E \Big\|q_{t-1}^{i}-q_{t-1}^{i'}-\eta\pr{\nabla \fcvtot\prn{q_{t-1}^{i}}-\nabla \fcvtot\prn{q_{t-1}^{i'}}}
        + \eta\pr{\nabla \fcvtot\prn{q_{t-1}^{i}}-\nabla \fcv{i}\prn{q_{t-1}^{i}}-\nabla \fcvtot\prn{q_{t-1}^{i'}}+\nabla \fcv{i'} \prn{q_{t-1}^{i'}}}\Big\|^2 \\
        \nonumber
        &\le \pr{1+\frac{1}{\epsilon}} \E\norm{q_{t-1}^{i}-q_{t-1}^{i'}-\eta\pr{\nabla \fcvtot\prn{q_{t-1}^{i}}-\nabla \fcvtot\prn{q_{t-1}^{i'}}}}^2 \\
        \nonumber
        &\qquad+(1+\epsilon) \eta^2 \E\norm{\nabla \fcvtot\prn{q_{t-1}^{i}}-\nabla \fcv{i}\prn{q_{t-1}^{i}}-\nabla \fcvtot\prn{q_{t-1}^{i'}}+ \nabla \fcv{i'} \prn{q_{t-1}^{i'}}}^2 \\
        \nonumber
        &\le \pr{1+\frac{1}{\epsilon}} \E\norm{q_{t-1}^{i}-q_{t-1}^{i'}}^2
        + (1+\epsilon) \eta^2 \E\norm{\nabla \fcvtot\prn{q_{t-1}^{i}}-\nabla \fcv{i}\prn{q_{t-1}^{i}}}^2 \\
        \nonumber
        &\qquad + (1+\epsilon) \eta^2 \E\norm{\nabla \fcvtot\prn{q_{t-1}^{i'}}-\nabla \fcv{i'} \prn{q_{t-1}^{i'}}}^2 \\
        \label{eq:lemA6:5}
        &\qquad -2(1+\epsilon) \eta^2 \E\ps{\nabla \fcvtot\prn{q_{t-1}^{i}}-\nabla \fcv{i}\prn{q_{t-1}^{i}}}{\nabla \fcvtot\prn{q_{t-1}^{i'}}-\nabla \fcv{i'} \prn{q_{t-1}^{i'}}}.
    \end{align}
    The third inequality is implied by the co-coercivity: $\eta\in(0,2/\sum_{i=1}^{\nclients} \lambda_{\yquery}^i H^i]$, $\forall (q,\tilde{q})\in \R^2$,
    \begin{multline*}
        \norm{q - \tilde{q} - \eta \pr{\nabla \fcvtot(q) - \nabla \fcvtot(\tilde{q})}}^2 \\
        = \norm{\tilde{q}-q}^2 - \eta \br{2 \ps{q - \tilde{q}}{\nabla \fcvtot(q) - \nabla \fcvtot(\tilde{q})} + \eta \norm{\nabla \fcvtot(q) - \nabla \fcvtot(\tilde{q})}^2}
        \le \norm{\tilde{q}-q}^2,
    \end{multline*}
    and we also have
    \begin{multline*}
        \sum_{i=1}^{\nclients} \sum_{i'=1}^{\nclients} \lambda_{\yquery}^i \lambda_{\yquery}^{i'} \E\ps{\nabla \fcvtot\prn{q_{t-1}^{i}}-\nabla \fcv{i}\prn{q_{t-1}^{i}}}{\nabla \fcvtot\prn{q_{t-1}^{i'}}-\nabla \fcv{i'} \prn{q_{t-1}^{i'}}} \\
        = \E\norm{\sum_{i=1}^{\nclients} \lambda_{\yquery}^i \pr{\nabla \fcvtot\prn{q_{t-1}^{i}}-\nabla \fcv{i}\prn{q_{t-1}^{i}}}}^2
        \ge 0.
    \end{multline*}
    Therefore, summing~\eqref{eq:lemA6:5} gives that
    \begin{equation*}
        \sum_{i=1}^{\nclients} \sum_{i'=1}^{\nclients} \lambda_{\yquery}^i \lambda_{\yquery}^{i'} \E\normbig{q_{t,k}^{i}-q_{t,k}^{i'}}^2
        \le \pr{1+\frac{1}{\epsilon}} \sum_{i=1}^{\nclients} \sum_{i'=1}^{\nclients} \lambda_{\yquery}^i \lambda_{\yquery}^{i'} \E\norm{q_{t-1}^{i}-q_{t-1}^n}^2
        + 2 \eta^2 \pr{\sigma^2 + \varnoise^2 + (1+\epsilon) \zeta^2}.
    \end{equation*}
    Set $\epsilon=K-1$, since for any $i,i'\in[\nclients]$, $q_{t,0}^{i}=q_{t,0}^{i'}$, we get
    \begin{align*}
        \sum_{i=1}^{\nclients} \sum_{i'=1}^{\nclients} \lambda_{\yquery}^i \lambda_{\yquery}^{i'} \E\normbig{q_{t,k}^{i} - q_{t,k}^{i'}}^2
        &\le 2\eta^2 \pr{\sigma^2 + \varnoise^2 + (1 + \epsilon) \zeta^2} \sum_{k'=0}^{K-1} \pr{1+\frac{1}{\epsilon}}^{k'} \\
        &\le 6K \eta^2 \pr{\sigma^2 + \varnoise^2 + (1 + \epsilon) \zeta^2}.
    \end{align*}
    Since $\sum_{i=1}^{\nclients} \lambda_{\yquery}^{i}=1$, the Jensen's inequality yields that $ \sum_{i=1}^{\nclients} \lambda_{\yquery}^i \E\normbig{q_{t,k}^{i} - \bar{q}_{t,k}}^2 \le \sum_{i=1}^{\nclients} \sum_{i'=1}^{\nclients} \lambda_{\yquery}^i \lambda_{\yquery}^{i'} \E\normbig{q_{t,k}^{i} - q_{t,k}^{i'}}^2$, which concludes the proof.
\end{proof}

In addition, with the previous notations consider the stepsize
\begin{equation}\label{eq:def:eta-cv}
    \eta_\star=\min \ac{\eta_{0}, \prbigg{\frac{\E\norm{\bar{q}_{0} - q_\star}^2}{14 H K^2 T \br{\sigma^2 + K \zeta^2}}}^{1 / 3}},
\end{equation}
and define the average parameter
\begin{equation}\label{eq:def:qhatT}
    \hat{q}_T
    = \frac{1}{T}\sum_{t=0}^{T-1}\ac{\sum_{i=1}^n \lambda_{\yquery}^{i} \br{\frac{1}{K}\sum_{k=0}^{K-1} q_{t,k}^i}}.
\end{equation}

\begin{theorem}\label{thm:cvFL}
    Assume \Cref{ass:suppl:minimizer}-\Cref{ass:suppl:convexity}-\Cref{ass:suppl:smoothness}-\Cref{ass:suppl:variance}-\Cref{ass:suppl:heterogeneity}.
    We consider $\eta\in(0,\eta_{0}]$ with $S_t = [\nclients]$, for all $t\in[T]$.
    Then, for any $t\in\acn{0,\ldots,T-1}$, $k\in\acn{0,\ldots,K-1}$, we have
    \begin{equation*}
        \E \fcvtot\pr{\hat{q}_T}-\fcvtot\prn{q_\star}
        \le \frac{\E\norm{\bar{q}_{0} - q_\star}^2}{\eta K T}
        + 2 \eta^2 \br{6 H (K \zeta)^2 + \sigma^2 \pr{6 H K + \max_{i\in[\nclients]}{\lambda_{\yquery}^i}}},
    \end{equation*}
    where $\eta_{0}, \hat{q}_T$ are given in~\eqref{eq:def:eta0} and~\eqref{eq:def:qhatT}.
    Moreover, for $\eta=\eta_\star$ defined in~\eqref{eq:def:eta-cv} and $H\ge K^{-1}\max_{i\in[\nclients]} \lambda_{\yquery}^i$, it follows
    \begin{equation*}
        \E \fcvtot\pr{\hat{q}_T}-\fcvtot\prn{q_\star}
        \le \frac{\E\norm{\bar{q}_{0} - q_\star}^2}{\eta_{0} K T}
        + \frac{5 (\E\norm{\bar{q}_{0} - q_\star}^2)^{2/3} \br{H \pr{\sigma^2 + K \zeta^2}}^{1/3}}{(K T^2)^{1/3}}.
    \end{equation*}
\end{theorem}

\begin{proof}
    For any $\eta \le \eta_{0}$, using \Cref{lem:E-cvFL} we have
    \begin{equation}\label{eq:thm:cv:1}
        \E\br{\fcvtot\pr{\bar{q}_{t,k}}} - \fcvtot\prn{q_\star} \le \frac{1}{\eta} \E\norm{\bar{q}_{t,k}-q_{\star}}^2-\frac{1}{\eta} \E\norm{\bar{q}_{t,k+1}-q_{\star}}^2
        + 2 H \sum_{i=1}^{\nclients} \lambda_{\yquery}^i \E\norm{\bar{q}_{t,k}-q_{t,k}^{i}}^2
        + 2 \eta^2 \sigma^2 \max_{i\in[\nclients]}\acn{\lambda_{\yquery}^i}.
    \end{equation}
    Moreover, by \Cref{lem:E-biasFL} it follows that
    \begin{equation}\label{eq:thm:cv:2}
        \sum_{i=1}^{\nclients}\lambda_{\yquery}^i \E\norm{q_{t,k}^{i}-\bar{q}_{t,k}}^2 \le 6 K \eta^2 (\sigma^2 +  K \zeta^2).
    \end{equation}
    Combining~\eqref{eq:thm:cv:1} and~\eqref{eq:thm:cv:2}, we obtain
    \begin{align*}
        \E{\br{\fcvtot\pr{\bar{q}_{t,k}}-\fcvtot\prn{q_\star}}}
        &\le \frac{1}{\eta} \E\norm{\bar{q}_{t,k}-q_{\star}}^2-\frac{1}{\eta} \E\norm{\bar{q}_{t,k+1}-q_{\star}}^2
        + 12 H (K \eta \zeta)^2 + 2 (\eta \sigma)^2 \pr{6 H K + \max_{i\in[\nclients]}\acn{\lambda_{\yquery}^i}}.
   \end{align*}
    Moreover, telescoping proves that
    \begin{equation*}
        \sum_{t=0}^{T-1}\sum_{k=0}^{K-1}\br{\E\norm{\bar{q}_{t,k}-q_{\star}}^2 - \E\norm{\bar{q}_{t,k+1}-q_{\star}}^2}
        \le \E\norm{\bar{q}_{0} - q_\star}^2.
    \end{equation*}
    Therefore, the convexity \Cref{ass:suppl:convexity} gives that
    \begin{align*}
        \E\br{\fcvtot \pr{\frac{1}{K T} \sum_{t=1}^{K T} \bar{q}_{t,k}}-\fcvtot\prn{q_\star}}
        &\le \frac{1}{K T} \sum_{t=0}^{T-1} \sum_{k=0}^{K-1} \E\br{\fcvtot\pr{\bar{q}_{t,k}}-\fcvtot\prn{q_\star}} \\
        &\le \frac{\E\norm{\bar{q}_{0} - q_\star}^2}{\eta K T}
        + 2 \eta^2 \br{6 H (K \zeta)^2 + \sigma^2 \pr{6 H K + \max_{i\in[\nclients]}\acn{\lambda_{\yquery}^i}}}.
    \end{align*}
    Finally, the choice of $\eta$ provided in~\eqref{eq:def:eta-cv} ensures that
    \begin{equation*}
        \E\br{\fcvtot\pr{\hat{q}_T}} - \fcvtot\prn{q_\star}
        \le \frac{\E\norm{\bar{q}_{0} - q_\star}^2}{\eta_{0} K T}
        + \frac{5 (\E\norm{\bar{q}_{0} - q_\star}^2)^{2/3} \br{H \pr{\sigma^2 + K \zeta^2}}^{1/3}}{(K T^2)^{1/3}}.
    \end{equation*}
\end{proof}

Now, we denote $\alpha \in (0,1)$ the confidence level, and consider the functions defined for $t\in\acn{0,\ldots,T}, k\in\acn{0,\ldots,K}$ by
\begin{align*}
    \fcvtot = S_{\alpha, \widehat{\mu}_{\yquery}}^{\gamma},&
    &\fcv{i} = S_{\alpha, \widehat{\mu}_{\yquery}^i}^{\gamma}.
\end{align*}
Denote by $q_\star$ the minimizer of $S_{\alpha, \widehat{\mu}_{\yquery}}^{\gamma}=\sum_{i=1}^\nclients \lambda_{\yquery}^i S_{\alpha, \widehat{\mu}_{\yquery}^i}^{\gamma}$, which always exists.
In addition, note that \Cref{ass:suppl:smoothness} and \Cref{ass:suppl:heterogeneity} are satisfied with:
\begin{align}\label{eq:def:cv-varS}
    H^i = \frac{1}{\gamma},&
    &\zeta = \max_{i \in [\nclients]}{\normn{\nabla \lossi - \nabla \lossglobal}_{\infty}^{1/2}}.
\end{align}
\begin{corollary}\label{cor:fedmoreau-cv}
    Let $\gamma\in(0, (\max_{i\in[\nclients]} \lambda_{\yquery}^i)^{-1} K]$ and consider the stepsize $\eta_0=\gamma / 10$.
    Then, for any $t\in\acn{0,\ldots,T-1}$, $k\in\acn{0,\ldots,K-1}$, we have
    \begin{equation*}
        \E S_{\alpha, \widehat{\mu}_{\yquery}}^{\gamma}\pr{\hat{q}_T}-S_{\alpha, \widehat{\mu}_{\yquery}}^{\gamma}\prn{q_\star}
        \le \frac{\E\norm{\bar{q}_{0} - q_\star}^2}{\eta_{0} K T}
        + \frac{5 (\sigma^2 +  K \zeta^2)^{1/3} (\E\norm{\bar{q}_{0} - q_\star}^2)^{2/3}}{(\gamma K T^{2})^{1/3}},
    \end{equation*}
    where $\hat{q}_T$ is provided in~\eqref{eq:def:qhatT}.
\end{corollary}

\begin{proof}
    Since \Cref{ass:suppl:minimizer}-\Cref{ass:suppl:convexity}-\Cref{ass:suppl:smoothness}-\Cref{ass:suppl:variance}-\Cref{ass:suppl:heterogeneity} are satisfied with $\acn{H^i}_{i \in [\nclients]},\zeta$ provided in~\eqref{eq:def:cv-varS}, applying \Cref{thm:cvFL} concludes the proof.
\end{proof}

\section{Theoretical Coverage Guarantee}\label{sec:coverage-fedmoreau}

\subsection{General coverage guarantee}\label{subsec:predsec:Ywidehat}

Consider an increasing sequence $\acn{v_k}_{k \in [\tcount + 1]}\in(\R \cup \acn{+\infty})^{N + 1}$ and $\acn{p_k}_{k \in [\tcount + 1]} \in \Delta_{\tcount+1}$.
For any $\alpha \in [0, 1]$, recall that 
\begin{equation*}
    \textstyle
    \q{1 - \alpha}\pr{\sum_{k = 1}^{N + 1} p_k \delta_{v_k}}
    = \inf\ac{t \in [-\infty, \infty]\colon \prob\pr{V\le t} \ge 1 - \alpha, \text{ where } V \sim \sum_{k = 1}^{N + 1} p_k \delta_{v_k}}.
\end{equation*}

\begin{lemma}\label{lem:quantile-inequality:basis}
    Let $\acn{v_\ell}_{\ell \in [\tcount + 1]}$ be an increasing sequence and $\acn{p_\ell}_{\ell\in [\tcount + 1]} \in \Delta_{\tcount+1}$.
    If $V \sim \sum_{l = 1}^{N + 1} p_l \delta_{v_l}$, then, for all $\alpha \in [0, 1)$, we have
    \begin{equation*}
        \textstyle 1 - \alpha \le \prob\pr{V\le\q{1 - \alpha}\pr{\sum_{k = 1}^{N + 1} p_k \delta_{v_k}}} < 1 - \alpha + \max_{k = 1}^{N + 1} \ac{p_k}.
    \end{equation*}
\end{lemma}

\begin{proof}
    Fix $\alpha \in [0, 1)$, and by convention set $\sum_{k=1}^{0} p_k=0$.
    There exists $k\in[\tcount+1]$, such that $1 - \alpha \in (\sum_{l=1}^{k-1}p_l,\sum_{l=1}^{k}p_l]$, hence $\textstyle \q{1 - \alpha}\prt{\sum_{l=1}^{\tcount+1}p_l \delta_{v_l}} = v_{k}$.
    This last identity implies that
    \begin{equation*}
        \textstyle 1 - \alpha \le 
        \prob\pr{V\le\q{1 - \alpha}\pr{\sum_{l = 1}^{N + 1} p_l \delta_{v_l}}}
        = \sum_{l=1}^{k}p_l
        < 1 - \alpha + \max_{k = 1}^{\tcount+1} \ac{p_k}.
    \end{equation*}
\end{proof}

Denote $\acn{(X_k,Y_k)}_{k\in[\tcount+1]}$ a set of pairwise independent random variables and for $k\in[\tcount+1]$, denote $Z_{k}=(X_k,Y_k)$, $V_k=V(X_k,Y_k)$.
Let $\mathfrak{S}_{\tcount+1}$ be the set of all permutations of $[\tcount+1]$ and consider $\mathfrak{S}_{\tcount+1}^{k}=\acn{\sigma\in \mathfrak{S}_{\tcount+1}\colon \sigma(\tcount+1)=k}$.
Moreover, write $f$ the joint density of $\acn{Z_k}_{k\in[\tcount+1]}$, and for all $k\in[\tcount+1]$ define
\begin{equation}\label{eq:def:approxN-probyy}
    \qweightz[k][z_{1:\tcount+1}] =
    \begin{cases}
        \frac{1}{\tcount+1}&\text{if } \sum_{\sigma\in\mathfrak{S}_{\tcount+1}} f(z_{\sigma(1)},\ldots,z_{\sigma(\tcount+1)})=0 \\
        \frac{\sum_{\sigma\in\mathfrak{S}_{\tcount+1}^{k}} f(z_{\sigma(1)},\ldots,z_{\sigma(\tcount+1)})}{{\sum_{\sigma\in\mathfrak{S}_{\tcount+1}} f(z_{\sigma(1)},\ldots,z_{\sigma(\tcount+1)})}} & \text{otherwise}    
    \end{cases}.
\end{equation}

\begin{lemma}\label{lem:int-exactweights}
    For any $\alpha\in [0,1)$, we have
    \begin{multline*}
        \int \1_{v_{\tcount+1}\le \q{1 - \alpha}\pr{\sum_{k=1}^{\tcount+1} \qweightz[k][z_{1:\tcount+1}] \delta_{v_{k}}}} f(z_1,\ldots,z_{\tcount+1}) \rmd z_1\cdots\rmd z_{\tcount+1}
        \\
        = \int \br{\sum_{k=1}^{\tcount+1} \qweightz[k][z_{1:\tcount+1}] \1_{v_{k}\le \q{1 - \alpha}\pr{\sum_{\bar{k}=1}^{\tcount+1} \qweightz[\bar{k}][z_{1:\tcount+1}] \delta_{v_{\bar{k}}}}}} \br{\sum_{\sigma\in\mathfrak{S}_{\tcount+1}} f(z_{\sigma(1)},\ldots,z_{\sigma(\tcount+1)})} \frac{\rmd z_{1}\cdots\rmd z_{\tcount+1}}{(\tcount+1)!}.
    \end{multline*}
\end{lemma}
\begin{proof}
    First, let's show the invariance of $\sigma\in\mathfrak{S}_{\tcount+1}\mapsto \q{1 - \alpha}\prn{\sum_{k=1}^{\tcount+1} \qweightz[\sigma(k)][z_{\sigma(1):\sigma(\tcount+1)}] \delta_{v_{\sigma(k)}}}\in\R$. For that, fix $\tilde{\sigma}\in \mathfrak{S}_{\tcount+1}$. The invariance is immediate when $\sum_{\sigma\in\mathfrak{S}_{\tcount+1}} f(z_{\sigma(1)},\ldots,z_{\sigma(\tcount+1)}) = 0$. Therefore, assume that $\sum_{\sigma\in\mathfrak{S}_{\tcount+1}} f(z_{\sigma(1)},\ldots,z_{\sigma(\tcount+1)})\neq 0$. We get
    \begin{align*}
        \sum_{k=1}^{\tcount+1} \qweightz[\tilde{\sigma}(k)][z_{\tilde{\sigma}(1):\tilde{\sigma}(\tcount+1)}] \delta_{v_{\tilde{\sigma}(k)}}
        &= \sum_{k=1}^{\tcount+1} \frac{{\sum_{\sigma\in\mathfrak{S}_{\tcount+1}^{\tilde{\sigma}(k)}} f(z_{\sigma(1)},\ldots,z_{\sigma(\tcount+1)})}}{{\sum_{\sigma\in\mathfrak{S}_{\tcount+1}} f(z_{\sigma(1)},\ldots,z_{\sigma(\tcount+1)})}} \delta_{v_{\tilde{\sigma}(k)}} \\
        &= \sum_{k=1}^{\tcount+1} \frac{{\sum_{\sigma\in\mathfrak{S}_{\tcount+1}^{k}} f(z_{\sigma(1)},\ldots,z_{\sigma(\tcount+1)})}}{{\sum_{\sigma\in\mathfrak{S}_{\tcount+1}} f(z_{\sigma(1)},\ldots,z_{\sigma(\tcount+1)})}} \delta_{v_{k}}
        = \sum_{k=1}^{\tcount+1} \qweightz[k][z_{1:\tcount+1}] \delta_{v_{k}}.
    \end{align*}
    Moreover, we can write
    \begin{align*}
        &\int \1_{v_{\tcount+1}\le \q{1 - \alpha}\pr{\sum_{k=1}^{\tcount+1} \qweightz[k][z_{1:\tcount+1}] \delta_{v_{k}}}} f(z_1,\ldots,z_{\tcount+1}) \rmd z_1\cdots\rmd z_{\tcount+1} \\
        &= \sum_{\sigma\in\mathfrak{S}_{\tcount+1}} \int \1_{v_{\sigma(\tcount+1)}\le \q{1 - \alpha}\pr{\sum_{\bar{k}=1}^{\tcount+1} \qweightz[\sigma(\bar{k})][z_{\sigma(1):\sigma(\tcount+1)}] \delta_{v_{\sigma(\bar{k})}}}} f(z_{\sigma(1)},\ldots,z_{\sigma(\tcount+1)}) \frac{\rmd z_{\sigma(1)}\cdots\rmd z_{\sigma(\tcount+1)}}{(\tcount+1)!} \\
        &= \sum_{k=1}^{\tcount+1} \sum_{\sigma\in\mathfrak{S}_{\tcount+1}^{k}} \int \1_{v_{\sigma(\tcount+1)}\le \q{1 - \alpha}\pr{\sum_{\bar{k}=1}^{\tcount+1} \qweightz[\sigma(\bar{k})][z_{\sigma(1):\sigma(\tcount+1)}] \delta_{v_{\sigma(\bar{k})}}}} f(z_{\sigma(1)},\ldots,z_{\sigma(\tcount+1)}) \frac{\rmd z_{\sigma(1)}\cdots\rmd z_{\sigma(\tcount+1)}}{(\tcount+1)!} \\
        &= \sum_{k=1}^{\tcount+1} \int \1_{v_{k}\le \q{1 - \alpha}\pr{\sum_{\bar{k}=1}^{\tcount+1} \qweightz[\bar{k}][z_{1:\tcount+1}] \delta_{v_{\bar{k}}}}} \br{\sum_{\sigma\in\mathfrak{S}_{\tcount+1}^{k}} f(z_{\sigma(1)},\ldots,z_{\sigma(\tcount+1)})} \frac{\rmd z_{1}\cdots\rmd z_{\tcount+1}}{(\tcount+1)!} \\
        &= \sum_{k=1}^{\tcount+1} \int \1_{v_{k}\le \q{1 - \alpha}\pr{\sum_{\bar{k}=1}^{\tcount+1} \qweightz[\bar{k}][z_{1:\tcount+1}] \delta_{v_{\bar{k}}}}} \qweightz[k][z_{1:\tcount+1}] \br{\sum_{\sigma\in\mathfrak{S}_{\tcount+1}} f(z_{\sigma(1)},\ldots,z_{\sigma(\tcount+1)})} \frac{\rmd z_{1}\cdots\rmd z_{\tcount+1}}{(\tcount+1)!}.
    \end{align*}
\end{proof}

Given $\mathrm{z}=(\xquery,\yquery)\in\XC\times\YC$ define 
\begin{align*}
    D_{\tcount}^{\mathrm{z}}=\pr{z_1,\ldots,z_{\tcount},\mathrm{z}},&
    &\mu_{\yquery}^{\tcount} = \qweightz[\tcount+1][D_{\tcount}^{\mathrm{z}}] \delta_{1} + \sum_{k=1}^{\tcount} \qweightz[k][D_{\tcount}^{\mathrm{z}}] \delta_{V_k},
\end{align*}
and consider the prediction set given by
\begin{align*}\label{eq:def:CalphaN}
    \mathcal{C}_{\alpha,\mu^{\tcount}}(\xquery) = \ac{\yquery\in\YC\colon V(\xquery,\yquery) \le \q{1-\alpha}(\mu_{\yquery}^{\tcount})}.
\end{align*}
\begin{theorem}\label{lem:YinCN}
    Assume there are no ties between $\acn{V_k}_{k \in [\tcount+1]}$ almost surely.
    Then, for any $\alpha \in [0, 1)$, we have
    \begin{equation*}
        1 - \alpha
        \le \prob\pr{Y_{\tcount+1} \in \mathcal{C}_{\alpha,\mu^{\tcount}}(X_{\tcount+1})}
        \le 1 - \alpha + \E\br{\max_{k=1}^{\tcount+1} \acn{\qweightz[k][Z_{1:\tcount+1}]}},
    \end{equation*}
    where $\qweightz[k][Z_{1:\tcount+1}]$ is defined in~\eqref{eq:def:approxN-probyy}.
\end{theorem}
\begin{proof}
    Let $\alpha$ be in $[0, 1)$ and for any $(x_k,y_k)\in\XC\times\YC$, denote $z_k=(x_k,y_k)$, $v_k=V(x_k,y_k)$. First, we can write
    \begin{align*}
        \nonumber
        &\prob\pr{Y_{\tcount+1} \in \mathcal{C}_{\alpha,\mu^{\tcount}}(X_{\tcount+1})}
        = \prob\pr{Y_{\tcount+1}\in \ac{\yquery \in \YC\colon V(X_{\tcount+1}, \yquery) \le \q{1 - \alpha}\prn{\mu_{\yquery}^{\tcount}}}} \\
        \nonumber
        &=\prob\pr{V(X_{\tcount+1},Y_{\tcount+1}) \le \q{1 - \alpha}\prn{\mu_{Y_{\tcount+1}}^{\tcount}}} \\
        \nonumber
        &\overset{(\star)}{=}\textstyle\prob\pr{V(X_{\tcount+1}, Y_{\tcount+1}) \le \q{1 - \alpha}\pr{\sum_{k=1}^{\tcount+1} \qweightz[k][Z_{1:\tcount+1}] \delta_{V_{k}}}} \\
        &=\int_{(\XC\times\YC)^{\tcount+1}} \1_{v_{\tcount+1}\le \q{1 - \alpha}\pr{\sum_{k=1}^{\tcount+1} \qweightz[k][z_{1:\tcount+1}] \delta_{V_{k}}}} f(z_1,\ldots,z_{\tcount+1}) \rmd z_1\cdots\rmd z_{\tcount+1}.
    \end{align*}
    Where $(\star)$ holds since 
    \begin{align*}
        &\textstyle\pr{V_{\tcount+1} \le \q{1 - \alpha}\prn{\mu_{Y_{\tcount+1}}^{\tcount}}}
        \\
        &\textstyle\iff \pr{\qweightz[\tcount+1][D_{\tcount}^{(X_{\tcount+1},Y_{\tcount+1})}] \delta_{1\le V_{\tcount+1}} + \sum_{k=1}^{\tcount} \qweightz[k][D_{\tcount}^{(X_{\tcount+1},Y_{\tcount+1})}] \delta_{V_k\le V_{\tcount+1}} \le \alpha} \\
        &\textstyle\iff \pr{\qweightz[\tcount+1][D_{\tcount}^{(X_{\tcount+1},Y_{\tcount+1})}] \delta_{V_{\tcount+1}\le V_{\tcount+1}} + \sum_{k=1}^{\tcount} \qweightz[k][D_{\tcount}^{(X_{\tcount+1},Y_{\tcount+1})}] \delta_{V_k\le V_{\tcount+1}} \le \alpha} \\
        &\textstyle\iff \pr{V(X_{\tcount+1}, Y_{\tcount+1}) \le \q{1 - \alpha}\pr{\sum_{k=1}^{\tcount+1} \qweightz[k][Z_{1:\tcount+1}] \delta_{V_{k}}}}.
    \end{align*}
    Define the set $E\subset (\XC\times\YC)^{\tcount+1}$ of points such that the non-conformity scores are pairwise distinct:
    \begin{align*}
        &E = \act{(z_1,\ldots,z_{\tcount+1})\in(\XC\times\YC)^{\tcount+1} \colon \prod_{k<\ell}\pr{v(x_k,y_k)-v(x_{\ell},y_{\ell})}\neq 0},\\
        &E^{c} = (\XC\times\YC)^{\tcount+1}\setminus E,\\
        &F = \act{(z_1,\ldots,z_{\tcount+1})\in(\XC\times\YC)^{\tcount+1} \colon \sum_{\sigma\in\mathfrak{S}_{\tcount+1}} f(z_{\sigma(1)},\ldots,z_{\sigma(\tcount+1)}) \neq 0}
    \end{align*}
    In addition, combining \Cref{lem:int-exactweights} with the no-tie assumption on $\acn{V_k}_{k\in[\tcount+1]}$ gives that
    \begin{align}\label{eq:eq:int-exactweights}
        &\int \1_{v_{\tcount+1}\le \q{1 - \alpha}\pr{\sum_{k=1}^{\tcount+1} \qweightz[k][z_{1:\tcount+1}] \delta_{v_{k}}}} f(z_1,\ldots,z_{\tcount+1}) \rmd z_1\cdots\rmd z_{\tcount+1} \\
        \nonumber
        &= \int_{E\cap F} \br{\sum_{k=1}^{\tcount+1} \qweightz[k][z_{1:\tcount+1}] \1_{v_{k}\le \q{1 - \alpha}\pr{\sum_{\bar{k}=1}^{\tcount+1} \qweightz[\bar{k}][z_{1:\tcount+1}] \delta_{v_{\bar{k}}}}}} \br{\sum_{\sigma\in\mathfrak{S}_{\tcount+1}} f(z_{\sigma(1)},\ldots,z_{\sigma(\tcount+1)})} \frac{\rmd z_{1}\cdots\rmd z_{\tcount+1}}{(\tcount+1)!}.
    \end{align}
    Consider the random variable $V \sim \sum_{k=1}^{\tcount+1} \qweightz[k][z_{1:\tcount+1}] \delta_{v_{k}}$, we have
    \begin{equation*}
        \prob\pr{V \le \q{1 - \alpha}\pr{\sum_{\bar{k}=1}^{\tcount+1} \qweightz[\bar{k}][z_{1:\tcount+1}] \delta_{v_{\bar{k}}}}}
        = \sum_{k=1}^{\tcount+1} \qweightz[k][z_{1:\tcount+1}] \1_{v_{k}\le \q{1 - \alpha}\pr{\sum_{\bar{k}=1}^{\tcount+1} \qweightz[\bar{k}][z_{1:\tcount+1}] \delta_{v_{\bar{k}}}}}.
    \end{equation*}
    Therefore, applying \Cref{lem:quantile-inequality:basis} on $(z_1,\ldots,z_{\tcount+1})\in E\cap F$ implies that
    \begin{equation}\label{eq:bound:thm-coverage:1}
        1 - \alpha
        \le \sum_{k=1}^{\tcount+1} \qweightz[k][z_{1:\tcount+1}] \1_{v_{k}\le \q{1 - \alpha}\pr{\sum_{\bar{k}=1}^{\tcount+1} \qweightz[\bar{k}][z_{1:\tcount+1}] \delta_{v_{\bar{k}}}}}
        \le 1 - \alpha + \max_{k=1}^{\tcount+1} \acn{\qweightz[k][z_{1:\tcount+1}]}.
    \end{equation}
    Lastly, using that 
    \begin{multline}\label{eq:bound:thm-coverage:2}
        \int_{E\cap F} \brBigg{\sum_{\sigma\in\mathfrak{S}_{\tcount+1}} f(z_{\sigma(1)},\ldots,z_{\sigma(\tcount+1)})} \frac{\rmd z_{1}\cdots\rmd z_{\tcount+1}}{(\tcount+1)!}
        \\
        = \int_{E} \brBigg{\sum_{\sigma\in\mathfrak{S}_{\tcount+1}} f(z_{\sigma(1)},\ldots,z_{\sigma(\tcount+1)})} \frac{\rmd z_{1}\cdots\rmd z_{\tcount+1}}{(\tcount+1)!}
        = 1,
    \end{multline}
    and combining the bounds \eqref{eq:bound:thm-coverage:1}-\eqref{eq:bound:thm-coverage:2} with \eqref{eq:eq:int-exactweights} yields the result.
\end{proof}

\subsection{Proof of \Cref{thm:hatYinhatC}}\label{subsec:predsec:Yexact}

First, recall that 
\begin{equation*}
    \mathcal{I} = \ac{(\target, \ccount{\target}+1)}\cup \ac{(i,k)\colon i\in[\nclients], k\in[\ccount{i}]}.
\end{equation*}
For any $\{(\mathrm{x}_k^i, \mathrm{y}_k^i)\}_{(i,k)\in\mathcal{I}}\in(\XC\times\YC)^{\tcount+1}$, we define the set
\begin{align*}
    D_{\tcount}^{(\mathrm{x}_{\ccount{\target}+1}^{\target}, \mathrm{y}_{\ccount{\target}+1}^{\target})} = \acn{(\mathrm{x}_k^i, \mathrm{y}_k^i) \colon i\in[\nclients], k\in[\ccount{i}]}\cup\acn{(\mathrm{x}_{\ccount{\target}+1}^{\target}, \mathrm{y}_{\ccount{\target}+1}^{\target})}.
\end{align*}
We consider a bijection $(\phi,\varphi)$ between the set $[\tcount+1]$ and $\mathcal{I}$.
This bijection is defined for any $k\in[\tcount]$ as follows:
\begin{equation*}
    (\phi(k),\varphi(k)) =
    \begin{cases}
        (j,\ell) & \text{if $1\le \ell := k - \sum_{i=1}^{j-1}\ccount{i} \le \sum_{i=1}^{j}\ccount{i}$} \\
        (\target,\ccount{\target}+1) & \text{otherwise}
    \end{cases}
    .
\end{equation*}
Recall that $\forall i\in[\nclients]$ and $y\in\YC$, the likelihood ratio is given by
\[
    w_{y}^{i} = \frac{P_{Y}^i(y)}{P_{Y}^{\mathrm{cal}}(y)},
\]
and for all $(i,k)\in\mathcal{I}$ we write 
\begin{equation}\label{eq:def:Wki-tibshirani-FL}
    \begin{aligned}
        &\mathfrak{P}_{k}^{i} = \ac{\sigma\in\mathfrak{S}_{\tcount+1}\colon\phi(\sigma(\tcount+1))=i,\varphi(\sigma(\tcount+1))=k},\\
        &W_{k}^{i}(D_{\tcount}^{(\mathrm{x}_{\ccount{\target}+1}^{\target}, \mathrm{y}_{\ccount{\target}+1}^{\target})}) = w_{\mathrm{y}_{k}^{i}}^{\target} \sum_{\sigma\in\mathfrak{P}_{k}^i} \prod_{\ell=1}^{\tcount} w_{\mathrm{y}_{\varphi(\ell)}^{\phi(\ell)}}^{\phi(\ell)}.
    \end{aligned}
\end{equation}
Given the set of points $D_{\tcount}^{(\xquery,\yquery)}$, for all $(i,k)\in\mathcal{I}$ define
\begin{equation}\label{eq:def:prob-federated}
    \qweightp{i}{k}
    = \frac{W_{k}^{i}(D_{\tcount}^{(\xquery,\yquery)})}{\sum_{\ell=1}^{\tcount+1} W_{\varphi(\ell)}^{\phi(\ell)}(D_{\tcount}^{(\xquery,\yquery)})}.
\end{equation}
Finally, define the probability measure and the prediction set given by
\begin{align*}
    &\mu_{\yquery}^{\target} = \qweightp{\target}{\ccount{\target}+1} \delta_{1} + \sum_{i=1}^{\nclients}\sum_{k=1}^{\ccount{i}} \qweightp{i}{k} \delta_{V_k^i}, \\
    \label{eq:def:Calpha-federated}
    &\mathcal{C}_{\alpha,\mu^{\target}}(\xquery) = \ac{\yquery\in\YC\colon V(\xquery,\yquery) \le \q{1-\alpha}(\mu_{\yquery}^{\target})}.
\end{align*}

\begin{theorem}\label{thm:exact-coverage-Zk}
    If \Cref{ass:iid-noties} holds, then for any $\alpha \in [0, 1)$, we have
    \begin{equation*}
        1 - \alpha
        \le \prob\pr{\Ytarget \in \predsetg{\alpha}{\qdistrp{}}(\Xtarget)}
        \le 1 - \alpha + \E\br{\max_{(i,k)\in\mathcal{I}} \acn{\qweightp{k}{i}}},
    \end{equation*}
    where $\qweightp{i}{k}$ is defined in~\eqref{eq:def:prob-federated}.
\end{theorem}

\begin{proof}
    By independence, the joint density $f$ of $\acn{(X_\ell^j,Y_\ell^j)\colon (j,\ell)\in \mathcal{I}}$ with respect to $(P_{X|Y}\times P_{Y}^{\mathrm{cal}})^{\otimes (\tcount+1)}$ is given for $\{(x_k^i,y_k^i)\colon (i,k)\in \mathcal{I}\}\in(\XC\times\YC)^{\tcount+1}$ by
    \begin{equation*}
        f\pr{(x_{1}^{1},y_{1}^{1}),\ldots,(x_{\ccount{1}}^{1},y_{\ccount{1}}^{1}),\ldots,(x_{1}^{\nclients},y_{1}^{\nclients}),\ldots,(x_{\ccount{\nclients}}^{\nclients},y_{\ccount{\nclients}}^{\nclients}),(x_{\ccount{\target}+1}^{\target}, y_{\ccount{\target}+1}^{\target})}
        = w_{y_{\ccount{\target}+1}^{\target}}^{\target} \prod_{j=1}^{\nclients} \prod_{\ell=1}^{\ccount{j}} w_{y_{\ell}^j}^{j}.
    \end{equation*}
    Using the definition of $\qweightp{i}{k}$~\eqref{eq:def:prob-federated}, for all $(i,k)\in\mathcal{I}$ we have
    \begin{align*}
        \qweightp{i}{k}
        &= \frac{W_{k}^{i}(D_{\tcount+1})}{\sum_{\ell=1}^{\tcount+1} W_{\varphi(\ell)}^{\phi(\ell)}(D_{\tcount+1})}
        \\
        &= \frac{\sum_{\sigma\in\mathfrak{P}_{k}^{i}} f\prn{z_{\varphi\circ\sigma(1)}^{\phi\circ\sigma(1)},\ldots,z_{\varphi\circ\sigma(\tcount+1)}^{\phi\circ\sigma(\tcount+1)}}}{\sum_{\sigma\in\mathfrak{S}_{\tcount+1}} f\prn{z_{\varphi\circ\sigma(1)}^{\phi\circ\sigma(1)},\ldots,z_{\varphi\circ\sigma(\tcount+1)}^{\phi\circ\sigma(\tcount+1)}}}
        \\
        &= \frac{\sum_{\sigma\in\mathfrak{S}_{\tcount+1}\colon \sigma(\tcount+1)=(\phi,\varphi)^{-1}(i,k)} f\prn{z_{\varphi\circ\sigma(1)}^{\phi\circ\sigma(1)},\ldots,z_{\varphi\circ\sigma(\tcount+1)}^{\phi\circ\sigma(\tcount+1)}}}{\sum_{\sigma\in\mathfrak{S}_{\tcount+1}} f\prn{z_{\varphi\circ\sigma(1)}^{\phi\circ\sigma(1)},\ldots,z_{\varphi\circ\sigma(\tcount+1)}^{\phi\circ\sigma(\tcount+1)}}}
        .
    \end{align*}
    Therefore, applying \Cref{lem:YinCN} concludes the proof.
\end{proof}

\subsection{Proof of \Cref{lem:YminushatY:general} and Equation~\eqref{eq:bound:YwidehatY}}\label{subsec:predsec:YwidehatY}

In this section, we consider a probability measure $P_{Y}^{\mathrm{cal}}$ dominating $P_{Y}^\target$ and suppose that the likelihood ratios are given by $\iweightpN{y}=P_{Y}^{\target}(y)/P_{Y}^{\mathrm{cal}}(y)$.
Let $\{\iweighthatN{y}\}_{y\in\YC}$ be fixed, and $\forall y \in \YC$, define $\nu_{y} = [\sum_{\tilde{y} \in \YC}\iweighthatN{\tilde{y}}P_{Y}^{\mathrm{cal}}\prn{\tilde{y}}]^{-1} \iweighthatN{y}P_{Y}^{\mathrm{cal}}\prn{y}$.
Denote $\widehat{Y}_{\ccount{\target}+1}^{\target}$ a multinomial random variable of parameter $\nu$ independent of the calibration dataset $\ac{(X_k^i, Y_k^i)\colon k \in [\ccount{i}]}_{i \in [\nclients]}$ and write $\mathcal{P}(\YC)$ the partition of the set $\YC$.
Moreover, denote $\widehat{P}_{Y}^{\target}$ the probability distribution of $\widehat{Y}_{\ccount{\target}+1}^{\target}$, and consider $\widehat{X}_{\ccount{\target}+1}^{\target}$ such that $(\widehat{X}_{\ccount{\target}+1}^{\target}, \widehat{Y}_{\ccount{\target}+1}^{\target})\sim P_{X|Y}\times \widehat{P}_{Y}^{\target}$.

\begin{lemma}\label{lem:Prob-diff:YwidehatY}
    For any prediction set $\mathcal{C}\colon \XC\to \mathcal{P}(\YC)$ independent of $(\Xtarget,\Ytarget)$ and $(\widehat{X}_{\ccount{\target}+1}^{\target},\widehat{Y}_{\ccount{\target}+1}^{\target})$, we have
    \begin{equation*}
        \abs{\prob\prn{\Ytarget \in\mathcal{C}(\Xtarget)} - \prob\prn{\widehat{Y}_{\ccount{\target}+1}^{\target} \in\mathcal{C}(\widehat{X}_{\ccount{\target}+1}^{\target})}}
        \le \frac{1}{2} \sum_{y \in \YC} \abs{P_{Y}^{\target}\prn{y} - \frac{\iweighthatN{y} P_{Y}^{\mathrm{cal}}\prn{y}}{\sum_{\tilde{y} \in \YC}\iweighthatN{\tilde{y}}P_{Y}^{\mathrm{cal}}\prn{\tilde{y}}}}.
    \end{equation*}
\end{lemma}
\begin{proof}
    Developing the left-hand side as follows, we get
    \begin{align*}
        &\abs{\prob\pr{\Ytarget \in\mathcal{C}(\Xtarget)} - \prob\pr{\widehat{Y}_{\ccount{\target}+1}^{\target} \in\mathcal{C}(\widehat{X}_{\ccount{\target}+1}^{\target})}}
        \\
        &= \abs{\E\br{\1_{\Ytarget \in\mathcal{C}(\Xtarget)}} - \E\br{\1_{\widehat{Y}_{\ccount{\target}+1}^{\target} \in\mathcal{C}(\widehat{X}_{\ccount{\target}+1}^{\target})}}}
        \\
        &= \abs{\sum_{y\in\YC} P_{Y}^{\mathrm{cal}}\prn{y} \pr{\iweightpN{y} - \frac{\iweighthatN{y}}{\sum_{\tilde{y} \in \YC}\iweighthatN{\tilde{y}}P_{Y}^{\mathrm{cal}}\prn{\tilde{y}}}} \int \E\1_{y \in\mathcal{C}(x)} \rmd P_{X|Y=y}(x)}
        \\
        &\le \frac{1}{2} \sum_{y \in \YC} P_{Y}^{\mathrm{cal}}\prn{y} \abs{\iweightpN{y} - \frac{\iweighthatN{y}}{\sum_{\tilde{y} \in \YC}\iweighthatN{\tilde{y}}P_{Y}^{\mathrm{cal}}\prn{\tilde{y}}}}.
    \end{align*}
    Finally, using that $P_{Y}^{\mathrm{cal}}\prn{y} \iweightpN{y}=P_{Y}^\target(y)$ concludes the proof.
\end{proof}

\begin{remark}
    If for some probability atoms $\acn{m_{i}}_{i\in[\nclients]}\in\Delta_{\nclients}$, we know good approximations $\iweighthatN{y}$ of the likelihood ratios $[(P_{Y}^{\mathrm{cal}})^{-1}\prn{\sum_{i=1}^{\nclients} m_{i} P_{Y}^{i}}](y)$. Then, \Cref{lem:Prob-diff:YwidehatY} implies the following result:
    \begin{multline*}
        \abs{\prob\pr{\Ytarget \in\mathcal{C}(\Xtarget)} - \prob\pr{\widehat{Y}_{\ccount{\target}+1}^{\target} \in\mathcal{C}(\widehat{X}_{\ccount{\target}+1}^{\target})}}
        \\
        \le \tv\pr{P_{Y}^{\target}, \sum_{i=1}^{\nclients}m_{i} P_{Y}^{i}}
        + \frac{1}{2} \sum_{y \in \YC} \abs{\pr{\sum_{i=1}^{\nclients} m_{i} P_{Y}^{i}}\prn{y} - \frac{\iweighthatN{y} P_{Y}^{\mathrm{cal}}\prn{y}}{\sum_{\tilde{y} \in \YC}\iweighthatN{\tilde{y}}P_{Y}^{\mathrm{cal}}\prn{\tilde{y}}}}.
    \end{multline*}
\end{remark}


\begin{lemma}\label{lem:bound-to-simplify-TvApprox}
    If $\absn{\YC}\ge 2$ and $\Mtcount\in \N^{\star}$, then we have
    \begin{equation*}
        \absn{\YC} \exp\pr{-\Mtcount \min_{y\in\YC}\acn{P_{Y}^{\mathrm{cal}}\prn{y}}} \wedge 1
        \le \sqrt{\frac{2 \log \absn{\YC}}{(\log 2) \Mtcount \min_{y\in\YC}\acn{P_{Y}^{\mathrm{cal}}\prn{y}}}}.
    \end{equation*}
\end{lemma}
\begin{proof}
    Introduce the set $E = \acn{2\log\absn{\YC} \le \Mtcount \min_{y\in\YC}\acn{P_{Y}^{\mathrm{cal}}\prn{y}}}$, we obtain
    \begin{equation}\label{eq:ref:lemmae12}
        \absn{\YC} \exp\prBig{-\Mtcount \min_{y\in\YC}\acn{P_{Y}^{\mathrm{cal}}\prn{y}}} \wedge 1
        \le \1_{E} \exp\prBig{-\Mtcount \min_{y\in\YC}\acn{P_{Y}^{\mathrm{cal}}\prn{y}}/2} + \1_{E^c}.
    \end{equation}
    We also have that for all $x\ge 0$, that $\rmexp^{-x}\le 1/\sqrt{x}$. Using this inequality on the first right-side term implies that
    \begin{align*}
        &\exp\pr{-\Mtcount \min_{y\in\YC}\acn{P_{Y}^{\mathrm{cal}}\prn{y}}/2} \le \sqrt{\frac{2}{\Mtcount \min_{y\in\YC}\acn{P_{Y}^{\mathrm{cal}}\prn{y}}}} \le \sqrt{\frac{\log \absn{\YC}}{\log 2}} \sqrt{\frac{2}{\Mtcount \min_{y\in\YC}\acn{P_{Y}^{\mathrm{cal}}\prn{y}}}}.
    \end{align*}
    Moreover, remark that
    \begin{align*}
        &\1_{E^c} \le \1_{E^c} \sqrt{\frac{2 \log \absn{\YC}}{\Mtcount \min_{y\in\YC}\acn{P_{Y}^{\mathrm{cal}}\prn{y}}}}.
    \end{align*}
    Finally, plugging the two previous inequalities in \eqref{eq:ref:lemmae12} concludes the proof.
\end{proof}
Recall that $\Mclcount{y}{i}$ denotes the number of training data on agent $i$ associated to label $y\in\YC$.
Consider the total number of local data $\Mccount{\target} = \sum_{y\in\YC} \Mclcount{y}{i}$, the number of training data with label $y$ is given by $\Mlcount{y} = \sum_{i=1}^{\nclients} \Mclcount{y}{i}$, and the total number of samples on all agents is written by $\Mtcount = \sum_{y\in\YC} \Mlcount{y}$.
Recall that the likelihood ratios and the weights are given for any labels $(y, \yquery) \in \YC^2$ by
\begin{align*}\label{eq:def:approx-probyyhat}
    \iweight{y} =
    \begin{cases}
        \frac{\Mtcount \Mclcount{y}{\target}}{\Mccount{\target} \Mlcount{y}} & \text{if $\Mlcount{y}\ge 1$}\\
        0 & \text{otherwise}
    \end{cases},&
    &\qweight{y}{\yquery}
    = \frac{(\nofrac{\Mclcount{y}{\target}}{\Mlcount{y}}) \cdot\1_{\Mlcount{y}\ge 1}}{\Mccount{\target} + (\nofrac{\Mclcount{\yquery}{\target}}{\Mlcount{\yquery}}) \cdot\1_{\Mlcount{\yquery}\ge 1}}.
\end{align*}
For any $y \in \YC$, we also consider $\nu \in \Delta_{\absn{\YC}}^{\Q}=\acn{p' \in \Q_+^{\absn{\YC}}\colon \sum_{y \in \YC}p_y'=1}$ defined by
\[
    \nu_y = \frac{\iweighthatN{y}P_{Y}^{\mathrm{cal}}\prn{y}}{\sum_{\tilde{y} \in \YC}\iweighthatN{\tilde{y}}P_{Y}^{\mathrm{cal}}\prn{\tilde{y}}}.
\]
For any parameter $p \in \Delta_{\absn{\YC}}^{\Q}$, denote $M_p$ a multinomial random variable independent of the training/calibration datasets, and define $\widehat{Y}_{\ccount{\target}+1}^{\target} = M_{\nu}$. For any set $A$ in the partition of $\YC$, we have $M_{\nu}^{-1}(A)=\cup_{p\in\Delta_{\absn{\YC}}^{\Q}}\{\nu^{-1}(\acn{p}) \cap M_p^{-1}(A)\}$. Therefore, $\widehat{Y}_{\ccount{\target}+1}^{\target}$ is a valid random variable.
Given the target coverage level $1-\alpha$, recall that the prediction set is defined for any $\xquery \in \XC$ by
\begin{equation*}
    \begin{aligned}
        &\textstyle\widehat{\mu}_{\yquery}^{\scalebox{.35}{$\mathrm{MLE}$}} = \qweight{\yquery}{\yquery} \delta_{1} + \sum_{i=1}^{\nclients} \sum_{k=1}^{\ccount{i}\wedge \Bar{\tcount}^i} \qweight{Y_k^i}{\yquery} \delta_{V_{k}^{i}},,
        \\
        &\predsetmle(\xquery) = \ac{\yquery\colon V(\xquery, \yquery) \le \q{1 - \alpha}\pr{\widehat{\mu}_{\yquery}^{\scalebox{.35}{$\mathrm{MLE}$}}}}.
    \end{aligned}
\end{equation*}
Since the considered likelihood ratios are now depending on the training dataset, it is no longer possible to apply \Cref{lem:Prob-diff:YwidehatY}.
However, conditioning by the training dataset, a similar reasoning shows that
\begin{align}\label{eq:bound:third-term}
    \abs{\prob\prn{\Ytarget \in{\predsetmle}(\Xtarget)} - \prob\prn{\widehat{Y}_{\ccount{\target}+1}^{\target} \in\predsetmle(\widehat{X}_{\ccount{\target}+1}^{\target})}}
    \le \frac{1}{2} \sum_{y \in \YC} \E \abs{P_{Y}^{\target}\prn{y} - \frac{\iweighthatN{y} P_{Y}^{\mathrm{cal}}\prn{y}}{\sum_{\tilde{y} \in \YC}\iweighthatN{\tilde{y}}P_{Y}^{\mathrm{cal}}\prn{\tilde{y}}}}.
\end{align}
By utilizing the following lemma combined with \eqref{eq:bound:third-term}, we can control the difference between the probabilities of the events $\Ytarget \in{\predsetmle}(\Xtarget)$ and $\widehat{Y}_{\ccount{\target}+1}^{\target} \in\predsetmle(\widehat{X}_{\ccount{\target}+1}^{\target})$.
\begin{theorem}\label{lem:Prob-diff:approxTv}
    For any $\alpha \in (0, 1)$, we have
    \begin{equation*}
        \sum_{y \in \YC} \E \abs{P_{Y}^{\target}\prn{y} - \frac{\iweighthatN{y} P_{Y}^{\mathrm{cal}}\prn{y}}{\sum_{\tilde{y} \in \YC}\iweighthatN{\tilde{y}}P_{Y}^{\mathrm{cal}}\prn{\tilde{y}}}}
        \le \frac{6}{\sqrt{\Mccount{\target}}} + 12 \sqrt{\frac{\log \absn{\YC} + \log \Mccount{\target}}{\Mtcount \min_{y \in \YC}\acn{P_{Y}^{\mathrm{cal}}\prn{y}}}}
        .
    \end{equation*}
\end{theorem}
\begin{proof}
    For any $y \in \YC$, introduce the following quantities:
    $\hat{f}_y^{\target} = \Mclcount{y}{\target}/\Mccount{\target}$,
    $f^{\target} = \acn{\tpmf{y}}_{y\in\YC}$, 
    $\hat{f}^{\target} = \acn{\hat{f}_y^{\target}}_{y \in \YC}$, and
    $\widehat{r}=\acn{(P_{Y}^{\mathrm{cal}}\prn{y} \Mtcount / \Mlcount{y}) \1_{\Mlcount{y}>0}}_{y\in\YC}$.
    We denote $\odot$ the Hadamard product, i.e., for any vectors $a,b\in\R^{\absn{\YC}}$, $a \odot b$ is the vector of the component-wise product between $a$ and $b$. We now bound the quantity in the right-hand side of the previous inequality, we obtain
    \begin{align*}
        \sum_{y \in \YC} \E \abs{P_{Y}^{\target}\prn{y} - \frac{\iweighthatN{y} P_{Y}^{\mathrm{cal}}\prn{y}}{\sum_{\tilde{y} \in \YC}\iweighthatN{\tilde{y}}P_{Y}^{\mathrm{cal}}\prn{\tilde{y}}}}
        &= \sum_{y\in\YC} \E\abs{f_y^{\target} - \hat{f}_y^{\target} 
            + \hat{f}_y^{\target} - \frac{\hat{f}_y^{\target} \widehat{r}_y}{\sum_{\tilde{y}\in\YC}\hat{f}_{\tilde{y}}^{\target} \widehat{r}_{\tilde{y}}}}
        \\
        &\le \E\norm{f^{\target} - \hat{f}^{\target} + \hat{f}^{\target} - \frac{\hat{f}^{\target} \odot \widehat{r}}{1 + \psn{\hat{f}^{\target}}{\widehat{r} - \1}}}_1
        \\
        &\le \E\normn{f^{\target} - \hat{f}^{\target}}_1
            + \E\norm{\frac{\hat{f}^{\target} \psn{\hat{f}^{\target}}{\widehat{r} - \1} - \hat{f}^{\target} \odot (\widehat{r} - \1)}{1 + \psn{\hat{f}^{\target}}{\widehat{r} - \1}}}_1.
    \end{align*}
    First, we establish the following equality that will be injected in the computation of $\E\normn{f^{\target} - \hat{f}^{\target}}_1$.
    \begin{equation}\label{eq:fnormune}
        \normn{f^{\target} - \hat{f}^{\target}}_1
        = \max_{u\in[-1/2,1/2]^{\absn{\YC}}}\psn{f^{\target} - \hat{f}^{\target}}{u + \1 / 2}
        = \max_{u\in[0,1]^{\absn{\YC}}}\psn{f^{\target} - \hat{f}^{\target}}{u}.
    \end{equation}
    Then, using the result provided by \citep[Lemma C.2]{agrawal2017optimistic}, for any $\delta\in (0,1)$, we get
    \begin{equation}
    \label{eq:agrawalineq}
        \prob\pr{\max_{u \in [0, 1]^{\absn{\YC}}}\psn{f^{\target} - \hat{f}^{\target}}{u} \ge \sqrt{\frac{- 2\log \delta}{\Mccount{\target}}}}
        \le \delta.
    \end{equation}
    Looking back to $\E\normn{f^{\target} - \hat{f}^{\target}}_1$ and using the two previous identities, the following lines hold
    \begin{align*}
        \E\normn{f^{\target} - \hat{f}^{\target}}_1
        &= \int_{t\ge 0} \prob\pr{\normn{f^{\target} - \hat{f}^{\target}}_1 \ge t} \rmd t \\
        &\le \delta + \int_{t\ge \delta} \prob\pr{\normn{f^{\target} - \hat{f}^{\target}}_1 \ge t} \rmd t \\
        &\le \delta + \int_{t\ge \delta} \exp\pr{-\nofrac{\Mccount{\target} t^2}{2}} \rmd t \tag*{using~\eqref{eq:fnormune}-\eqref{eq:agrawalineq}}\\
        &\le \delta + \frac{1}{\sqrt{\Mccount{\target}}} \int_{t \ge \delta\sqrt{\Mccount{\target}}} \exp\pr{-\nofrac{t^2}{2}} \rmd t .
    \end{align*}
    After optimizing for $\delta > 0$, we can retrieve the following upper bound
    \begin{equation*}
        \E\normn{f^{\target} - \hat{f}^{\target}}_1
        \le \frac{1.4}{\sqrt{\Mccount{\target}}}.
    \end{equation*}
    Let $\epsilon\in (0,1/2]$, we have that
    \begin{equation*}
        \norm{\frac{\hat{f}^{\target} \psn{\hat{f}^{\target}}{\widehat{r} - \1} - \hat{f}^{\target} \odot (\widehat{r} - \1)}{1 + \psn{\hat{f}^{\target}}{\widehat{r} - \1}}}_1
        \le 2 \1_{\norm{\widehat{r} - \1}_{\infty}>\epsilon} + \frac{2\epsilon}{1-\epsilon}.
    \end{equation*}
    Taking the expectation for both sides, it shows
    \begin{align*}
    \E\norm{\frac{\hat{f}^{\target} \psn{\hat{f}^{\target}}{\widehat{r} - \1} - \hat{f}^{\target} \odot (\widehat{r} - \1)}{1 + \psn{\hat{f}^{\target}}{\widehat{r} - \1}}}_1 \le 2\prob\pr{\norm{\widehat{r} - \1}_{\infty}>\epsilon} + \frac{2\epsilon}{1-\epsilon}.
    \end{align*}
    We now upper bound the first term in the right-hand side of the inequality, we obtain
    \begin{align}
        \nonumber
        \prob\pr{\norm{\widehat{r} - \1}_{\infty}>\epsilon}
        &= \prob\pr{\max_{y\in\YC} \abs{\frac{P_{Y}^{\mathrm{cal}}\prn{y} \1_{\Mlcount{y}>0}}{\Mlcount{y} / \Mtcount} - 1}>\epsilon} \\
        \nonumber
        &\le \sum_{y\in\YC} \prob\pr{\abs{\frac{P_{Y}^{\mathrm{cal}}\prn{y} \1_{\Mlcount{y}>0}}{\Mlcount{y} / \Mtcount} - 1}>\epsilon} \\
        \label{eq:bound:prob-r}
        &\le \sum_{y\in\YC} \ac{
            \prob\pr{\frac{\Mlcount{y}}{\Mtcount} < \frac{P_{Y}^{\mathrm{cal}}\prn{y} \1_{\Mlcount{y}>0}}{1 + \epsilon}}
            + \prob\pr{\frac{\Mlcount{y}}{\Mtcount} > \frac{P_{Y}^{\mathrm{cal}}\prn{y}}{1 - \epsilon}}}.
    \end{align}
    Since the random variable $\Mlcount{y}$ is the sum of independent Bernoulli random variables, using the Chernoff bound it follows that
    \begin{equation}\label{eq:bound:prob-count-ratios}
        \begin{aligned}
            &\prob\pr{\nofrac{\Mlcount{y}}{\Mtcount} < \nofrac{P_{Y}^{\mathrm{cal}}\prn{y}}{\prn{1 + \epsilon}}} \le \exp\pr{-\nofrac{2\epsilon^2 N P_{Y}^{\mathrm{cal}}\prn{y}}{9}}, \\
            &\prob\pr{\nofrac{\Mlcount{y}}{\Mtcount} > \nofrac{P_{Y}^{\mathrm{cal}}\prn{y}}{\prn{1 - \epsilon}}} \le \exp\pr{-\nofrac{4\epsilon^2 N P_{Y}^{\mathrm{cal}}\prn{y}}{3}}.
        \end{aligned}
    \end{equation}
    Therefore, combining~\eqref{eq:bound:prob-r} with~\eqref{eq:bound:prob-count-ratios} gives
    \begin{equation*}
        \prob\pr{\norm{\widehat{r} - \1}_{\infty}>\epsilon}
        \le \sum_{y\in\YC} \br{\prob\pr{\Mlcount{y}=0} + 2 \exp\pr{-\nofrac{\epsilon^2 N P_{Y}^{\mathrm{cal}}\prn{y}}{5}}}.
    \end{equation*}
    Putting all the previous results together, we obtain
    \begin{equation}\label{eq:bound:sumPy}
        \sum_{y\in\YC} P_{Y}^{\mathrm{cal}}\prn{y} \E\abs{\iweightp{y} - \frac{\iweight{y}}{\sum_{\tilde{y}\in\YC}\iweight{\tilde{y}}P_{Y}^{\mathrm{cal}}\prn{\tilde{y}}}}
        \le \frac{1.4}{\sqrt{\Mccount{\target}}} + 4 \epsilon
        + 4 \sum_{y\in\YC} \exp\pr{-\nofrac{2\epsilon^2 N P_{Y}^{\mathrm{cal}}\prn{y}}{9}}
        + \sum_{y\in\YC} \prn{1 - P_{Y}^{\mathrm{cal}}\prn{y}}^{\Mtcount}.
    \end{equation}
    Consider the following quantity
    \begin{equation*}
        \epsilon = \frac{3}{2} \sqrt{\frac{2 \log \absn{\YC} + \log \Mccount{\target}}{\Mtcount \min_{y\in\YC}\acn{P_{Y}^{\mathrm{cal}}\prn{y}}}}.
    \end{equation*}
    If $\epsilon \le 1/2$, it yields that
    \begin{equation*}
        \sum_{y\in\YC} \exp\pr{-\nofrac{2\epsilon^2 N P_{Y}^{\mathrm{cal}}\prn{y}}{9}}
        \le \sum_{y\in\YC} \exp\pr{- \log \absn{\YC} - (1/2) \log \Mccount{\target}}
       = \frac{1}{\sqrt{\Mccount{\target}}}.
    \end{equation*}
    Therefore, combining this last inequality with~\eqref{eq:bound:third-term}-\eqref{eq:bound:sumPy} implies that
    \begin{multline*}
        \abs{\prob\prn{\Ytarget \in {\predsetmle}(\Xtarget)} - \prob\prn{\widehat{Y} \in {\predsetmle}(\Xtarget)}}
        \le \frac{3}{\sqrt{\Mccount{\target}}} + 3\sqrt{\frac{2 \log \absn{\YC} + \log \Mccount{\target}}{\Mtcount \min_{y\in\YC}\acn{P_{Y}^{\mathrm{cal}}\prn{y}}}}
        \\
        + \absn{\YC} \pr{1-\min_{y\in\YC}\acn{P_{Y}^{\mathrm{cal}}\prn{y}}}^{\Mtcount} \wedge 1.
    \end{multline*}
    Otherwise, if $\epsilon>1/2$, then, the last inequality immediately holds since the right-hand term is greater than $1$.
    Lastly, applying \Cref{lem:bound-to-simplify-TvApprox} concludes the proof
\end{proof}

\subsection{Proof of \Cref{thm:right-coverage}}\label{subsec:predsec:moreau}

First, for all $i\in[\nclients]$, denote by $F_{V}^{i}:u\in [0,1]\mapsto \prob\prn{V(X,Y)\le u} \in [0,1]$ the cumulative distribution function of $V(X^i,Y^i)$, where $(X^i,Y^i) \sim P^{\mathrm{i}}$.
Recall that $\tcount = \sum_{i = 1}^{\nclients}\ccount{i}$, $\mathcal{I}={(i,k)\colon i\in[\nclients], k\in[\ccount{i}]}\cup\{(\target,\ccount{\target}+1)\}$, and also that there is almost surely no ties between the $\{V(X_k^i,Y_k^i)\}_{(i,k)\in\mathcal{I}}$.
To simplify the notation, we re-index $\{(X_k^i, Y_k^i, V_k^i)\}_{(i,k)\in\mathcal{I}}$ into $\acn{X_k, Y_k, V_k}_{k \in [\tcount + 1]}$, sorted such that $\acn{V_k}_{k \in [\tcount + 1]}$ is non-decreasing.

Now, we consider $\alpha \in [0, 1] \setminus \Q$. Using \Cref{thm:FedMoreau-approx}, this condition ensures the existence and uniqueness of $k_{\mathrm{opt}}\in[\tcount+1]$ such that $V_{k_{\mathrm{opt}}}=\argmin_{q\in\R}\ac{\E_{V\sim \widehat{\mu}_{\yquery}}\br{\pinball[\alpha][V](q)}}$, and this condition also proves the existence and uniqueness of $\qmoreau[1 - \alpha]$ minimizing $\{\E_{V \sim \widehat{\mu}_{\yquery}}\brn{\pinballmoreau{V}{q}} \colon q \in \R\}$.
Moreover, for any $k \in [\tcount + 1]\setminus\acn{k_{\mathrm{opt}}}$, define
\begin{equation}\label{eq:def:alphak}
    \qthmhat =
    \begin{cases}
        \sum_{\ell\colon V_{\ell}\in (V_{k_{\mathrm{opt}}}, V_k]} \qweight{Y_\ell}{\Ytarget},& \text{if $k>k_{\mathrm{opt}}$}\\
        \sum_{\ell\colon V_{\ell}\in [V_k, V_{k_{\mathrm{opt}}})} \qweight{Y_\ell}{\Ytarget},& \text{if $k<k_{\mathrm{opt}}$}
    \end{cases}.
\end{equation}

\begin{lemma}\label{lem:moreau-approx-error}
    Let $\alpha \in [0, 1] \setminus \Q$, for any $V_k\in [\min(\qhatmoreau[1 - \alpha][\gamma], V_{k_{\mathrm{opt}}}), \max(\qhatmoreau[1 - \alpha][\gamma], V_{k_{\mathrm{opt}}})]$, we have
    \begin{equation*}
        \absn{\qhatmoreau[1 - \alpha][\gamma] - V_k}
        \le \qthmhat^{-1}\pr{\lossglobal \prn{\qhatmoreau[1 - \alpha][\gamma]} - \lossglobal \prn{\qmoreau[1 - \alpha][\gamma]\prn{\widehat{\mu}_{\yquery}}}}
        + \qthmhat^{-1} \gamma,
    \end{equation*}
    where $\qthmhat$ is given in~\eqref{eq:def:alphak}.
\end{lemma}
\begin{proof}
    First, suppose that $V_{k_{\mathrm{opt}}} \le V_k < \qhatmoreau[1 - \alpha][\gamma]$. Since $\partial S_{\alpha, \widehat{\mu}_{\yquery}}(V_k) = \qthmhat + \partial S_{\alpha, \widehat{\mu}_{\yquery}}(V_{k_{\mathrm{opt}}})$, the convexity of $S_{\alpha, \widehat{\mu}_{\yquery}}$ implies that
    \begin{multline}\label{eq:bound:qhatmoreau-Vk}
        \qhatmoreau[1 - \alpha][\gamma] - V_k
        \le \qthmhat^{-1}\pr{S_{\alpha, \widehat{\mu}_{\yquery}} \prn{\qhatmoreau[1 - \alpha][\gamma]} - S_{\alpha, \widehat{\mu}_{\yquery}} \prn{V_k}}
        \\
        \le \qthmhat^{-1}\pr{\lossglobal \prn{\qhatmoreau[1 - \alpha][\gamma]} - \lossglobal \prn{\qmoreau[1 - \alpha][\gamma]\prn{\widehat{\mu}_{\yquery}}}}
            + \qthmhat^{-1}\gamma.
    \end{multline}
    The last inequality holds since $\normn{\lossglobal - S_{\alpha, \widehat{\mu}_{\yquery}}}_{\infty}\le \gamma/2$.
    Moreover, \eqref{eq:bound:qhatmoreau-Vk} is immediately satisfied when $V_k=\qhatmoreau[1 - \alpha][\gamma]$.
    Therefore, \eqref{eq:bound:qhatmoreau-Vk} holds for all $V_k\in [V_{k_{\mathrm{opt}}}, \qhatmoreau[1 - \alpha][\gamma]]$.
    Finally, the same lines show that~\eqref{eq:bound:qhatmoreau-Vk} is also satisfied when $\qhatmoreau[1 - \alpha][\gamma]\le V_k\le V_{k_{\mathrm{opt}}}$.
\end{proof}

For any $\gamma>0$ and $T\in\N^{\star}$, recall that $\qmoreau[1 - \alpha][\gamma]\prn{\widehat{\mu}_{\yquery}}$ and $\qhatmoreau[1 - \alpha][\gamma]$ are defined in \eqref{eq:def:qmoreau}, \eqref{eq:def:qhatmoreau}.
Consider
\begin{equation}\label{eq:def:CTgamma}
    C_{T}^{\gamma} = \frac{2 \tcount \sum_{y\in\YC} P_Y^{\mathrm{cal}}(y) \iweight{y}}{\min_{y\in\YC}\iweight{y}} \br{\E \lossglobal \prn{\qhatmoreau[1 - \alpha][\gamma]} - \lossglobal \prn{\qmoreau[1 - \alpha][\gamma]\prn{\widehat{\mu}_{\yquery}}} + \gamma}
\end{equation}
and define
\begin{equation}\label{eq:def:kc}
    k_{\mathrm{c}} =
    \begin{cases}
        k_{\mathrm{opt}} + \mathrm{Ent}\pr{\sqrt{\frac{m\tcount C_{T}^{\gamma}}{2 \log\tcount}}} & \text{if $k_{\mathrm{opt}} + \mathrm{Ent}\pr{\sqrt{\frac{m\tcount C_{T}^{\gamma}}{2 \log\tcount}}} \le \tcount+1$} \\
        k_{\mathrm{opt}} - \mathrm{Ent}\pr{\sqrt{\frac{m\tcount C_{T}^{\gamma}}{2 \log\tcount}}} & \text{otherwise}
    \end{cases}.
\end{equation}

\begin{lemma}\label{lem:bound:Vkopt-Vkc}
    Assume there exists $m>0$ such that for any $i\in[\nclients]$, $P_{V}^i$ admits a density $f_V^i$ with respect to the Lebesgue measure that satisfies $m \le f_{V}^i$.
    Fix $\alpha \in [0, 1] \setminus \Q$, and suppose that $C_{T}^{\gamma}<2 m^{-1} \tcount \log \tcount$.
    We have $k_{\mathrm{c}}\in[\tcount+1]$, and with probability at least $m (2\tcount \log \tcount)^{-1} + \tcount^{-2}$ the next inequality holds
    \begin{equation*}
        \absn{V_{k_{\mathrm{c}}} - V_{k_{\mathrm{opt}}}}
        \le \sqrt{\frac{2 C_{T}^{\gamma} \log \tcount}{m \tcount}} + \frac{2 \log \tcount}{m \tcount}.
    \end{equation*}
\end{lemma}
\begin{proof}
    Since we suppose $\alpha \in [0, 1] \setminus \Q$, we can apply \Cref{thm:FedMoreau-approx} to prove the existence and uniqueness of $k_{\mathrm{opt}}\in[\tcount+1]$ such that $\q{1 - \alpha}(\widehat{\mu}_{\yquery}) = V_{k_{\mathrm{opt}}}$.
    Since by assumption we have
    \begin{equation*}
        \frac{m\tcount C_{T}^{\gamma}}{2 \log\tcount} < (\tcount + 1)^2.
    \end{equation*}
    Therefore, it holds that
    \begin{equation*}
        \mathrm{Ent}\pr{\sqrt{\frac{m\tcount C_{T}^{\gamma}}{2 \log\tcount}}} \le \tcount.
    \end{equation*}
    Hence, we deduce that $k_{\mathrm{c}}\in[\tcount+1]$.
    Next, consider
    \begin{equation*}
        L_{\tcount} = \frac{2 \log \tcount}{m \tcount}
    \end{equation*}
    and denote by $P_{N}$ the partitioned obtained by splitting the interval $[0,1]$ into intervals of length $L_{\tcount}$.
    Finally, define
    \begin{equation*}
        A_{\tcount} = \ac{\forall S\in P_N, \exists (i,k)\in\mathcal{I}, V_k^i\in S}.
    \end{equation*}
    Note that $\absn{P_N} = \lceil 1/L_N \rceil \le m\tcount/(2 \log \tcount) + 1$ and $\log(1-m L_N)\le -m L_N$, thus
    \begin{align*}
        \prob\pr{A_N}
        &\le \sum_{S\in P_N} \prod_{(i,k)\in\mathcal{I}} \prob\pr{V_k^i\notin S}\\
        &\le \sum_{S\in P_N} \prod_{i=1}^{\nclients} \prob\pr{V_1^i\notin S}^{\ccount{i}+\1_{\target}(i)}\\
        &\le \absn{P_N} \prn{1 - m L_N}^{\tcount+1} \\
        &\le \pr{m\tcount/(2 \log \tcount) + 1} \exp\pr{- m (\tcount + 1) L_N} \\
        &\le \frac{m}{2\tcount \log \tcount} + \frac{1}{\tcount^2}.
    \end{align*}
    Without loss of generality, we can assume that $k_{\mathrm{opt}}<k_{\mathrm{c}}$.
    Denote $K = \{k_{\mathrm{opt}},\ldots,k_{\mathrm{c}}\}$ the indices between $k_{\mathrm{c}}$ and $k_{\mathrm{opt}}$.
    Consider $\mathcal{S}=\{I\in P_N\colon \exists k\in K, V_k\in I\}$, on the event $A_N$ we get
    \begin{align*}
        \absn{V_{k_{\mathrm{c}}} - V_{k_{\mathrm{opt}}}}
        &\le \sum_{I\in S}\absn{I} \\
        &\le L_N (\absn{k_{\mathrm{c}} - k_{\mathrm{opt}}} + 1) \\
        &\le L_N \sqrt{\frac{m\tcount C_{T}^{\gamma}}{2 \log\tcount}} + L_N
        = \sqrt{C_{T}^{\gamma} L_N} + L_N.
    \end{align*}
\end{proof}

\begin{lemma}\label{lem:bound:prob-minqweight}
    For any $(i,k)\in\mathcal{I}$, assume that $(X_k^i,Y_k^i)$ is distributed according to $P_{X|Y}\times P_{Y}^{i}$ and suppose the random variables are pairwise independent.
    We have
    \begin{equation*}
        \prob\pr{\min_{y\in\YC} \qweight{y}{\Ytarget} < \frac{\min_{y\in\YC} \iweight{y}}{2 \tcount \sum_{y\in\YC} P_Y^{\mathrm{cal}}(y) \iweight{y}}}
        \le \frac{4 \var(\iweight{Y^{\mathrm{cal}}})}{\tcount (\E\iweight{Y^{\mathrm{cal}}})^2} + \frac{2 \E \iweight{\Ytarget}}{\tcount \E\iweight{Y^{\mathrm{cal}}}}
        .
    \end{equation*}
\end{lemma}
\begin{proof}
    First, recall that $I=\{(i,k)\colon i\in[\nclients], k\in[\ccount{i}]\}\cup \{(\target,\ccount{\target}+1)\}$.
    We have
    \begin{align*}
        \E\br{\sum_{(i,k)\in\mathcal{I}} \iweight{Y_k^i}}
        &= \sum_{i=1}^{\nclients} \sum_{y\in\YC} \pr{\ccount{i} + \1_{\target}(i)} P_Y^i(y) \iweight{y} \\
        &= \sum_{y\in\YC} P_Y^{\target}(y) \iweight{y} + \tcount \sum_{y\in\YC} \pr{\sum_{i=1}^{\nclients} \pi_i P_Y^i(y)} \iweight{y} \\
        &= \sum_{y\in\YC} \br{P_Y^{\target}(y) + \tcount P_Y^{\mathrm{cal}}(y)} \iweight{y}
        .
    \end{align*}
    Therefore, using the Bienaymé-Tchebytchev inequality implies that
    \begin{align*}
        \nonumber
        &\prob\pr{\min_{(i,k)\in\mathcal{I}} \acn{\qweight{Y_k^i}{\Ytarget}} < \frac{\min_{y\in\YC}\iweight{y}}{2 \tcount \sum_{y\in\YC} P_Y^{\mathrm{cal}}(y) \iweight{y}}}
        \\
        \nonumber
        &= \prob\pr{\min_{(i,k)\in\mathcal{I}} \acn{\iweight{Y_k^i}} < \frac{\min_{y\in\YC}\iweight{y}}{2 \tcount \sum_{y\in\YC} P_Y^{\mathrm{cal}}(y) \iweight{y}} \sum_{(i,k)\in\mathcal{I}} \iweight{Y_k^i}} \\
        \nonumber
        &\le \prob\pr{\sum_{(i,k)\in\mathcal{I}} \iweight{Y_k^i} \ge 2 \tcount \sum_{y\in\YC} P_Y^{\mathrm{cal}}(y) \iweight{y}} \\
        \nonumber
        &\le \prob\pr{\sum_{i=1}^{\nclients} \sum_{k=1}^{\ccount{i}} \pr{\iweight{Y_k^i}- \E\iweight{Y_k^i}} \ge \frac{\tcount}{2} \sum_{y\in\YC} P_Y^{\mathrm{cal}}(y) \iweight{y}} 
        + \prob\pr{\iweight{\Ytarget} \ge \frac{\tcount}{2} \sum_{y\in\YC} P_Y^{\mathrm{cal}}(y) \iweight{y}}
        \\
        &\le \frac{4 \var(\iweight{Y^{\mathrm{cal}}})}{\tcount (\E\iweight{Y^{\mathrm{cal}}})^2} + \frac{2 \E \iweight{\Ytarget}}{\tcount \E\iweight{Y^{\mathrm{cal}}}}
        .
    \end{align*}
\end{proof}

\begin{theorem}\label{lem:Prob-diff:QFL}
    Assume there exist $m,M>0$ such that for any $i\in[\nclients]$, $P_{V}^i$ admits a density $f_V^i$ with respect to the Lebesgue measure that satisfies $m \le f_{V}^i \le M$.
    Let $\alpha \in [0, 1] \setminus \Q$, and suppose that $C_{T}^{\gamma}<2 m^{-1} \tcount \log \tcount$.
    It holds
    \begin{multline}\label{eq:Prob-diff:QFL}
        \abs{\prob\prn{\Ytarget \in \predsethatmoreau{\alpha}(\Xtarget)} - \prob\prn{\Ytarget \in \mathcal{C}_{\alpha,\widehat{\mu}}(\Xtarget)}}
        \le 
        3 \Flip \sqrt{\frac{2 C_{T}^{\gamma} \log \tcount}{m \tcount}}
        \\
        + \frac{2 \Flip \log \tcount}{m \tcount}
        + \frac{4 \var(\iweight{Y^{\mathrm{cal}}})}{\tcount (\E\iweight{Y^{\mathrm{cal}}})^2} + \frac{2 \E \iweight{\Ytarget}}{\tcount \E\iweight{Y^{\mathrm{cal}}}}
        + \frac{m}{2\tcount \log \tcount} + \frac{1}{\tcount^{2}}
        ,
    \end{multline}
    where $C_T^{\gamma}$ is defined in \eqref{eq:def:CTgamma}.
\end{theorem}
\begin{proof}
    Using the definitions of $\predsethatmoreau{\alpha}(\Xtarget)$, $\mathcal{C}_{\alpha,\widehat{\mu}}(\Xtarget)$ provided in~\Cref{algo:conf-fed} and \eqref{eq:def:approx-muytarget}, we can write
    \begin{multline}\label{eq:Prob-diff:QFL:1}
        \prob\prn{\Ytarget \in \predsethatmoreau{\alpha}(\Xtarget)} - \prob\prn{\Ytarget \in \mathcal{C}_{\alpha,\widehat{\mu}}(\Xtarget)}
        \\
        = \prob\pr{V(\Xtarget,\Ytarget) \le \qhatmoreau[1 - \alpha][\gamma]} - \prob\pr{V(\Xtarget,\Ytarget) \le \q{1 - \alpha}(\widehat{\mu}_{\yquery})}.
    \end{multline}
    Recall that $k_{\mathrm{opt}}=\argmin_{k\in[\tcount+1]}\E_{V \sim \widehat{\mu}_{\yquery}}[\pinballmoreau{V}{V_k}]$ is a random variable and consider the event
    \begin{equation*}
        B_{\tcount}
        = \ac{\absn{V_{k_{\mathrm{c}}} - V_{k_{\mathrm{opt}}}} \le \sqrt{C_{T}^{\gamma} L_N} + L_N}
        \cap \ac{\min_{y\in\YC} \qweight{y}{\Ytarget} \ge \frac{\min_{y\in\YC} \iweight{y}}{2 \tcount \sum_{y\in\YC} P_Y^{\mathrm{cal}}(y) \iweight{y}}}
    \end{equation*}    
    and denote by $B_{\tcount}^{c}$ its complement.
    Since $V(\Xtarget,\Ytarget)$ and $\{\qhatmoreau[1 - \alpha][\gamma], \q{1 - \alpha}(\widehat{\mu}_{\yquery})\}$ are independent, we have
    \begin{multline}\label{eq:bound:pvtarget}
        \abs{\prob\prn{V(\Xtarget,\Ytarget) \le \qhatmoreau[1 - \alpha][\gamma]} - \prob\prn{V(\Xtarget,\Ytarget) \le \q{1 - \alpha}(\widehat{\mu}_{\yquery})}}
        \\
        = \abs{\E\br{\1_{V(\Xtarget,\Ytarget) \le \qhatmoreau[1 - \alpha][\gamma]} - \1_{V(\Xtarget,\Ytarget) \le \q{1 - \alpha}(\widehat{\mu}_{\yquery})}}}
        \\
        \le \prob\pr{B_{\tcount}^{c}} + \E\br{\1_{B_{\tcount}}\abs{F_{V^{\target}} \prn{\qhatmoreau[1 - \alpha][\gamma]} - F_{V^{\target}}\prn{\q{1 - \alpha}(\widehat{\mu}_{\yquery})}}}.
    \end{multline}
    The following inequality holds
    \begin{equation*} 
        \abs{F_{V^{\target}} \prn{\qhatmoreau[1 - \alpha][\gamma]} - F_{V^{\target}}\prn{\q{1 - \alpha}(\widehat{\mu}_{\yquery})}}
        \le \normn{F_{V^{\target}} \prn{\cdot + \qhatmoreau[1 - \alpha][\gamma] - \q{1 - \alpha}(\widehat{\mu}_{\yquery})} - F_{V^{\target}}}_{\infty}.
    \end{equation*}
    Thus, using that $F_{V^{\target}}$ is $\Flip$-Lipschitz, we get
    \begin{equation}\label{eq:Prob-diff:QFL:3}
        \E\br{\1_{B_{\tcount}}\normn{F_{V^{\target}} \prn{\cdot + \qhatmoreau[1 - \alpha][\gamma] - \q{1 - \alpha}(\widehat{\mu}_{\yquery})} - F_{V^{\target}}}_{\infty}}
        \le \Flip \E\br{\1_{B_{\tcount}}\absn{\qhatmoreau[1 - \alpha][\gamma] - \q{1 - \alpha}(\widehat{\mu}_{\yquery})}}.
    \end{equation}
    Furthermore, we have
    \begin{multline*} 
        \E\br{\1_{B_{\tcount}}\absn{\qhatmoreau[1 - \alpha][\gamma] - \q{1 - \alpha}(\widehat{\mu}_{\yquery})}}
        \\
        \le \E\brbig{\1_{B_{\tcount}}\abs{V_{k_{\mathrm{c}}} - V_{k_{\mathrm{opt}}}}}
        + \E\br{\1_{V_{k_{\mathrm{c}}}\in [\min(\qhatmoreau[1 - \alpha][\gamma], V_{k_{\mathrm{opt}}}), \max(\qhatmoreau[1 - \alpha][\gamma], V_{k_{\mathrm{opt}}})]} \absn{\qhatmoreau[1 - \alpha][\gamma] - V_{k_{\mathrm{c}}}}}
        .
    \end{multline*}
    Applying \Cref{lem:moreau-approx-error}, this implies that
    \begin{multline*} 
        \1_{V_{k_{\mathrm{c}}}\in [\min(\qhatmoreau[1 - \alpha][\gamma], V_{k_{\mathrm{opt}}}), \max(\qhatmoreau[1 - \alpha][\gamma], V_{k_{\mathrm{opt}}})]}
        \absn{\qhatmoreau[1 - \alpha][\gamma] - V_{k_{\mathrm{c}}}} \\
        \le 
        \begin{cases}
            \qthmhat^{-1}\pr{\lossglobal \prn{\qhatmoreau[1 - \alpha][\gamma]} - \lossglobal \prn{\qmoreau[1 - \alpha][\gamma]\prn{\widehat{\mu}_{\yquery}}}} + \qthmhat^{-1}\gamma & \text{if $k_{\mathrm{c}} \neq k_{\mathrm{opt}}$} \\
            0 &\text{otherwise}
        \end{cases}
        ,
    \end{multline*}
    where recall that $\qthmhat$ is defined in~\eqref{eq:def:alphak} and also that $\mathcal{I}=\{(i,k)\colon i\in[\nclients], k\in[\ccount{i}]\}\cup\{(\target,\ccount{\target}+1)\}$.
    Moreover, on the event $B_{\tcount}$, we immediately have that
    \begin{equation*}
        \qthmhat \ge \absn{k_{\mathrm{c}} - k_{\mathrm{opt}}} \min_{(i,k)\in\mathcal{I}} \acn{\qweight{Y_k^i}{\Ytarget}} \ge \frac{\absn{k_{\mathrm{c}}-k_{\mathrm{opt}}}\min_{y\in\YC}\iweight{y}}{2 \tcount \sum_{y\in\YC} P_Y^{\mathrm{cal}}(y) \iweight{y}}.
    \end{equation*}
    Finally, recall that $k_{\mathrm{c}}$ is given in \eqref{eq:def:kc} and suppose that $C_{T}^{\gamma}<2 m^{-1} \tcount \log \tcount$. Therefore, using the bound provided in \Cref{lem:bound:Vkopt-Vkc} implies that
    \begin{multline}\label{eq:Prob-diff:QFL:8}
        \E\br{\1_{B_{\tcount}}\absn{\qhatmoreau[1 - \alpha][\gamma] - \q{1 - \alpha}(\widehat{\mu}_{\yquery})}}
        \\
        \le \E\brbig{\1_{B_{\tcount}}\abs{V_{k_{\mathrm{c}}} - V_{k_{\mathrm{opt}}}}}
        + \E\br{
                \frac{\pr{\E \lossglobal \prn{\qhatmoreau[1 - \alpha][\gamma]} - \lossglobal \prn{\qmoreau[1 - \alpha][\gamma]\prn{\widehat{\mu}_{\yquery}}} + \gamma} \1_{k_{\mathrm{c}}\neq k_{\mathrm{opt}}}}{\prn{2 \tcount \sum_{y\in\YC} P_Y^{\mathrm{cal}}(y) \iweight{y}}^{-1} \absn{k_{\mathrm{c}}-k_{\mathrm{opt}}} \min_{y\in\YC}\iweight{y}}
        }
        \\
        \le \sqrt{\frac{2 C_{T}^{\gamma} \log \tcount}{m \tcount}} + \frac{2 \log \tcount}{m \tcount}
        + \frac{C_{T}^{\gamma} \1_{m\tcount C_{T}^{\gamma}\ge 2\log\tcount}}{\mathrm{Ent}\pr{\sqrt{\frac{m\tcount C_{T}^{\gamma}}{2 \log\tcount}}}}.
    \end{multline}
    Combining~\eqref{eq:Prob-diff:QFL:1}-\eqref{eq:bound:pvtarget}-\eqref{eq:Prob-diff:QFL:3}-\eqref{eq:Prob-diff:QFL:8} shows that
    \begin{equation}\label{eq:Prob-diff:QFL:9}
        \abs{\prob\prn{\Ytarget \in \predsethatmoreau{\alpha}(\Xtarget)} - \prob\prn{\Ytarget \in \mathcal{C}_{\alpha,\widehat{\mu}}(\Xtarget)}}
        \le \prob\pr{B_{\tcount}^{c}} + \Flip \pr{3 \sqrt{\frac{2 C_{T}^{\gamma} \log \tcount}{m \tcount}} + \frac{2 \log \tcount}{m \tcount}}.
    \end{equation}
    Using \Cref{lem:bound:prob-minqweight} gives that
    \begin{equation}\label{eq:bound:BNc}
        \prob\pr{B_{\tcount}^{c}}
        \le \frac{4 \var(\iweight{Y^{\mathrm{cal}}})}{\tcount (\E\iweight{Y^{\mathrm{cal}}})^2} + \frac{2 \E \iweight{\Ytarget}}{\tcount \E\iweight{Y^{\mathrm{cal}}}}
        + \frac{m}{2\tcount \log \tcount} + \frac{1}{\tcount^{2}}
        .
    \end{equation}
    Lastly, plugging \eqref{eq:bound:BNc} into \eqref{eq:Prob-diff:QFL:9} concludes the proof when $C_{T}^{\gamma}<2 m^{-1} \tcount \log \tcount$. However, if $C_{T}^{\gamma}\ge 2 m^{-1} \tcount \log \tcount$ then \eqref{eq:Prob-diff:QFL} immediately holds.
    Thus, \eqref{eq:Prob-diff:QFL} always holds.
\end{proof}

\subsection{Proofs of \Cref{thm:YinCN-approx} and \Cref{cor:main-coverage}}\label{sec:cov-tv}

Recall that $\acn{\pi_{i}}_{i\in[\nclients]}\in\Delta_{\nclients}$ and $P^{\mathrm{cal}}=\sum_{i=1}^{\nclients}\pi_{i} P^{i}$.
Moreover, draw $(\widehat{X}_{\ccount{\target}+1}^{\target},\widehat{Y}_{\ccount{\target}+1}^{\target})$ according to $P_{X|Y}\times \widehat{P}_{Y}^{\target}$, where $\widehat{P}_{Y}^{\target}$ is defined in \Cref{subsec:predsec:YwidehatY} and denote $\widehat{V}_{\tcount+1}=V(\widehat{X}_{\ccount{\target}+1}^{\target},\widehat{Y}_{\ccount{\target}+1}^{\target})$.
In the following paragraph, we explain how to construct a sequence $\acn{V_{k}}_{k\in[\tcount]}$ of i.i.d. random variables distributed according to $P^{\mathrm{cal}}(V)$ -- see \Cref{lem:Pcal-Vk}, where $P^{\mathrm{cal}}(V)$ denotes the distribution of $V(X,Y)$ with $(X,Y)\sim P^{\mathrm{cal}}$. We also explain the construction of a bijection $\psi:\acn{(i,k)\colon i\in[\nclients], k\in[\ccount{i}]}\to [\tcount]$.
For all $k\in[\tcount]$, draw $M_{k}$ according to a categorical random variable with parameter $\acn{\pi_{i}}_{i\in[\nclients]}$ and define $\Bar{\tcount}_{k}^{i}=\sum_{l=1}^{k}\indiacc{i}(M_l)$.
If $\Bar{\tcount}_{k}^{M_{k}}\le \ccount{M_{k}}$ then define $(X_{k},Y_{k})\gets (X_{\Bar{\tcount}_{k}^{M_{k}}}^{M_{k}},Y_{\Bar{\tcount}_{k}^{M_{k}}}^{M_{k}})$ and $\psi(M_{k},\Bar{\tcount}_{k}^{M_{k}})=k$.
Else $\Bar{\tcount}_{k}^{M_{k}}> \ccount{M_{k}}$, then draw $(X_{k},Y_{k})$ according to $P^{M_{k}}$ -- where we recall that $P^{M_{k}}$ is the distribution of calibration of agent $M_k\in[\nclients]$.
If there exists $k\in[\tcount]$ such that $\Bar{\tcount}_{k}^{M_{k}}> \ccount{M_{k}}$, consider
\begin{align*}
    J_{0} = \ac{k\in[\tcount]\colon \Bar{\tcount}_{k}^{M_{k}}> \ccount{M_{k}}},&
    &J_{1} = \ac{(i,k)\in[\nclients]\times\N\colon \Bar{\tcount}_{\tcount}^{i}< k \le \ccount{i}}.
\end{align*}
Since the following inequalities hold:
\begin{align*}
    \card{J_{0}} + \sum_{i=1}^{\nclients} \min(\ccount{i}, \Bar{\tcount}_{\tcount}^{i}) = \tcount,&
    &\card{J_{1}} + \sum_{i=1}^{\nclients} \min(\ccount{i}, \Bar{\tcount}_{\tcount}^{i}) = \tcount,
\end{align*}
we deduce that $\card{J_{0}}=\card{J_{1}}$.
Moreover, using the existence of $k\in[\tcount]$ such that $\Bar{\tcount}_{k}^{M_{k}}> \ccount{M_{k}}$, we deduce that $J_{0}\neq \emptyset$.
Therefore, there exists a bijection $\varphi:J_{0}\to J_{1}$.
We have previously defined $\psi$ on $\acn{(i,k)\colon i\in[\nclients], k\in[\ccount{i}]}\setminus J_1$. For any $k\in J_{0}$, define $\psi(\varphi(k))=k$. Remark, $\psi$ is now correctly defined on $\acn{(i,k)\colon i\in[\nclients], k\in[\ccount{i}]}\to [\tcount]$. 

\begin{lemma}\label{lem:Pcal-Vk}
    Denote $P^{\mathrm{cal}}(V)$ the distribution of $V(X,Y)$ with $(X,Y)\sim P^{\mathrm{cal}}$.
    The sequence $\acn{V_{k}}_{k\in[\tcount]}$ is a sequence of i.i.d. random variables distributed according to $P^{\mathrm{cal}}(V)$.
\end{lemma}
\begin{proof}
    Let $h:\R^{\tcount}\to\R$ be a continuous and bounded function, we have
    \begin{align*}
        &\E\br{h(V_{1},\ldots,V_{\tcount})}
        = \sum_{i_{1},\ldots,i_{\tcount}\in[\nclients]} \pr{\prod_{k'=1}^{\tcount} \pi_{i_{k'}}} \int h(v_{1},\ldots,v_{\tcount}) \prod_{k=1}^{\tcount} \rmd P^{i_{k}}(v_{k}) \\
        &= \sum_{i_{1},\ldots,i_{\tcount-1}\in[\nclients]} \pr{\prod_{k'=1}^{\tcount-1} \pi_{i_{k'}}} \int h(v_{1},\ldots,v_{\tcount}) \prod_{k=1}^{\tcount-1} \rmd P^{i_{k}}(v_{k}) \pr{\sum_{i=1}^{\nclients}\pi_{i} \rmd P^{i}}(v_{\tcount}) \\
        &= \sum_{i_{1},\ldots,i_{\tcount-1}\in[\nclients]} \pr{\prod_{k'=1}^{\tcount-1} \pi_{i_{k'}}} \int h(v_{1},\ldots,v_{\tcount}) \prod_{k=1}^{\tcount-1} \rmd P^{i_{k}}(v_{k}) \rmd P^{\mathrm{cal}}(v_{\tcount}) \\
        &=\cdots= \int h(v_{1},\ldots,v_{\tcount}) \prod_{k=1}^{\tcount} \rmd P^{\mathrm{cal}}(v_{k}).
    \end{align*}
    This last line concludes the proof.
\end{proof}

In the following, we consider general likelihood ratios $\{\iweighthatN{y}\}_{y\in\YC}$ and recall that $\iweightpN{y}=P_{Y}^{\target}(y)/P_{Y}^{\mathrm{cal}}(y)$.
In addition, let $\Bar{\tcount}\in [\tcount]$, denote for any $i\in[\nclients], k\in[\ccount{i}\wedge \Bar{\tcount}_{\Bar{\tcount}}^{i}], k'\in[\Bar{\tcount}]$, $V_{k'}=V(X_{k'},Y_{k'})$ and define
\begin{equation}\label{eq:def:qweighthat}
    \begin{aligned}
        &\Bar{D}_{\widehat{Y}_{\ccount{\target}+1}^{\target}} = \acn{\widehat{Y}_{\ccount{\target}+1}^{\target}} \cup \acn{Y_k \colon k\in[\Bar{\tcount}]},&
        &\widehat{D}_{\widehat{Y}_{\ccount{\target}+1}^{\target}} = \acn{\widehat{Y}_{\ccount{\target}+1}^{\target}} \cup \acn{Y_k \colon i\in[\nclients], k \in[\ccount{i}\wedge \Bar{\tcount}_{\Bar{\tcount}}^{i}]} \\
        &\qweightwhatN{Y_{k'}}{\widehat{Y}_{\ccount{\target}+1}^{\target}}=\frac{\iweighthatN{Y_{k'}}}{\iweighthatN{\widehat{Y}_{\ccount{\target}+1}^{\target}} + \sum_{l=1}^{\Bar{\tcount}} \iweighthatN{Y_{l}}},&
        &\qweighthatN{Y_k^i}{\widehat{Y}_{\ccount{\target}+1}^{\target}}=\frac{\iweighthatN{Y_{k}^{i}}}{\iweighthatN{\widehat{Y}_{\ccount{\target}+1}^{\target}} + \sum_{j=1}^{\nclients}\sum_{l=1}^{\ccount{j}\wedge \Bar{\tcount}_{\Bar{\tcount}}^{j}}\iweighthatN{Y_{l}^{j}}}.
    \end{aligned}
\end{equation}
Moreover, consider $\qweightwhatN{\widehat{Y}_{\ccount{\target}+1}^{\target}}{\widehat{Y}_{\ccount{\target}+1}^{\target}}=\iweighthatN{\widehat{Y}_{\ccount{\target}+1}^{\target}}(\iweighthatN{\widehat{Y}_{\ccount{\target}+1}^{\target}} + \sum_{l=1}^{\Bar{\tcount}} \iweighthatN{Y_{l}})^{-1}$ and $\qweighthatN{\widehat{Y}_{\ccount{\target}+1}^{\target}}{\widehat{Y}_{\ccount{\target}+1}^{\target}}=\iweighthatN{\widehat{Y}_{\ccount{\target}+1}^{\target}}(\iweighthatN{\widehat{Y}_{\ccount{\target}+1}^{\target}} + \sum_{j=1}^{\nclients}\sum_{l=1}^{\ccount{j}\wedge \Bar{\tcount}_{\Bar{\tcount}}^{j}}\iweighthatN{Y_{l}^{j}})^{-1}$.
Lastly, define
\begin{equation}\label{eq:def:X-delta}
    \begin{aligned}
        &X = \sum_{k=1}^{\Bar{\tcount}} \qweightwhatN{Y_k}{\widehat{Y}_{\ccount{\target}+1}^{\target}} \1_{V_{k}<\widehat{V}_{\tcount+1}},\\
        &\delta = \sum_{i=1}^{\nclients}\sum_{k=1}^{\ccount{i}\wedge \Bar{\tcount}_{\Bar{\tcount}}^{i}} \qweighthatN{Y_k^i}{\widehat{Y}_{\ccount{\target}+1}^{\target}} \1_{V_{k}^{i}<\widehat{V}_{\tcount+1}} - \sum_{k=1}^{\Bar{\tcount}} \qweightwhatN{Y_k}{\widehat{Y}_{\ccount{\target}+1}^{\target}} \1_{V_{k}<\widehat{V}_{\tcount+1}}.
    \end{aligned}
\end{equation} 

\begin{lemma}\label{lem:probVN-infQuantile}
    Assume there are no ties between $\{V_1^i\colon i\in[\tcount]\} \cup \{1\}$ almost surely.
    For any $\acn{\pi_{i}}_{i\in[\nclients]}\in\Delta_{\nclients}$, it holds
    \begin{multline*}
        \textstyle
        - \epsilon - \prob\pr{\delta \le -\epsilon}
        \le \prob\pr{\widehat{V}_{\tcount+1}\le \q{1 - \alpha}\pr{\textstyle\sum_{i=1}^{\nclients} \sum_{k=1}^{\ccount{i}\wedge \Bar{\tcount}_{\Bar{\tcount}}^{i}} \qweighthatN{Y_k^i}{\widehat{Y}_{\ccount{\target}+1}^{\target}} \delta_{V_{k}^{i}} + \qweighthatN{\widehat{Y}_{\ccount{\target}+1}^{\target}}{\widehat{Y}_{\ccount{\target}+1}^{\target}} \delta_{1}}}
        - 1 + \alpha
        \\
        \le \E\br{\max_{k=1}^{\Bar{\tcount}+1}\acn{\qweightwhatN{Y_k}{\widehat{Y}_{\ccount{\target}+1}^{\target}}}} + \epsilon + \prob\pr{\delta \ge \epsilon},
    \end{multline*}
    where $\qweightwhatN{Y_k}{\widehat{Y}_{\ccount{\target}+1}^{\target}}$ and $ \qweighthatN{Y_k^i}{\widehat{Y}_{\ccount{\target}+1}^{\target}}$ are defined in \eqref{eq:def:qweighthat}.
\end{lemma}
\begin{proof}
    By the definition of the quantile combined with \eqref{eq:def:X-delta}, we get
    \begin{equation*}
        \textstyle
        \pr{\widehat{V}_{\tcount+1}< \q{1 - \alpha}\pr{\textstyle\sum_{i=1}^{\nclients} \sum_{k=1}^{\ccount{i}\wedge \Bar{\tcount}_{\Bar{\tcount}}^{i}} \qweighthatN{Y_k^i}{\widehat{Y}_{\ccount{\target}+1}^{\target}} \delta_{V_{k}^{i}} + \qweighthatN{\widehat{Y}_{\ccount{\target}+1}^{\target}}{\widehat{Y}_{\ccount{\target}+1}^{\target}} \delta_{1}}}
        \iff \pr{X+\delta > \alpha}.
    \end{equation*}
    Since there are no ties between $\{V_1^i\colon i\in[\tcount]\} \cup \{1\}$, it implies that $\widehat{V}_{\tcount+1}$ does not belong to $\{V_k^i\colon i\in[\nclients], k\in[\ccount{i}\wedge \Bar{\tcount}_{\Bar{\tcount}}^{i}]\}_{}\cup\{V_k\colon k\in[\Bar{\tcount}]\}\cup\{1\}$ almost surely.
    Therefore, it holds that
    \begin{equation}\label{eq:probVN-indic}
        \textstyle
        \prob\pr{\widehat{V}_{\tcount+1}\le \q{1 - \alpha}\pr{\textstyle\sum_{i=1}^{\nclients} \sum_{k=1}^{\ccount{i}\wedge \Bar{\tcount}_{\Bar{\tcount}}^{i}} \qweighthatN{Y_k^i}{\widehat{Y}_{\ccount{\target}+1}^{\target}} \delta_{V_{k}^{i}} + \qweighthatN{\widehat{Y}_{\ccount{\target}+1}^{\target}}{\widehat{Y}_{\ccount{\target}+1}^{\target}} \delta_{1}}}
        = \E\br{\1_{X>\alpha}}
        + \E\br{\1_{X+\delta>\alpha} - \1_{X>\alpha}}.
    \end{equation}
    Remark that
    \begin{equation}\label{eq:bounds:indic-X-delta}
        \1_{X > \alpha+\epsilon} - \1_{X> \alpha} - \1_{\delta\le -\epsilon}
        \le \1_{X+\delta>\alpha} - \1_{X>\alpha}
        \le \1_{X > \alpha-\epsilon} - \1_{X> \alpha} + \1_{\delta\ge \epsilon}.
    \end{equation}
    Thus, combining \eqref{eq:bounds:indic-X-delta} with \eqref{eq:probVN-indic} gives
    \begin{multline}\label{eq:bounds:probX-alpha-eps}
        \textstyle
        \prob\pr{X>\alpha+\epsilon} - \prob\pr{\delta \le -\epsilon} \\
        \le \prob\pr{\widehat{V}_{\tcount+1}\le \q{1 - \alpha}\pr{\textstyle\sum_{i=1}^{\nclients} \sum_{k=1}^{\ccount{i}\wedge \Bar{\tcount}_{\Bar{\tcount}}^{i}} \qweighthatN{Y_k^i}{\widehat{Y}_{\ccount{\target}+1}^{\target}} \delta_{V_{k}^{i}} + \qweighthatN{\widehat{Y}_{\ccount{\target}+1}^{\target}}{\widehat{Y}_{\ccount{\target}+1}^{\target}} \delta_{1}}} \\
        \le \prob\pr{X>\alpha-\epsilon} + \prob\pr{\delta \ge \epsilon}.
    \end{multline}
    Consider $\tilde{\alpha}\in (0,1)$,
    \begin{equation*}
        (X> \tilde{\alpha})
        \iff \pr{\widehat{V}_{\tcount+1}< \q{1 - \tilde{\alpha}}\pr{\textstyle \qweightwhatN{\widehat{Y}_{\ccount{\target}+1}^{\target}}{\widehat{Y}_{\ccount{\target}+1}^{\target}} \delta_{1} + \sum_{k=1}^{\Bar{\tcount}} \qweightwhatN{Y_k}{\widehat{Y}_{\ccount{\target}+1}^{\target}} \delta_{V_{k}}}}.
    \end{equation*}
    Hence, $\prob\prn{X> \tilde{\alpha}}=\prob\prn{\widehat{V}_{\tcount+1}\le \q{1 - \tilde{\alpha}}\prn{\textstyle \qweightwhatN{\widehat{Y}_{\ccount{\target}+1}^{\target}}{\widehat{Y}_{\ccount{\target}+1}^{\target}} \delta_{1} + \sum_{k=1}^{\Bar{\tcount}} \qweightwhatN{Y_k}{\widehat{Y}_{\ccount{\target}+1}^{\target}} \delta_{V_{k}}}}$.
    Applying \Cref{lem:YinCN} gives that
    \begin{equation*}
        0\le \prob\pr{\widehat{V}_{\tcount+1}\le \q{1 - \tilde{\alpha}}\pr{\textstyle \qweightwhatN{\widehat{Y}_{\ccount{\target}+1}^{\target}}{\widehat{Y}_{\ccount{\target}+1}^{\target}} \delta_{1} + \sum_{k=1}^{\Bar{\tcount}} \qweightwhatN{Y_k}{\widehat{Y}_{\ccount{\target}+1}^{\target}} \delta_{V_{k}}}} - 1 + \tilde{\alpha}
        \le \E\br{\max_{k=1}^{\Bar{\tcount}+1}\acn{\qweightwhatN{Y_k}{\widehat{Y}_{\ccount{\target}+1}^{\target}}}}.
    \end{equation*}
    Therefore, plugging the previous inequality into \eqref{eq:bounds:probX-alpha-eps} concludes the proof.
\end{proof}

Recall that $(\iweighthatN{Y_1}-\E\iweighthatN{Y_1})$ is $\sigma$-sub Gaussian if for any $s\in\R$, the following inequality holds
\begin{equation*}
    \int_{y\in\YC}\exp\pr{s\br{\iweighthatN{y} - \int_{y'\in\YC} \iweighthatN{y'} \rmd P_{Y}^{\mathrm{cal}}(y')}} \rmd P_{Y}^{\mathrm{cal}}(y)
    \le \exp\pr{\frac{\sigma^2 s^2}{2}}.
\end{equation*}

\begin{lemma}\label{lem:bound:delta}
    Assume the random variable $(\iweighthatN{Y_1}-\E\iweighthatN{Y_1})$ is $\sigma$-sub Gaussian with parameter $\sigma\ge 0$.
    For any $\epsilon\ge 8\sigma^2\log \tcount/\prn{\Bar{\tcount} \E\brn{\iweighthatN{Y_{1}}}^{2}}$ and $\acn{\pi_{i}}_{i\in[\nclients]}\in\Delta_{\nclients}$, we have
    \begin{equation*}
        \prob\pr{\abs{\delta}>\epsilon}
        \le \prob\pr{\sum_{i=1}^{\nclients}\pr{\Bar{\tcount}_{\Bar{\tcount}}^i - \ccount{i}}_{+} > \frac{\tcount \epsilon}{4}}
        + \frac{4\var\pr{\iweighthatN{Y_{1}}}}{\Bar{\tcount} \E\br{\iweighthatN{Y_{1}}}^2}
        + \frac{1}{\tcount},
    \end{equation*}
    where $\delta$ is defined in \eqref{eq:def:X-delta}.
\end{lemma}
\begin{proof}
    First, denote $\mathcal{J}_{0} = J_{0}\cap [\Bar{\tcount}]$ and remark that $\card{\mathcal{J}_{0}}=\sum_{i=1}^\nclients (\Bar{\tcount}_{\Bar{\tcount}}^{i} - \ccount{i})_+$. Moreover, recall that $\delta$ is defined by
    \begin{equation*}
        \delta = \sum_{i=1}^{\nclients} \sum_{k=1}^{\ccount{i} \wedge \Bar{\tcount}_{\Bar{\tcount}}^{i}} \pr{\qweighthatN{Y_k^i}{\widehat{Y}_{\ccount{\target}+1}^{\target}} - \qweightwhatN{Y_{\psi(i,k)}}{\widehat{Y}_{\ccount{\target}+1}^{\target}}} \1_{V_{k}^{i}<V_{\tcount+1}}
        - \sum_{k\in \mathcal{J}_{0}} \qweightwhatN{Y_k}{\widehat{Y}_{\ccount{\target}+1}^{\target}} \1_{V_{k}<\widehat{V}_{\tcount+1}}.
    \end{equation*}
    Using definition of the weighs $\qweightwhatN{}{}$ given in \eqref{eq:def:qweighthat} and the definition of $\psi$ provide at the beginning of this section, note that
    \begin{equation}\label{eq:boundJ0J1}
        \sum_{k\in \mathcal{J}_{0}} \qweightwhatN{Y_{\psi(i,k)}}{\widehat{Y}_{\ccount{\target}+1}^{\target}} \1_{V_{k}<\widehat{V}_{\tcount+1}}
        \le \frac{\sum_{k\in \mathcal{J}_{0}} \iweighthatN{Y_{k}}}{\iweighthatN{\widehat{Y}_{\ccount{\target}+1}^{\target}} + \sum_{l=1}^{\Bar{\tcount}} \iweighthatN{Y_{l}}}.
    \end{equation}
    From the definition of $\qweighthatN{Y_k^i}{\widehat{Y}_{\ccount{\target}+1}^{\target}}$ and $\qweightwhatN{Y_{\psi(i,k)}}{\widehat{Y}_{\ccount{\target}+1}^{\target}}$ provided in \eqref{eq:def:qweighthat}, we obtain
    \begin{multline*}
        \sum_{i=1}^{\nclients} \sum_{k=1}^{\ccount{i} \wedge \Bar{\tcount}_{\tcount}^{i}} \pr{\qweighthatN{Y_k^i}{\widehat{Y}_{\ccount{\target}+1}^{\target}} - \qweightwhatN{Y_{\psi(i,k)}}{\widehat{Y}_{\ccount{\target}+1}^{\target}}} \1_{V_{k}^{i}<V_{\tcount+1}}
        = \sum_{i=1}^{\nclients} \sum_{k=1}^{\ccount{i} \wedge \Bar{\tcount}_{\tcount}^{i}} \pr{\frac{\iweighthatN{Y_{k}^{i}} \1_{V_{k}^{i}<V_{\tcount+1}}}{\iweighthatN{\widehat{Y}_{\ccount{\target}+1}^{\target}} + \sum_{j=1}^{\nclients}\sum_{l\in J_{i}}\iweighthatN{Y_{l}^{j}}} - \frac{\iweighthatN{Y_{k}^{i}} \1_{V_{k}^{i}<V_{\tcount+1}}}{\iweighthatN{\widehat{Y}_{\ccount{\target}+1}^{\target}} + \sum_{l=1}^{\Bar{\tcount}} \iweighthatN{Y_{l}}}},
        \\
        = \pr{\frac{1}{\sum_{l\in[\Bar{\tcount}+1]\setminus \mathcal{J}_{0}} \iweighthatN{Y_{l}}} - \frac{1}{\sum_{l\in[\Bar{\tcount}+1]\setminus \mathcal{J}_{0}} \iweighthatN{Y_{l}} + \sum_{l\in \mathcal{J}_{0}} \iweighthatN{Y_{l}}}} \sum_{k\in[\Bar{\tcount}]\setminus \mathcal{J}_{0}} \iweighthatN{Y_{k}} \1_{V_{k}^{i}<V_{\tcount+1}}.
    \end{multline*}
    Therefore, we know that
    \begin{align*}
        \abs{\sum_{i=1}^{\nclients} \sum_{k=1}^{\ccount{i} \wedge \Bar{\tcount}_{\tcount}^{i}} \pr{\qweighthatN{Y_k^i}{\widehat{Y}_{\ccount{\target}+1}^{\target}} - \qweightwhatN{Y_{\psi(i,k)}}{\widehat{Y}_{\ccount{\target}+1}^{\target}}} \1_{V_{k}^{i}<V_{\tcount+1}}}
        \le \frac{\sum_{k\in \mathcal{J}_{0}} \iweighthatN{Y_{k}}}{\iweighthatN{\widehat{Y}_{\ccount{\target}+1}^{\target}} + \sum_{l=1}^{\Bar{\tcount}} \iweighthatN{Y_{l}}}.
    \end{align*}
    Hence, plugging the previous line into \eqref{eq:def:X-delta} combined with \eqref{eq:boundJ0J1} gives
    \begin{equation}\label{eq:bound:absdelta}
        \abs{\delta}
        \le \frac{2\sum_{k\in \mathcal{J}_{0}} \iweighthatN{Y_{k}}}{\iweighthatN{\widehat{Y}_{\ccount{\target}+1}^{\target}} + \sum_{l=1}^{\Bar{\tcount}} \iweighthatN{Y_{l}}}.
    \end{equation}
    Moreover, define the following event:
    \begin{equation*}\label{eq:def:EN-event}
        E_{\tcount}
        = \ac{\sum_{i=1}^{\nclients} \pr{\Bar{\tcount}_{\Bar{\tcount}}^{i} - \ccount{i}}_{+} \le \frac{\tcount \epsilon}{4}}.
    \end{equation*}
    Next, using the event $E_{\tcount}$, we decompose the following probability
    \begin{multline}\label{eq:bound:probJ0:0}
        \prob\pr{\frac{\sum_{k\in \mathcal{J}_{0}} \iweighthatN{Y_{k}}}{\iweighthatN{\widehat{Y}_{\ccount{\target}+1}^{\target}} + \sum_{l=1}^{\Bar{\tcount}} \iweighthatN{Y_{l}}} \ge \frac{\epsilon}{2}; E_{\tcount}}
        \le \prob\pr{\sum_{k\in \mathcal{J}_{0}} \iweighthatN{Y_{k}}\ge \frac{\epsilon \Bar{\tcount} \E\br{\iweighthatN{Y_{1}}}}{4}; E_{\tcount}} \\
        + \prob\pr{{\iweighthatN{\widehat{Y}_{\ccount{\target}+1}^{\target}} + \sum_{l=1}^{\Bar{\tcount}} \iweighthatN{Y_{l}}} < \frac{\Bar{\tcount} \E\br{\iweighthatN{Y_{1}}}}{2}}.
    \end{multline}
    Since the $\acn{\iweighthatN{Y_{k}}}_{k\in [\tcount]}$ are $\sigma$-sub Gaussian, the first term of the previous right-hand side inequality is upper bounded thanks to Hoeffding's inequality
    \begin{multline}\label{eq:bound:probJ0:1}
        \prob\pr{\sum_{k\in \mathcal{J}_{0}} \iweighthatN{Y_{k}}\ge \frac{\epsilon \Bar{\tcount} \E\br{\iweighthatN{Y_{1}}}}{4}; E_{\tcount}}
        = \E\br{\prob\pr{\sum_{k\in \mathcal{J}_{0}} \iweighthatN{Y_{k}}\ge \frac{\epsilon \Bar{\tcount} \E\br{\iweighthatN{Y_{1}}}}{4}\,\Big\vert\, \card{\mathcal{J}_{0}}} \1_{E_{\tcount}}} \\
        \le \exp\pr{- \frac{(\epsilon\Bar{\tcount}/4)^2 \E\brn{\iweighthatN{Y_{1}}}^{2}}{2\card{\mathcal{J}_{0}} \sigma^2}}
        \le \frac{1}{\tcount},
    \end{multline}
    where the last inequality holds by setting $\epsilon \ge 8\sigma^2\log \tcount/\prn{\Bar{\tcount} \E\brn{\iweighthatN{Y_{1}}}^{2}}$.
    Moreover, since $\iweighthatN{Y_{1}}\ge 0$ almost surely, from the Chebyshev inequality we deduce that
    \begin{align}\label{eq:bound:probJ0:2}
        \prob\pr{{\iweighthatN{\widehat{Y}_{\ccount{\target}+1}^{\target}} + \sum_{l=1}^{\Bar{\tcount}} \iweighthatN{Y_{l}}} < \frac{\Bar{\tcount} \E\br{\iweighthatN{Y_{1}}}}{2}}
        \le \frac{4\var\pr{\iweighthatN{Y_{1}}}}{\Bar{\tcount} \E\br{\iweighthatN{Y_{1}}}^2}.
    \end{align}
    Therefore, combining \eqref{eq:bound:probJ0:0}, \eqref{eq:bound:probJ0:1} and \eqref{eq:bound:probJ0:2} shows
    \begin{equation}\label{eq:bound:probJ0:3}
        \prob\pr{\frac{\sum_{k\in \mathcal{J}_{0}} \iweighthatN{Y_{k}}}{\iweighthatN{\widehat{Y}_{\ccount{\target}+1}^{\target}} + \sum_{l=1}^{\Bar{\tcount}} \iweighthatN{Y_{l}}} \ge \frac{\epsilon}{4}; E_{\tcount}}
        \le \frac{1}{\tcount} + \frac{4\var\pr{\iweighthatN{Y_{1}}}}{\Bar{\tcount} \E\br{\iweighthatN{Y_{1}}}^2}.
    \end{equation}
    Lastly, using \eqref{eq:bound:absdelta} we derive the next inequality: 
    \begin{equation*}
        \prob\pr{\abs{\delta}>\epsilon}
        \le 1 - \prob\pr{E_{\tcount}}
        + \prob\pr{\frac{\sum_{k\in \mathcal{J}_{0}} \iweighthatN{Y_{k}}}{\iweighthatN{\widehat{Y}_{\ccount{\target}+1}^{\target}} + \sum_{l=1}^{\Bar{\tcount}} \iweighthatN{Y_{l}}} \ge \frac{\epsilon}{2}; E_{\tcount}}.
    \end{equation*}
    Plugging \eqref{eq:bound:probJ0:3} into the previous line completes the proof.
\end{proof}

\begin{lemma}\label{lem:prob:outside}
    For any $i\in[\nclients]$, consider $\pi_{i} = \ccount{i}/\tcount$ and $2\Bar{\tcount}\le \tcount$. We have
    \begin{equation}\label{eq:prob:outside}
        \prob\pr{\sum_{i=1}^{\nclients}\pr{\Bar{\tcount}_{\Bar{\tcount}}^i - \ccount{i}}_{+} > \frac{7}{4} \log(\nclients\tcount) \sum_{j \colon \frac{\ccount{j}}{6} < \log (\nclients \tcount)} \sqrt{\ccount{j}}}
        \le \frac{1}{\tcount}.
    \end{equation}    
\end{lemma}
\begin{proof}
    First, define $\epsilon = 7 \brn{\log(\nclients\tcount) / \tcount} \sum_{j\in A} \sqrt{\ccount{j}}$ where we consider the following set
    \begin{equation*}
        A = \ac{i\in[\nclients]\colon \ccount{i} < 6 \log (\nclients \tcount)}.
    \end{equation*}
    If $A\neq \emptyset$, for all $i\in A$ take
    \begin{equation*}
        \alpha_{i} = \frac{\sqrt{\pi_{i}}}{\sum_{j\in A} \sqrt{\pi_j}}.
    \end{equation*}
    Using the union bound, we get
    \begin{equation}\label{eq:bound:probmultinomial:1}
        \prob\pr{\sum_{i=1}^{\nclients}\pr{\Bar{\tcount}_{\Bar{\tcount}}^i - \ccount{i}}_{+} > \frac{\tcount \epsilon}{4}}
        \le \sum_{i\in [\nclients]\setminus A} \prob\pr{\Bar{\tcount}_{\Bar{\tcount}}^i \ge \ccount{i}}
        + \sum_{i\in A} \prob\pr{\Bar{\tcount}_{\Bar{\tcount}}^i \ge \ccount{i} + \frac{\alpha_{i} \tcount \epsilon}{4}}.
    \end{equation}
    If $i\in[\nclients]\setminus A$, then applying the Chernoff bound gives
    \begin{equation}\label{eq:bound:probmultinomial:2}
        \prob\pr{\Bar{\tcount}_{\Bar{\tcount}}^i \ge \ccount{i}}
        \le \exp\pr{-\frac{\pi_{i} (\tcount - \Bar{\tcount})}{1 + 2 \Bar{\tcount} / (\tcount - \Bar{\tcount})}}
        \le \exp\pr{-\frac{\pi_{i} \tcount}{6}} \le \frac{1}{\nclients \tcount}.
    \end{equation}
    If $i\in A$, then applying the Bernstein inequality gives
    \begin{equation}\label{eq:bound:probmultinomial:Bernstein}
        \prob\pr{\Bar{\tcount}_{\Bar{\tcount}}^i \ge \ccount{i} + \frac{\alpha_{i} \tcount \epsilon}{4}}
        \le \exp\pr{- \frac{\prn{\alpha_{i} \tcount \epsilon/4}^2}{2 \Bar{\tcount} \pi_{i} (1-\pi_{i}) + \alpha_{i} \tcount \epsilon / 6}}.
    \end{equation}
    Moreover, we can suppose that $\nclients \ge 2$ otherwise \eqref{eq:prob:outside} immediately holds. Thus, $\nclients\tcount\ge 4$ and using the fact that $2\Bar{\tcount}\le \tcount$ we deduce that
    \begin{equation*}
        4 \pr{\frac{2}{3} + \sqrt{\frac{2\Bar{\tcount}}{\tcount \log(\nclients\tcount)}}}
        \le 7.
    \end{equation*}
    Therefore, it follows
    \begin{equation*}
        \tcount\epsilon
        \ge 4 \log(\nclients\tcount) \pr{\frac{2}{3} + \sqrt{\frac{2\Bar{\tcount}}{\tcount \log(\nclients\tcount)}}} \sum_{j\in A} \sqrt{\ccount{j}}.
    \end{equation*}
    The last inequality is enough to ensure that
    \begin{equation}\label{eq:bound:probmultinomial:lognN}
        \frac{\prn{\tcount \epsilon}^2 / 8}{4 (\Bar{\tcount}/\tcount) \prn{\sum_{j\in A}\sqrt{\ccount{j}}}^2 + \tcount \epsilon \sum_{j\in A} \sqrt{\ccount{j}} / 3}
        \ge \log(\nclients \tcount).
    \end{equation}
        Moreover, using the definition of $\alpha_{i}$ implies
    \begin{align*}
        \frac{\prn{\alpha_{i} \tcount \epsilon/4}^2}{2 \Bar{\tcount} \pi_{i} (1-\pi_{i}) + \alpha_{i} \tcount \epsilon / 6}
        \ge \frac{\prn{\tcount \epsilon}^2 / 8}{4 (\Bar{\tcount}/\tcount) \prn{\sum_{j\in A}\sqrt{\ccount{j}}}^2 + \tcount \epsilon \sum_{j\in A} \sqrt{\ccount{j}} / 3}.
    \end{align*}
    Hence, combining \eqref{eq:bound:probmultinomial:Bernstein} with \eqref{eq:bound:probmultinomial:lognN} shows that
    \begin{equation}\label{eq:bound:probmultinomial:3}
        \prob\pr{\Bar{\tcount}_{\Bar{\tcount}}^i \ge \ccount{i} + \frac{\alpha_{i} \tcount \epsilon}{4}}
        \le \frac{1}{\nclients \tcount}.
    \end{equation}
    Finally, plugging \eqref{eq:bound:probmultinomial:2} and \eqref{eq:bound:probmultinomial:3} into \eqref{eq:bound:probmultinomial:1} yields the result.
\end{proof}

\begin{lemma}\label{lem:Emax-pk}
    Recall that $\acn{\qweightwhatN{Y_{k}}{\widehat{Y}_{\ccount{\target}+1}^{\target}}}_{k\in[\Bar{\tcount}]}$ is defined in~\eqref{eq:def:qweighthat} and for the sake of simplicity define $Y_{\Bar{\tcount}+1}=\widehat{Y}_{\ccount{\target}+1}^{\target}$.
    If $\normn{\iweightN{}}_{\infty}<\infty$, then 
    \[
        \E\br{\max_{k=1}^{\Bar{\tcount}+1}{\qweightwhatN{Y_{k}}{Y_{\Bar{\tcount}+1}}}}
        \le \frac{2 \normn{\iweightN{}}_{\infty}}{\Bar{\tcount} \E\iweightN{Y_1}}
        + \frac{4}{\Bar{\tcount} (\E\iweightN{Y_1})^2} \pr{\var\prn{\iweightN{Y_1}} + \frac{\var\prn{\iweightN{Y_{\Bar{\tcount}+1}}}}{\Bar{\tcount}}}.
    \]
    Moreover, if the random variables $(\iweightN{Y_k}-\E \iweightN{Y_k})_{k\in[\Bar{\tcount}+1]}$ are $\sigma$-subGaussian with parameter $\sigma\ge 0$, then
    \begin{equation*}
        \E\br{\max_{k=1}^{\Bar{\tcount}+1}{\qweightwhatN{Y_{k}}{Y_{\Bar{\tcount}+1}}}}
        \le \frac{\sigma \sqrt{8 \log (\Bar{\tcount}+1)}}{\Bar{\tcount} \E\iweightN{Y_1}}
        + \frac{2}{\Bar{\tcount}} \pr{1 \vee \frac{\E\iweightN{Y_{\Bar{\tcount}+1}}}{\E\iweightN{Y_1}}
        + \frac{2 \var\prn{\iweightN{Y_1}}}{(\E\iweightN{Y_1})^2}
        + \frac{2 \var\prn{\iweightN{Y_{\Bar{\tcount}+1}}}}{\Bar{\tcount} (\E\iweightN{Y_1})^2}}
        .
    \end{equation*}
\end{lemma}
\begin{proof}
    By definition of the probabilities $\acn{\qweightwhatN{Y_{k}}{Y_{\Bar{\tcount}+1}}}_{k\in[\Bar{\tcount}+1]}$, we have
    \begin{equation*}
        \E\br{\max_{k=1}^{\Bar{\tcount}+1} \acn{\qweightwhatN{Y_{k}}{Y_{\Bar{\tcount}+1}}}}
        = \E\br{\frac{\max_{k=1}^{\Bar{\tcount}+1}\ac{\iweightN{Y_k}}}{\sum_{l=1}^{\Bar{\tcount}+1}\iweightN{Y_l}}}.
    \end{equation*}
    The expectation is split by introducing $\mathcal{A}_{\tcount+1}=\acn{\sum_{l=1}^{\Bar{\tcount}+1}\iweightN{Y_l} \le \prn{\Bar{\tcount} \E\iweightN{Y_1}+\E\iweightN{Y_{\Bar{\tcount}+1}}}/2}$, we obtain
    \begin{equation}\label{eq:lem-proof:Emax}
        \E\br{\frac{\max_{k=1}^{\Bar{\tcount}+1}\ac{\iweightN{Y_k}}}{\sum_{l=1}^{\Bar{\tcount}+1}\iweightN{Y_l}}}
        \le \frac{2}{\Bar{\tcount} \E\iweightN{Y_1}}\E\br{\1_{\mathcal{A}_{\tcount+1}^{c}} \cdot \max_{k=1}^{\Bar{\tcount}+1}\ac{\iweightN{Y_k}}} + \prob\pr{\mathcal{A}_{\tcount+1}}.
    \end{equation}
    Using the Chebyshev's inequality it follows that
    \begin{multline}\label{eq:lem-proof:probAN}
        \prob\pr{\mathcal{A}_{\tcount+1}}
        = \prob\pr{2\sum_{k=1}^{\Bar{\tcount}+1} \iweightN{Y_k} < \Bar{\tcount} \E\iweightN{Y_1}+\E\iweightN{Y_{\Bar{\tcount}+1}}} \\
        \le \frac{4}{\Bar{\tcount} (\E\iweightN{Y_1}+\E\iweightN{Y_{\Bar{\tcount}+1}}/\Bar{\tcount})^2} \pr{\var\prn{\iweightN{Y_1}} + \frac{\var\prn{\iweightN{Y_{\Bar{\tcount}+1}}}}{\Bar{\tcount}}}.
    \end{multline}
    When $\iweightN{}$ is bounded, combining \eqref{eq:lem-proof:Emax} with \eqref{eq:lem-proof:probAN} gives
    \begin{align*}
        \E\br{\max_{k=1}^{\Bar{\tcount}+1}\acn{\qweightwhatN{Y_{k}}{\widehat{Y}_{\ccount{\target}+1}^{\target}}}}
        &\le \frac{2 \normn{\iweightN{}}_{\infty}}{\Bar{\tcount} \E\iweightN{Y_1}} + \prob\pr{\mathcal{A}_{\tcount+1}} \\
        &\le \frac{2 \normn{\iweightN{}}_{\infty}}{\Bar{\tcount} \E\iweightN{Y_1}}
            + \frac{4}{\Bar{\tcount} (\E\iweightN{Y_1})^2} \pr{\var\prn{\iweightN{Y_1}} + \frac{\var\prn{\iweightN{Y_{\Bar{\tcount}+1}}}}{\Bar{\tcount}}}.
    \end{align*}
    Otherwise, if we suppose that the $(\iweightN{Y_k} - E \iweightN{Y_k})_{k\in[\Bar{\tcount}+1]}$ are $\sigma$-sub-Gaussian, then applying the result given in~\citep[Section~2.5]{boucheron2013concentration} combined with \eqref{eq:lem-proof:Emax}-\eqref{eq:lem-proof:probAN} concludes the proof since it follows that $\E\brn{\max_{k=1}^{\Bar{\tcount}+1}\acn{\iweightN{Y_k}}}\le \max_{k=1}^{\Bar{\tcount}+1}\acn{\E \iweightN{Y_k}} + \sigma\sqrt{2 \log (\Bar{\tcount}+1)}$.
\end{proof}

Finally, for any point $(\xquery,\yquery)\in\XC\times\YC$, define the following measure and prediction set
\begin{align*}
    &\widehat{\mu}_{\yquery} = \textstyle \qweighthatN{\yquery}{\yquery} \delta_{1} + \sum_{i=1}^{\nclients} \sum_{k=1}^{\ccount{i}\wedge \Bar{\tcount}_{\Bar{\tcount}}^i} \qweighthatN{Y_k^i}{\yquery} \delta_{V_{k}^{i}}
    \\
    &\mathcal{C}_{\alpha,\widehat{\mu}}(\xquery) =
    \ac{\yquery\in\YC\colon V(\xquery,\yquery)\le \q{1 - \alpha}\prn{\widehat{\mu}_{\yquery}}}.
\end{align*}

\begin{theorem}\label{thm:general-coverage}
    Assume there are no ties between $\{V_1^i\colon i\in[\tcount]\} \cup \{1\}$ almost surely.
    For any $i\in[\nclients]$, let $\pi_{i} = \ccount{i}/\tcount$ and consider $\Bar{\tcount}=\lfloor\tcount/2\rfloor$.
    If $\normn{\iweighthatN{}}_{\infty}<\infty$, then it holds
    \begin{multline*}
        \abs{\prob\pr{\Ytarget \in \mathcal{C}_{\alpha,\widehat{\mu}}(\Xtarget)}
        - 1 + \alpha}
        \le
        \frac{1}{2} \sum_{y \in \YC} \abs{P_{Y}^{\target}(y) - \frac{\iweighthatN{y} P_{Y}^{\mathrm{cal}}(y)}{\sum_{\tilde{y} \in \YC}\iweighthatN{\tilde{y}}P_{Y}^{\mathrm{cal}}\prn{\tilde{y}}}}
        \\
        + \frac{6}{\tcount}
        + \frac{36 \normn{\iweightN{}}_{\infty}^2}{\tcount (\E\iweightN{Y_1})^2}
        + \frac{2\log \tcount}{\tcount} \pr{\frac{3\normn{\iweightN{}}_{\infty}^2}{\prn{\E\iweighthatN{Y_{1}}}^{2}} \vee 7 \textstyle\sum_{j \colon \frac{\ccount{j}}{12} < \log \tcount} \sqrt{\ccount{j}}}
        .
    \end{multline*}
\end{theorem}
\begin{proof}
    Using \Cref{lem:Prob-diff:YwidehatY} implies that
    \begin{equation}\label{eq:bound:abs-TV}
        \abs{\prob\pr{\Ytarget\in \mathcal{C}_{\alpha,\widehat{\mu}}(\Xtarget)} - \prob\pr{\widehat{Y}_{\ccount{\target}+1}^{\target}\in \mathcal{C}_{\alpha,\widehat{\mu}}(\widehat{X}_{\ccount{\target}+1}^{\target})}} \\
        \le \frac{1}{2} \sum_{y \in \YC} P_{Y}^{\mathrm{cal}}\prn{y} \abs{\iweightpN{y} - \frac{\iweighthatN{y}}{\sum_{\tilde{y} \in \YC}\iweighthatN{\tilde{y}}P_{Y}^{\mathrm{cal}}\prn{\tilde{y}}}}.
    \end{equation}
    Since $\acn{\iweighthatN{y}}_{y\in\YC}$ are bounded, we deduce that $\iweighthatN{Y_1}$ is $\sigma$-subGaussian with $\sigma=2^{-1}(\max_{y\in\YC}\acn{\iweighthatN{y}} - \min_{y\in\YC}\acn{\iweighthatN{y}})$.
    Therefore, the inequality derived in \Cref{lem:prob:outside} combined with \Cref{lem:bound:delta} provide an upper bound on $\prob\prn{\absn{\delta}>\epsilon}$ with
    \[
        \epsilon
        = \frac{8\sigma^2\log \tcount}{\Bar{\tcount} \E\brn{\iweighthatN{Y_{1}}}^{2}}
        \vee \frac{7 \log(\nclients\tcount) \textstyle \sum_{j \colon \frac{\ccount{j}}{6} < \log (\nclients \tcount)} \sqrt{\ccount{j}}}{\tcount}.
    \]
    Plugging this result into the bound derived in \Cref{lem:probVN-infQuantile} shows that
    \begin{multline}\label{eq:bound:VN-bound}
        \abs{\prob\pr{\widehat{Y}_{\ccount{\target}+1}^{\target}\in \mathcal{C}_{\alpha,\widehat{\mu}}(\widehat{X}_{\ccount{\target}+1}^{\target})} - 1 + \alpha}
        \le \E\br{\qweightwhatN{\widehat{Y}_{\ccount{\target}+1}^{\target}}{\widehat{Y}_{\ccount{\target}+1}^{\target}}\vee \max_{k=1}^{\Bar{\tcount}}\acn{\qweightwhatN{Y_k}{\widehat{Y}_{\ccount{\target}+1}^{\target}}}}
        + \frac{2}{\tcount}
        + \frac{4\var\pr{\iweighthatN{Y_{1}}}}{\Bar{\tcount} \prn{\E\iweighthatN{Y_{1}}}^2} \\
        + \frac{8 \sigma^2 \log \tcount}{\Bar{\tcount} \E\brn{\iweighthatN{Y_{1}}}^{2}} \vee \frac{7 \log(\nclients\tcount) \textstyle\sum_{j \colon \frac{\ccount{j}}{6} < \log (\nclients \tcount)} \sqrt{\ccount{j}}}{\tcount}.
    \end{multline}
    Moreover, applying \Cref{lem:Emax-pk}, we deduce that
    \begin{align}\label{eq:bound:phatN}
        \E\br{\qweightwhatN{\widehat{Y}_{\ccount{\target}+1}^{\target}}{\widehat{Y}_{\ccount{\target}+1}^{\target}}\vee \max_{k=1}^{\Bar{\tcount}}\acn{\qweightwhatN{Y_k}{\widehat{Y}_{\ccount{\target}+1}^{\target}}}}
        \le
        \frac{2 \normn{\iweightN{}}_{\infty}}{\Bar{\tcount} \E\iweightN{Y_1}}
        + \frac{4}{\Bar{\tcount} (\E\iweightN{Y_1})^2} \pr{\var\prn{\iweightN{Y_1}} + \frac{\var\prn{\iweightN{\widehat{Y}_{\ccount{\target}+1}^{\target}}}}{\Bar{\tcount}}}
        .
    \end{align}
    Therefore, combining \eqref{eq:bound:VN-bound} with \eqref{eq:bound:phatN} implies that
    \begin{multline*}
        \abs{\prob\pr{\widehat{Y}_{\ccount{\target}+1}^{\target}\in \mathcal{C}_{\alpha,\widehat{\mu}}(\widehat{X}_{\ccount{\target}+1}^{\target})} - 1 + \alpha}
        \le
        \frac{2}{\Bar{\tcount}} \pr{
            \frac{\Bar{\tcount}}{\tcount}
            + \frac{\normn{\iweightN{}}_{\infty}}{\E\iweightN{Y_1}}
            + \frac{4\var\pr{\iweighthatN{Y_{1}}}}{\prn{\E\iweighthatN{Y_{1}}}^2}
            + \frac{2 \var\prn{\iweightN{\widehat{Y}_{\ccount{\target}+1}^{\target}}}}{\Bar{\tcount} (\E\iweightN{Y_1})^2}
        }
        \\
        + \frac{8 \sigma^2 \log \tcount}{\Bar{\tcount} \prn{\E\iweighthatN{Y_{1}}}^{2}} \vee \frac{14 \log \tcount \textstyle\sum_{j \colon \frac{\ccount{j}}{6} < \log (\nclients \tcount)} \sqrt{\ccount{j}}}{\tcount}.
    \end{multline*}
    If $\tcount\ge 6$, then remark that $\tcount/\Bar{\tcount}\le 3$. Thus, we deduce that
    \begin{multline}\label{eq:bound:prob-widehatYBartcount:2}
        \abs{\prob\pr{\widehat{Y}_{\ccount{\target}+1}^{\target}\in \mathcal{C}_{\alpha,\widehat{\mu}}(\widehat{X}_{\ccount{\target}+1}^{\target})} - 1 + \alpha}
        \le
        \frac{6}{\tcount} \pr{
            1
            + \frac{\normn{\iweightN{}}_{\infty}}{\E\iweightN{Y_1}}
            + \frac{4\var\pr{\iweighthatN{Y_{1}}}}{\prn{\E\iweighthatN{Y_{1}}}^2}
            + \frac{\var\prn{\iweightN{\widehat{Y}_{\ccount{\target}+1}^{\target}}}}{(\E\iweightN{Y_1})^2}
        }
        \\
        + \frac{24 \sigma^2 \log \tcount}{\tcount \prn{\E\iweighthatN{Y_{1}}}^{2}} \vee \frac{14 \log \tcount \textstyle\sum_{j \colon \frac{\ccount{j}}{12} < \log \tcount} \sqrt{\ccount{j}}}{\tcount}.
    \end{multline}
    Finally, using the next inequality combined with \eqref{eq:bound:abs-TV} and \eqref{eq:bound:prob-widehatYBartcount:2} concludes the proof
    \begin{equation*}
        \prob\prn{\Ytarget\in\mathcal{C}_{\alpha,\widehat{\mu}}\prn{\Xtarget}}
        = \prob\prn{\Ytarget\in\mathcal{C}_{\alpha,\widehat{\mu}}\prn{\Xtarget}}
        - \prob\prn{\widehat{Y}_{\ccount{\target}+1}^{\target}\in\mathcal{C}_{\alpha,\widehat{\mu}}\prn{\widehat{X}_{\ccount{\target}+1}^{\target}}}
        + \prob\prn{\widehat{Y}_{\ccount{\target}+1}^{\target}\in\mathcal{C}_{\alpha,\widehat{\mu}}\prn{\widehat{X}_{\ccount{\target}+1}^{\target}}}.
    \end{equation*}
\end{proof}

The proof of the following result is similar to that of \Cref{thm:general-coverage}.

\begin{theorem}\label{thm:general-coverage-exact-weights}
    For any $i\in[\nclients]$, let $\pi_{i} = \ccount{i}/\tcount$ and set $\Bar{\tcount}=\lfloor\tcount/2\rfloor$.
    Moreover, for $y\in\YC$ consider $\iweighthatN{y}=\iweightp{y}$.
    If $\normn{\iweightp{}}_{\infty}<\infty$, then it holds
    \begin{equation*}
        \abs{\prob\pr{\Ytarget \in \mathcal{C}_{\alpha,\widehat{\mu}}(\Xtarget)}
        - 1 + \alpha}
        \le
        \frac{6}{\tcount} + \frac{36 + 6 \log \tcount}{\tcount} \normn{\iweightp{}}_{\infty}^2
        + \frac{14\log \tcount}{\tcount} \textstyle\sum_{j \colon \frac{\ccount{j}}{12} < \log \tcount} \sqrt{\ccount{j}}.
    \end{equation*}
\end{theorem}
\begin{proof}
    First, for $y\in\YC$ recall that $\iweightp{y}=(P_{Y}^{\target}/P_{Y}^{\mathrm{cal}})\prn{y}$. Thus, we remark that $\E\iweightp{Y_1}=1$.
    Therefore, applying \Cref{thm:general-coverage} concludes the proof.
\end{proof}

\section{Differential privacy guarantee: proof of \Cref{thm:privacy:fedavg}}\label{sec:differential-privacy}

In this section, we recall the definition of being $(\epsilon,\delta)$-DP \citep{dwork2014algorithmic}.
The idea behind differential privacy is to ensure that no attacker can determine with high confidence whether a particular individual's data is included in the dataset or not.
Often, a controlled amount of random noise is added to the data, so that any individual data point becomes indistinguishable from the noise.
This ensures that the probability distribution of an algorithm's output does not change significantly when a single individual's data is added or removed.

\begin{definition}
  For any $\epsilon>0$ and $\delta\in[0,1)$, a randomized mechanism $\mathcal{A}$ is said to be $(\epsilon, \delta)$-DP, if for all neighboring datasets $D$, $D'$, and for any event $E$:
  \begin{equation*}
    \prob\pr{\mathcal{A}(D)\in E}
    \le \exp(\epsilon) \prob\pr{\mathcal{A}(D')\in E} + \delta.
  \end{equation*}
\end{definition}

The following result gives the noise level sufficient to ensure the $(\epsilon,\delta)$-DP regarding the third-party attacker.
This type of attacker is an external entity who does not have access to the private data but can observe the algorithm outputs.
This attacker tries to infer sensitive information about individuals by analyzing the output.
\begin{theorem}
    Assume there is a constant number $S\in[\nclients]$ of sampled agents, i.e., $S_t=S$, for all $t \in [T]$.
    Then, for all $\epsilon > 0$ and $\delta \in (0, 1-(1+\sqrt{\epsilon})(1-S/\nclients)^{T})$, the \Cref{algo:Dp-moreau} is $(\epsilon, \delta)$-DP towards a third party if
    \begin{align*}
      \noiseparamone \ge 2 \sqrt{\frac{K \max_{i\in[\nclients]} \lambda_{\yquery}^i}{\epsilon} \pr{1 + \frac{24 S \sqrt{T} \log(1/\Bar{\delta})}{\epsilon \nclients}}},&
      &\Bar{\delta} = \frac{\nclients}{S} \br{1 - \pr{\frac{1-\delta}{1+\sqrt{\epsilon}}}^{1/T}}.
    \end{align*}
\end{theorem}

\begin{proof}
  The loss function $\nabla S_{\alpha}^{i, \gamma}$ has a sensitivity of $1$. Therefore, for any $\tilde{\alpha}>1$, we know that $\nabla S_{\alpha}^{i, \gamma} + \gauss(0,\noiseparamone^2)$ is $(\tilde{\alpha},\tilde{\alpha}/2\noiseparamone^2)$-RDP \citep[Corollary 3]{mironov2017renyi}.
  By assumption, note that $\Bar{\delta}\in(0,1)$ and for any $t\in[T]$, consider
  \begin{align*}
    &\epsilon_t = \frac{\nclients K \tilde{\alpha} \max_{i\in[\nclients]} \lambda_{\yquery}^i}{2 (t+1) S \noiseparamone^2} - \frac{\log \Bar{\delta}}{\tilde{\alpha} - 1}, \\
    &\tilde{\alpha}
    = 1  + \noiseparamone \sqrt{\frac{2 \sqrt{T} S \log(1/\Bar{\delta})}{\nclients K \max_{i\in[\nclients]} \lambda_{\yquery}^i}}.
  \end{align*}
  After $K$ local iterations, using the RDP composition result, the mechanism becomes $(\tilde{\alpha}, K \tilde{\alpha} / 2 \noiseparamone^2)$-RDP.
  Using the aggregation step on the server, the mechanism is now $(\tilde{\alpha}, \nclients K \tilde{\alpha} \max_{i\in [n]} \lambda_{\yquery}^i / 2 (t + 1) S \noiseparamone^2)$-RDP.
  Based on the RDP to DP conversion, we know that the mechanism is $(\epsilon_t,\Bar{\delta})$-DP.
  Define $f\colon x\in\R \mapsto \log\acn{1+(S/\nclients) \prn{\rme^{x} - 1}}\in\R$. For any $x\in\R$, we have
  \begin{equation*}
    f'(x) = \frac{S}{\nclients + 2 x \pr{\nclients - S}}.
  \end{equation*}
  In addition, since the agents are subsampled, it yields the $(\tilde{\epsilon}_t, \tilde{\delta}_t)$-DP \citep[Theorem 9]{balle2018privacy}, where we denote
  \begin{align*}
    \tilde{\epsilon}_t
    = f\pr{\epsilon_t},&
    &\tilde{\delta}_t = \frac{S \Bar{\delta}}{\nclients}.
  \end{align*}
  Since $f'(x)\le S/\nclients$ on $\R_+$, we deduce that
  \begin{equation*}
    \tilde{\epsilon}_t \le \frac{K \tilde{\alpha} \max_{i\in[\nclients]} \lambda_{\yquery}^i}{2 (t+1) \noiseparamone^2} - \frac{S \log \Bar{\delta}}{\nclients \pr{\tilde{\alpha} - 1}}.
  \end{equation*}
  For all $a,b\in\R$, using that $(a+b)^2\le 2 a^2 + 2 b^2$ gives that
  \begin{equation}\label{eq:bound:sumepsquared}
    \sum_{t=1}^{T} \tilde{\epsilon}_t^2
    \le \frac{2 \pr{K \tilde{\alpha} \max_{i\in[\nclients]} \lambda_{\yquery}^i}^2}{4 \noiseparamone^4} \sum_{t=2}^{T+1}\frac{1}{t^2}
    + \frac{2 T S^2 \log(1/\Bar{\delta})^2}{\nclients^2 (\tilde{\alpha} - 1)^2}.
  \end{equation}
  Plugging the definition of $\tilde{\alpha}$ into \eqref{eq:bound:sumepsquared} combined with $\sum_{t=2}^{T+1}t^{-2}\le 1$ show that
  \begin{multline}\label{eq:bound:sumepsquared:2}
    \sum_{t=1}^{T} \tilde{\epsilon}_t^2
    \le \frac{\pr{K \max_{i\in[\nclients]} \lambda_{\yquery}^i}^2}{\noiseparamone^4} \pr{1 + \frac{3 \noiseparamone^2 \sqrt{T} S \log(1/\Bar{\delta})}{\nclients K \max_{i\in[\nclients]} \lambda_{\yquery}^i}} \\
    = \frac{\pr{K \max_{i\in[\nclients]} \lambda_{\yquery}^i}^2}{\noiseparamone^4} + \frac{3 \sqrt{T} K S \log(1/\Bar{\delta}) \max_{i\in[\nclients]} \lambda_{\yquery}^i}{\nclients \noiseparamone^2}.
  \end{multline}
  By assumption, recall that
  \begin{equation*}
    \noiseparamone
    \ge 2 \sqrt{\frac{K \max_{i\in[\nclients]} \lambda_{\yquery}^i}{\epsilon}} \sqrt{1 + \frac{24 \sqrt{T} S \log(1/\Bar{\delta})}{\epsilon \nclients}}.
  \end{equation*}
  Therefore, we obtain
  \begin{equation*}
    \frac{\noiseparamone^4 \epsilon^2}{16}
    \ge \pr{K \max_{i\in[\nclients]} \lambda_{\yquery}^i}^2 + \frac{3\noiseparamone^2 \sqrt{T} S K \log(1/\Bar{\delta}) \max_{i\in[\nclients]} \lambda_{\yquery}^i}{\nclients}.
  \end{equation*}
  The previous inequality combined with \eqref{eq:bound:sumepsquared:2} implies that
  \begin{equation*}
    \sum_{t=1}^{T} \tilde{\epsilon}_t^2
    \le \epsilon^2 / 16.
  \end{equation*}
  Define the following quantity
  \begin{equation*}
    \tilde{\delta}
    = \frac{1 - \delta}{\prod_{t=1}^{T} (1-\tilde{\delta}_t)} - 1.
  \end{equation*}
  Since $1-(1-\delta)^{1/T}\le \tilde{\delta}_t \le 1 - [(1-\delta)/2]^{1/T}$, we can verify that $\tilde{\delta}\in[0,1]$ and also $\tilde{\delta} = \sqrt{\epsilon}$.
  Using the assumption $\delta \le 1 - (1+\sqrt{\epsilon}) (1-S/\nclients)^{T}$, it yields
  \begin{equation*}
    \rme + \tilde{\delta}^{-1} \sqrt{\sum_{t=1}^{T} \epsilon_t^2}
    \le \rme + 1 \le \rme^2.
  \end{equation*}
  Thus, we have
  \begin{equation*}
    2 \sum_{t=1}^{T} \tilde{\epsilon}_t^2 \log\pr{\rme + \tilde{\delta}^{-1} \sqrt{\sum_{t=1}^{T} \epsilon_t^2}}
    \le \frac{\epsilon^2}{4}.
  \end{equation*}
  Since $(\exp(\tilde{\epsilon}_t) - 1)/(\exp(\tilde{\epsilon}_t)+1)\le \tilde{\epsilon}_t$, we deduce that
  \begin{equation*}
    \sum_{t=1}^{T} \frac{\exp(\tilde{\epsilon}_t) - 1}{\exp(\tilde{\epsilon}_t)+1}
    \le \sum_{t=1}^{T} \tilde{\epsilon}_t^2
    \le \frac{\epsilon^2}{16}.
  \end{equation*}
  Finally, applying \citep[Theorem~3.5]{kairouz2015composition} concludes the proof.
\end{proof}

Note that the local loss function $\lossi \colon q\in \R \mapsto\E_{V \sim \qdistr{y}^{i}}\brn{\pinballmoreau{V}{q}}\in\R_+$ is expressed as the expectation of pinball loss functions. Since the sensitivity of these pinball loss functions is 1, we do not need to clip the gradient. It is sufficient to add additional Gaussian noise $\gauss(0,\noiseparamone)$ to guarantee differential privacy. The value of $\noiseparamone$ is chosen to provide a suitable trade-off between privacy and utility, balancing the need for strong privacy protection with the requirement for useful output.

\section{Additional numerical results}\label{sec:algorithms}

\subsection{Algorithm design}

The objective is to generate valid federated prediction sets for the testset by leveraging the calibration datasets of the agents. In this section, we present four different algorithms that were compared in our experiments. The first two are unweighted algorithms (\unwghtloc\ and \unwghtglo), while the other two are weighted benchmarks: \texttt{Oracle Weights} which uses true weights, and \estwght\ which employs estimated weights. Additionally, we propose the \algopred\ method, available in both basic and differentially private versions. The main differences between these approaches lie in the way the resulting quantile is computed, such as the importance given to the set of non-conformity scores and their corresponding weights. Moreover, the approaches' ability to maintain coverage and privacy guarantees varies.

\paragraph{Unweighted Local.}
  The Unweighted method assigns equal weight to all non-conformity scores, regardless of the agent or label. As a result, the resulting quantile and prediction sets are only influenced by the local scores of the querying agent.
  Specifically, the \unwghtloc\ approach only employs the local calibration dataset of Agent $\target$ to compute the corresponding prediction set. This approach is easily computable because it does not involve any exchange of information between agents. Formally, the confidence set can be expressed as follows:
  \begin{align*}
    &\bar{\mu}^{\operatorname{loc},\target} = \frac{1}{N^{\target} + 1}\sum_{k=1}^{N^{\target}}\delta_{V^{\target}_k} + \frac{1}{N^{\target} + 1} \delta_{1},
    \\
    &\unweightlocpredset{\alpha}(\xquery) = \ac{\yquery\in\YC\colon V(\xquery,\yquery)\le \q{1 - \alpha}\pr{\bar{\mu}^{\operatorname{loc},\target}}}.
  \end{align*}

\paragraph{Unweighted Global.}
  The \unwghtglo\ approach calculates the quantile by using all non-conformity scores gathered on the central server from all agents. This approach violates FL constraints since all non-conformity scores are shared with the central server. Furthermore, each collected score is given equal weight, without taking into account any label shift. The corresponding prediction set is represented as follows:
  \begin{align*}
    &\Bar{\mu} = \frac{1}{N + 1}\sum_{i = 1}^{\nclients}\sum_{k = 1}^{N^i} \delta_{V^i_k} + \frac{1}{N + 1} \delta_{1},
    \\
    &\unweightglopredset{\alpha}(\xquery) = \ac{\yquery \in \mathcal{Y}\colon V(\xquery, \yquery) \le \q{1 - \alpha}\pr{\Bar{\mu}}}.
  \end{align*}

\paragraph{Oracle Weights.}
  This approach utilize importance weights to compute the prediction set.
  The \texttt{Oracle Weights} method has access to the true distribution of all the agents, enabling it to calculate the exact likelihood ratios.
  For this method, a number of $\Bar{\tcount}=\lfloor \tcount/2 \rfloor$ data points of the calibration dataset is randomly selected --- denoted as ${(X_k,Y_k)}_{k\in[\Bar{\tcount}]}$. This sumbsampling is based on a multinomial random variable with parameter $(\Bar{\tcount},\acn{\ccount{i}/\tcount}_{i\in[\nclients]})$; more details are provided in \Cref{sec:approach} (see for example \Cref{thm:YinCN-approx}).
  For any label $y\in\YC$, the prediction set is determined by:
  \begin{equation}\label{eq:def:oracle-weights}
    \begin{aligned}      
      &\Bar{\mu}_{\yquery}^{\target} = \qweightb{\yquery}{\yquery} \delta_{1} + \sum_{k=1}^{\Bar{\tcount}} \qweightb{Y_k}{\yquery} \delta_{V_k},
      \\
      &\mathcal{C}_{\alpha,\Bar{\mu}^{\target}}(\xquery)
      = \ac{\yquery\in\YC\colon V(\xquery,\yquery) \le \q{1 - \alpha}\prn{\Bar{\mu}_{\yquery}^{\target}}}
      .
  \end{aligned}
\end{equation}

\paragraph{Estimated Weights.}
  The empirical equivalent of \texttt{Oracle Weights} based on client label counts is \estwght{}.
  However, two sources of error are introduced: (1) the calibration subsampling and (2) the likelihood ratio estimations (refer to \eqref{eq:def:approx-probyy}).
  Similar to \texttt{Oracle Weights}, we draw a multinomial distribution $\{\Bar{\tcount}^i\}_{i\in[\nclients]}\sim \multi(\Bar{\tcount},\acn{\ccount{i}/\tcount}_{i\in[\nclients]})$ and subsample $\ccount{i}\wedge \Bar{\tcount}^i$ calibration data from each client $i$.
  The resulting prediction set is represented by:
  \begin{align*}
    &\qweight{y}{\yquery} = \frac{(\nofrac{\Mclcount{y}{\target}}{\Mlcount{y}}) \cdot\1_{\Mlcount{y}\ge 1}}{(\nofrac{\Mclcount{\yquery}{\target}}{\Mlcount{\yquery}}) \cdot\1_{\Mlcount{\yquery}\ge 1} + \sum_{\tilde{y}\in\YC} \card{\acn{(i,k)\in [\nclients]\times[\ccount{i}]\colon k\le \Bar{\tcount}^i, Y_k^i=y}} \times (\nofrac{\Mclcount{\tilde{y}}{\target}}{\Mlcount{\tilde{y}}}) \cdot\1_{\Mlcount{\tilde{y}}\ge 1} },
    \\
    &\widehat{\mu}_{\yquery}^{\scalebox{.35}{$\mathrm{MLE}$}} = \qweight{\yquery}{\yquery} \delta_{1} + \sum_{i=1}^{\nclients} \sum_{k=1}^{\ccount{i}\wedge \Bar{\tcount}^i} \qweight{Y_k^i}{\yquery} \delta_{V_{k}^{i}},
    \\
    &\predsetmle(\xquery) = \ac{\yquery\colon V(\xquery, \yquery) \le \q{1 - \alpha}\pr{\widehat{\mu}_{\yquery}^{\scalebox{.35}{$\mathrm{MLE}$}}}}.
  \end{align*}

\subsection{Additional Details on Numerical Experiments}\label{sec:suppl_experiments}

To provide a comprehensive overview of our experimental methodology, we present additional details about the federated optimization parameters, along with supplementary figures to complement those presented in \Cref{sec:experiments}.

\format{Optimization parameters.}
  For all experiments, we split the initial dataset $\mathcal{D}$ across the clients $(\mathcal{D}_{i})_{i \in [\nclients]}$ and the test dataset $\mathcal{D}_{\text{test}}$ as detailed in \Cref{sec:experiments}. 
  Label shift is then simulated by resampling using the clients' local distributions $\ac{\Pcal{y}}_{y \in \YC}$. For ImageNet experiments, we also assume that we have labels sampled from the target client's distribution for weights' approximation. The optimization parameters taken for \algopred\ experiments are $T=200$ iterations, $\gamma=1e^{-6}$ regularization parameter, $\eta=1e^{-3}$ stepsize, and $K=20$ local iteration rounds. We sample all clients during each communication round with the server.

\format{Simulated Data Experiments.}
  The generated data consists of $3$ Gaussians, two of which significantly overlap each other, see \Cref{fig:gaussians_data}.
  This data design is chosen such that we obtain different distributions of non-conformity scores for different classes, which is directly related to the different degrees of model confidence for different data samples. 
  \begin{figure}
    \centering
    \includegraphics[width=0.4\textwidth]{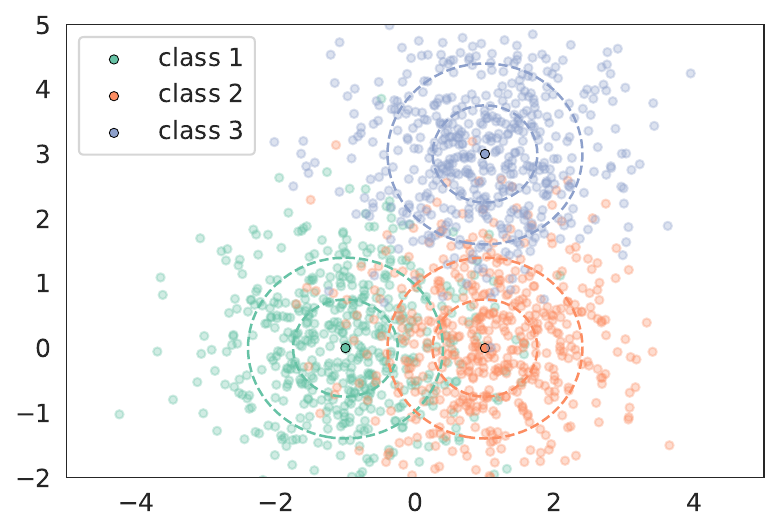}
    \caption{Simulated data distribution: $3$ two-dimensional Gaussians with means $\thetav_1 = [-1, 0], \thetav_2 = [1, 0], \thetav_3 = [1, 3]$ and identity covariance matrices.} \label{fig:toy_data}
  \label{fig:gaussians_data}
  \end{figure}

\format{CIFAR-10 Experiments.}
Using the CIFAR-$10$ data, we demonstrate a comparison of the empirical coverage for all considered methods: 
\unwghtloc, \unwghtglo, \texttt{Oracle Weights}, \estwght\ and the proposed \algopred\ method  version with $(\sigma_g=0)$, see \Cref{fig:cifar10_five_methods}.
\algopred\, along with weighted baselines, shows valid coverage results, unlike unweighted baselines.
At the same time, both weighted algorithms are extremely similar in performance to \algopred.
Unlike weighted baselines, \algopred\ is federated and privacy preserving. 

\begin{figure}
  \centering
    \includegraphics[width=\textwidth]{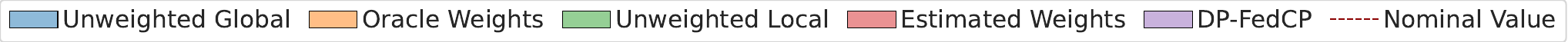}
    \label{fig:five_methods}\\
  \begin{subfigure}{0.4\textwidth}
    \includegraphics[width=\textwidth]{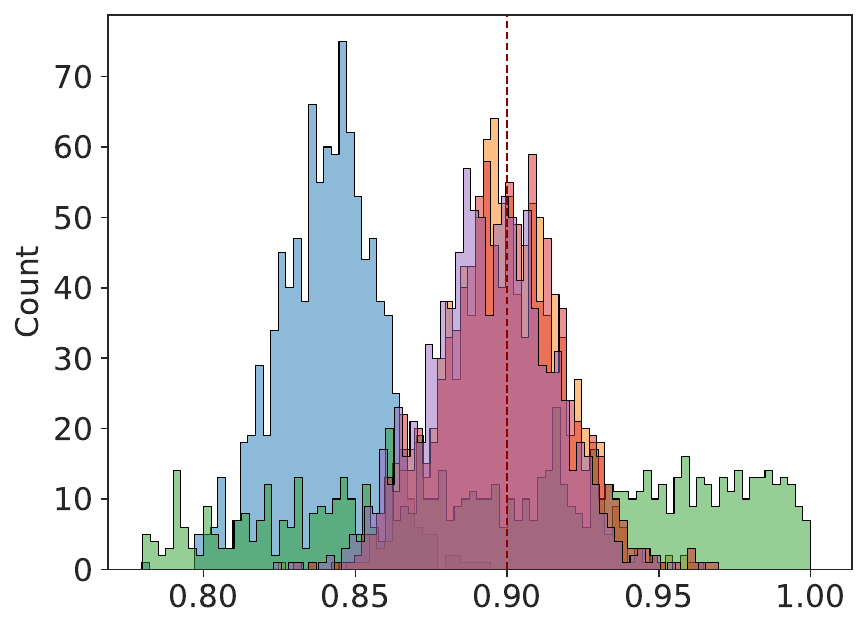}
    \caption{} \label{fig:cifar10_five_methods}
  \end{subfigure}
  ~~~~
  \begin{subfigure}{0.4\textwidth}
    \includegraphics[width=\textwidth]{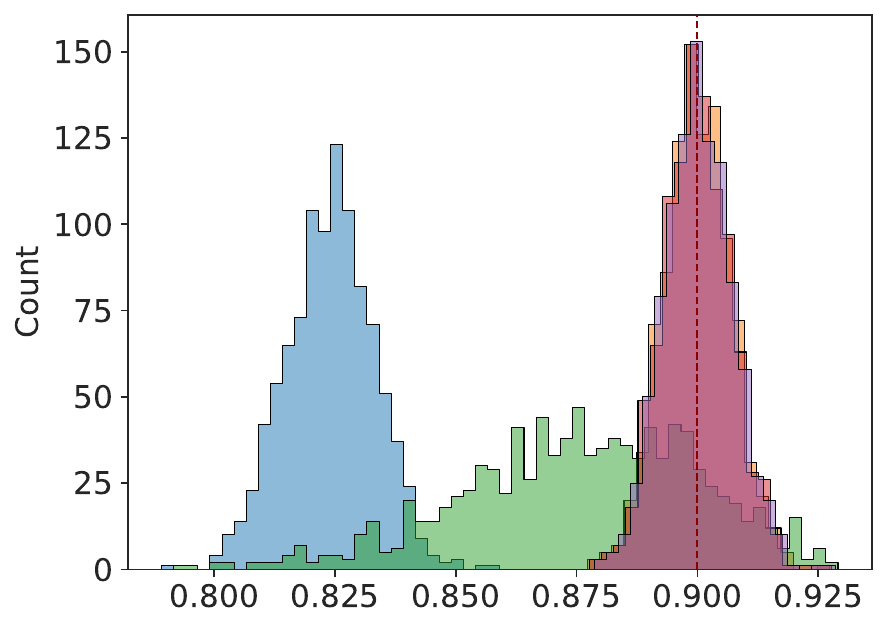}
    \caption{} \label{fig:imagenet_five_methods}
  \end{subfigure}
  \caption{Experimental results with all benchmarks. (a) CIFAR-10 empirical coverage. (b) ImageNet empirical coverage.}
\end{figure}

\format{ImageNet Experiments.}
  The ImageNet experiment is designed to have very different score distributions across agents.
  The grouping scheme of clients into a \texttt{low-score} group (\Cref{fig:imagenet_lowscore}) and a \texttt{high-score} group that consists of the querying agent (\Cref{fig:imagenet_highscore}) creates adversial heterogeneity, possible in real-life scenarios, under which unweighted methods are more prone to perform very poorly.
  Comparing the \algopred\ method with weighted and unweighted baselines on ImageNet data, we note the same behavior as on CIFAR-$10$ dataset, see \Cref{fig:imagenet_five_methods}. 
  Only algorithms that account for shifts between agents' data achieve the desired empirical coverage.
  There is an additional violin plot for the ImageNet differential privacy study that demonstrates the effect of the DP parameter $\sigma_g$ on the resulting empirical coverage; see~\Cref{fig:imagenet_DP_violin}.

  \begin{figure}
    \centering
    \begin{subfigure}{0.43\textwidth}
      \includegraphics[width=\linewidth]{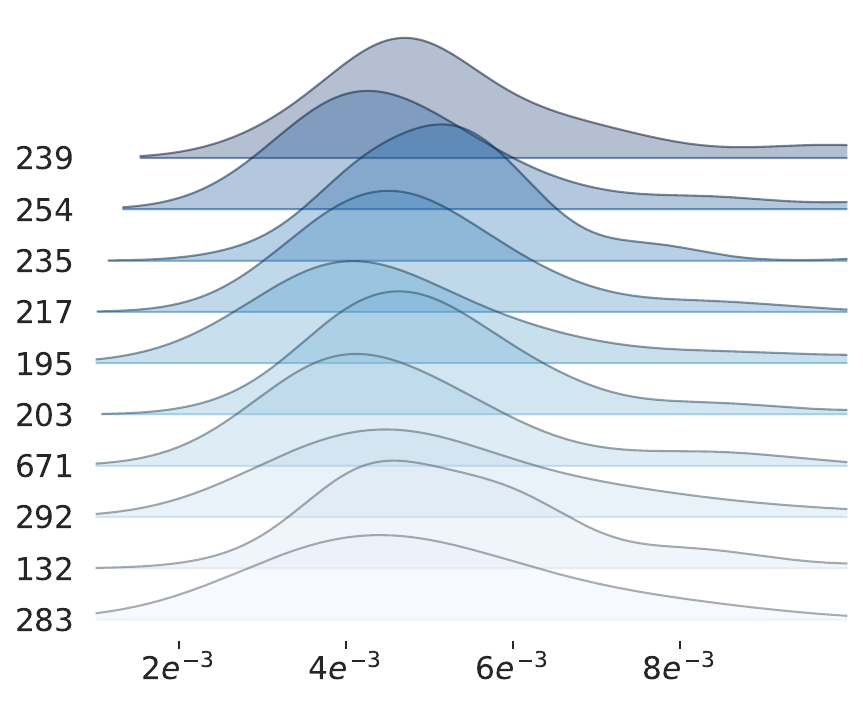}
      \caption{Distribution of scores for labels in G1} \label{fig:imagenet_lowscore}
    \end{subfigure}
    ~~~~
    \begin{subfigure}{0.43\textwidth}
      \includegraphics[width=\linewidth]{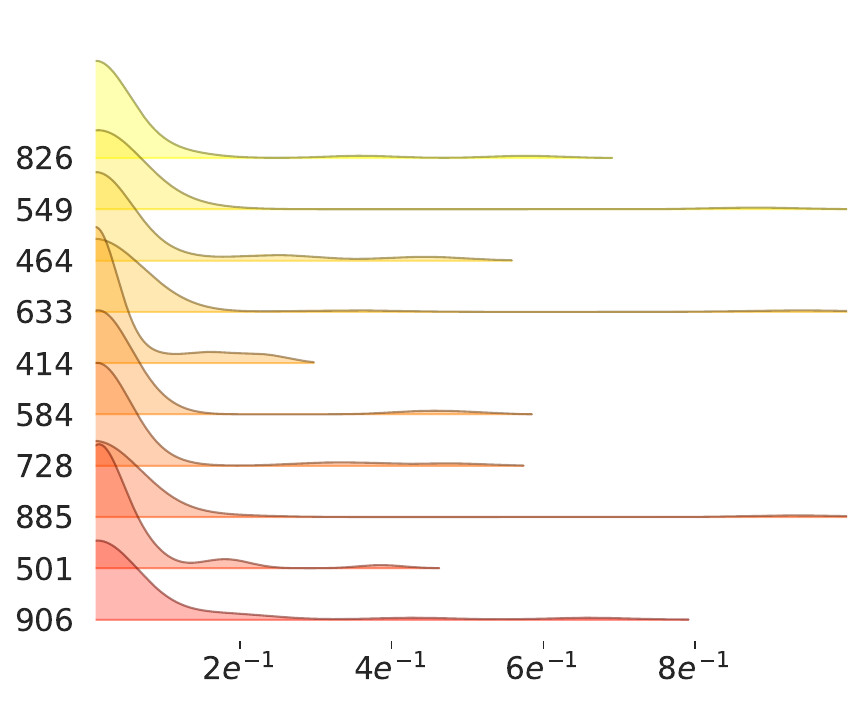}
      \caption{Distribution of scores for labels in G2} \label{fig:imagenet_highscore}
    \end{subfigure}\\
    \vspace{2em}
    \centering
    \includegraphics[width=0.7\linewidth]{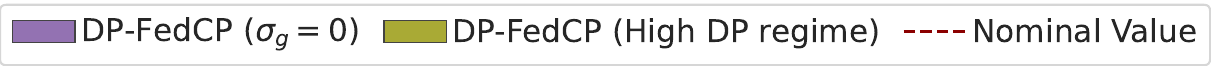}
    \label{fig:DP_methods}
    \begin{subfigure}{.9\textwidth}
      \includegraphics[width=\linewidth]{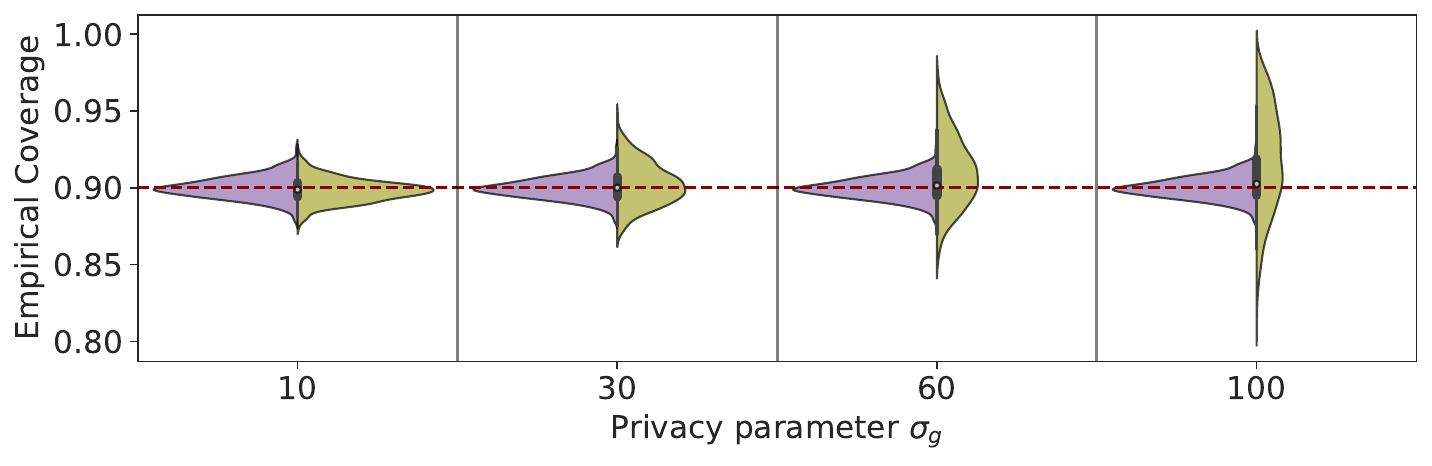}
      \caption{Privacy-Utility tradeoff} \label{fig:imagenet_DP_violin}
    \end{subfigure}
    \caption{ImageNet complementary results. (a) Score distributions of the classes with the lowest non-conformity scores. (b) Score distributions of the classes with the highest non-conformity scores. (c) Violin plot of the empirical coverage distribution for \algopred with different DP regimes.}
  \end{figure}

\end{document}